\documentclass[11pt,letterpaper]{article}

\usepackage[margin=1in]{geometry}
\usepackage[utf8]{inputenc}
\usepackage[T1]{fontenc}
\usepackage[hyphens]{url}
\usepackage{amsmath,amssymb,amsfonts,amsthm,mathtools,bm}
\usepackage{booktabs}
\usepackage{algorithm}
\usepackage{algorithmic}
\usepackage{hyperref}
\usepackage{url}
\usepackage{booktabs}
\usepackage{tabularx}
\usepackage{graphicx}

\newtheorem{theorem}{Theorem}[section]
\newtheorem{lemma}[theorem]{Lemma}
\newtheorem{proposition}[theorem]{Proposition}
\newtheorem{corollary}[theorem]{Corollary}
\newtheorem{assumption}[theorem]{Assumption}
\newtheorem{definition}[theorem]{Definition}
\theoremstyle{remark}
\newtheorem{remark}[theorem]{Remark}

\newcommand{\R}{\mathbb{R}}
\newcommand{\E}{\mathbb{E}}
\newcommand{\Prob}{\mathbb{P}}
\newcommand{\cS}{\mathcal{S}}
\newcommand{\cA}{\mathcal{A}}
\newcommand{\cF}{\mathcal{F}}
\newcommand{\cG}{\mathcal{G}}
\newcommand{\cL}{\mathcal{L}}

\newcommand{\cC}{\mathcal{C}}

\newcommand{\one}{\mathbf{1}}
\newcommand{\dist}{\operatorname{dist}}
\newcommand{\supp}{\operatorname{supp}}
\newcommand{\KKT}{\mathrm{KKT}}
\newcommand{\trans}{\mathsf{T}}
\newcommand{\eps}{\varepsilon}

\title{Markovian Nonconvex ADMM for Reinforcement Learning:\\
Bellman-Resolvent Stability Beyond Smooth Blocks}
\author{
  Zhaojun Peng \\
  \texttt{Hantsukipzj@gmail.com}
}
\date{}

\begin{document}

\maketitle

\begin{abstract}
We identify and study a structural mechanism for Markovian nonconvex ADMM in reinforcement learning. Using finite discounted MDPs as a canonical proving ground, we show that the discounted Bellman resolvent $(I-\gamma P_\pi)^{-1}$ can provide the multiplier stability that classical nonconvex ADMM analyses often obtain from a designated smooth block. Starting from this mechanism, we establish convergence under controlled Markov sampling and then under stochastic observations using an empirical Bellman surrogate that jointly represents the random residual and its Jacobian. Markov mixing, initialization drift, observation noise, and decaying bias enter as one operator perturbation, avoiding unbiased product and double sampling requirements. When the perturbations are square summable, the true KKT residual converges almost surely to zero. Under a finite conditional fourth moment condition, a companion iterate satisfies
$ \mathbb{E}[\widetilde G_{K+1}] \le A/T+(B/T)\sum_{k<T}m_k^{-1}, $
which becomes $O(T^{-1}+T/N)$ for total Markov sample budget $N$, giving $O(\epsilon^{-1})$ iteration complexity and $O(\epsilon^{-2})$ sample complexity for squared KKT accuracy $\epsilon$. Beyond stationarity, discounted occupancy coverage yields $J^\star-J(\pi)=O(\sqrt G)$ for direct tabular policies, so covered exact KKT points are globally optimal, while a statewise quadratic Bellman-improvement condition sharpens the relation to $O(G)$. Finally, nonlinear policy, projected Bellman, and explicit occupancy formulations exhibit the same chain of operator invertibility, dual representation, and multiplier stability. This supports discounted operator invertibility as a reusable structural principle for primal-dual reinforcement learning.
\end{abstract}

\section{Introduction}
The alternating direction method of multipliers is a natural tool for constrained nonconvex optimization because it separates difficult variables while preserving the interaction between primal and dual updates. Its analysis is delicate because the multiplier step increases the augmented Lagrangian. What the convergence proof ultimately needs is a structural estimate that controls \(\|\Delta\lambda_{k+1}\|\) by quantities dissipated in the primal steps. Existing nonconvex analyses commonly obtain such multiplier stability from a designated smooth block and a suitable linear operator~\cite{hong2016convergence,zeng2022moreau,yuan2025admm}. This observation shifts the question from smoothness itself to the source of dual stability: when regularity is carried by the constraint geometry, can that geometry provide the mechanism for controlling the multiplier directly? We use the finite discounted MDP as a canonical instantiation in which this mechanism can be exposed and verified end to end, from multiplier stability and Markovian perturbations to complexity bounds at a finite time horizon and policy performance.

Discounted reinforcement learning provides such a constraint. For a policy $\pi$, the Bellman equation is $M_\pi V-r_\pi=0$ with $M_\pi=I-\gamma P_\pi$. Since $0<\gamma<1$ and $P_\pi$ is stochastic, $M_\pi$ is invertible and $M_\pi^{-1}=\sum_{t\ge0}\gamma^tP_\pi^t$. The central question is whether this resolvent can replace the classical mechanism based on a smooth block in a nonconvex ADMM proof. The answer is affirmative. The optimality condition for the value update and the dual update give $\lambda^{k+1}=-\widehat M_{k+1}^{-\trans}(c+\eta_V^{-1}\Delta V_{k+1})$. Stability of the inverse resolvent then converts changes in the policy and value variables into a bound on the multiplier increment.

The theory is organized in three layers. The first is a structural layer. Bellman resolvent stability closes Lyapunov arguments for deterministic and controlled Markov settings without assuming that $\phi$ is smooth. The second is a stochastic layer that preserves the structure of the Bellman operator. Here the difficulty is not variance alone: the augmented policy gradient couples a random Bellman Jacobian with a random Bellman residual, so separate unbiasedness does not remove their covariance. We therefore generate the residual, its Jacobians, and the dual update from one empirical Bellman operator, turning Markov dependence, initialization drift, observation noise, and decaying bias into a single operator perturbation. The third is an RL geometry layer. Occupancy measures, advantage functions, and the performance difference identity turn the optimization residual into a global policy guarantee. This layer further characterizes two performance regimes: occupancy coverage yields the $O(\sqrt G)$ guarantee, while a curvature condition on the Bellman improvement at each state sharpens the conversion to $O(G)$.

Our contributions are as follows. First, we identify the discounted Bellman resolvent as a source of multiplier stability specific to RL and prove the corresponding bounds on dual increments and the Lyapunov function. Second, we show that the stochastic extension must preserve the Bellman operator coupling. Under Markov sampling controlled by resets and uniform exploration, one shared empirical Bellman surrogate jointly generates the residual, its Jacobians, and the dual update, converting dependence between the residual and its Jacobian, observation noise, bias, and Markov dependence into a controlled operator perturbation and yielding convergence of the KKT residual to zero almost surely. Third, assuming a finite conditional fourth moment, we introduce a companion iterate aligned with the policy subproblem and prove $\E[\widetilde G_{K+1}]=O(T^{-1}+T/N)$, yielding $O(\eps^{-1})$ outer iterations and $O(\eps^{-2})$ Markov samples for accuracy $\eps$ measured by the squared KKT residual. Fourth, for direct tabular policies, occupancy coverage gives $J^\star-J(\pi)=O(\sqrt G)$ and exact covered KKT points are globally optimal. Statewise quadratic Bellman improvement strengthens this to $O(G)$. Matching constructions characterize the roles of occupancy coverage and curvature of statewise Bellman improvement in the two performance regimes. Fifth, we use formulations based on nonlinear policies, projected Bellman equations, and occupancy measures to test which parts of the mechanism are specific to the tabular Bellman model. These extensions show that the reusable ingredient is the existence of a stable invertible discounted operator yielding an explicit dual representation.

\section{Related Work }

\paragraph{Nonconvex ADMM.}
Classical ADMM is well understood in convex settings~\cite{gabay1976dual,boyd2011distributed,he2012rate}, while nonconvex analyses require additional structure to control multiplier motion and recover descent~\cite{hong2016convergence,li2016douglas,zeng2022moreau}. Recent approaches weaken objective regularity through restricted strong convexity, smoothing, increasing penalization, or inexact block updates~\cite{barber2024convergence,yuan2024smoothing,yuan2025admm,bai2025inexact}. The multiaffine ADMM theory of \cite{gao_multiaffine} controls multiplier increments through a fixed linear subblock whose image contains the image of the remaining constraint map. In our Bellman formulation, the value-block Jacobian $M_\pi=I-\gamma P_\pi$ varies with the policy, so this fixed-operator structure does not directly cover the policy–value splitting. We exploit uniform invertibility and inverse-map stability of $M_\pi$ to control multiplier motion, and preserve the resulting value–dual identity under Markov sampling through a shared empirical Bellman operator.

\paragraph{Stochastic ADMM and Markovian optimization.}
Recent stochastic ADMM methods mainly study finite-sum or expectation objectives using variance reduction, momentum, Bregman geometry, adaptive batches, or hybrid estimators~\cite{zeng2024accelerated,liu2025inertial,jin2026batch,zeng2026hybrid}. Complementary Markovian optimization theory develops finite-time stochastic approximation, unbounded-noise analysis, lower bounds, and single-trajectory variance reduction~\cite{adibi2024delayed,haque2025unbounded,sun2025lower,peng2026variance,moulos2019optimal}. Our stochasticity enters the equality-constraint operator itself: the Bellman residual and its Jacobian are generated from the same controlled trajectory. We therefore use one shared empirical Bellman operator across the primal and dual updates, preserving row stochasticity, resolvent invertibility, and the value--dual identity while treating residual--Jacobian dependence as operator perturbation.

\paragraph{Policy optimization and occupancy geometry.}
Performance-difference and occupancy identities connect local policy information to global return~\cite{kakade2002approximately,agarwal2021theory}. Recent actor--critic and policy-optimization analyses establish global or nonasymptotic guarantees under Markov sampling, neural or general function approximation, average-reward objectives, constraints, and occupancy-induced geometry~\cite{gaur2024actorcritic,wang2024singleloop,patel2024average,ganesh2025order,chen2025continuous,ganesh2025average,montenegro2024constrained,satheesh2026cmdp,sherman2024linear,sherman2025pmd,asad2025spma}. Occupancy has also been used directly as an optimization representation~\cite{huang2024occupancy,neu2025featureoccupancy,barakat2025global}. We first establish convergence of the constrained primal--dual system and then use occupancy geometry to convert its KKT residual into global policy performance and the same discounted occupancy structure also yields a second invertible operator in Section~\ref{sec:extensions}.


\begin{table}[t]
\centering
\caption{Structural positioning across neighboring theories. Rates are reported in each paper's native stationarity or performance metric.}
\label{tab:positioning}
\fontsize{6.35}{6.85}\selectfont
\setlength{\tabcolsep}{1.7pt}
\renewcommand{\arraystretch}{1.03}
\begin{tabularx}{\textwidth}{@{}p{1.28cm}p{1.28cm}p{1.12cm}X X X p{1.35cm}@{}}
\toprule
Ref. & Method & Data & Structural assumption & Stability / analysis device & Guarantee & RL meaning \\
\midrule
\cite{hong2016convergence} & ADMM & Det. & Nonconvex consensus / sharing; large penalty & Consensus / sharing structure + AL descent & Stationary-set convergence & -- \\
\cite{zeng2022moreau} & MEAL & Det. & Weakly convex nonsmooth, linear constraints & Moreau envelope of AL & $o(\eps^{-2})$ FOSP complexity & -- \\
\cite{barber2024convergence} & ADMM & Det. & Nonsmooth nonconvex under RSC & Restricted strong convexity & Convergence without differentiability & -- \\
\cite{yuan2025admm} & Prox-ADMM & Det. & Multi-block; one continuous block; BI / SU operator & Increasing penalty + decreasing smoothing & $O(\eps^{-3})$ critical-point complexity & -- \\
\cite{zeng2024accelerated} & Stoch. ADMM & Finite sum & Nonconvex nonsmooth linear constraint & SVRG + momentum + potential & $O(1/T)$ stationarity; linear under KL & -- \\
\cite{sun2025lower} & 1st-order opt. & Markov & Smooth nonconvex gradient oracle & Markov-oracle lower-bound construction & $\Omega(\eps^{-4})$ countable; $\Omega(\eps^{-2})$ finite chain & RL-relevant oracle \\
\cite{gaur2024actorcritic} & Actor--critic & Markov & Neural actor / critic; weak gradient domination & Critic-error control + gradient domination & $\widetilde O(\eps^{-3})$ global sample complexity & Global last iterate \\
\textbf{Ours} & \textbf{Bellman-ADMM} & \textbf{Ctrl. Markov} & \textbf{Proper l.s.c. $\phi$; exact policy subproblem; invertible discounted RL operator} & \textbf{Resolvent + shared empirical operator + companion iterate} & \textbf{$\E\widetilde G=O(T^{-1}+T/N)$; $N=O(\eps^{-2})$ for squared KKT} & \textbf{$O(\sqrt G)$; $O(G)$ with curvature} \\
\bottomrule
\end{tabularx}
\vspace{-1.5ex}
\end{table}

\section{Problem Formulation and Algorithm}
\label{sec:problem}
We develop the mechanism through a finite discounted MDP, where the discounted operator, its inverse, the induced Markov sampling process, and the policy-performance geometry can all be characterized explicitly. This model serves as a concrete proving ground for the operator-stability principle rather than an abstract definition of its scope. Section 7 then tests the same mechanism under nonlinear policies, projected Bellman equations, and an independent occupancy formulation.

In detail, we consider a finite discounted MDP $(\cS,\cA,P,r,\rho,\gamma)$ with $0<\gamma<1$. The direct policy space is
$\Pi=\prod_{s\in\cS}\Delta(\cA).$
For $\pi\in\Pi$, define
$P_\pi(s,s')=\sum_{a\in\cA}\pi(a\mid s)P(s'\mid s,a), \quad r_\pi(s)=\sum_{a\in\cA}\pi(a\mid s)r(s,a),$
and
$M_\pi=I-\gamma P_\pi, \quad R(\pi,V)=M_\pi V-r_\pi.$
We study
\begin{equation}
\min_{\pi\in\Pi,\,V\in\R^{|\cS|}}
\phi(\pi)+c^\trans V
\quad
\text{subject to}
\quad
R(\pi,V)=0.
\label{eq:problem}
\end{equation}
The function $\phi$ may be nonsmooth and nonconvex. It is proper, lower semicontinuous, and bounded below on $\Pi$, with $\operatorname{dom}\phi\cap\Pi\ne\varnothing$. Define $\psi=\phi+\delta_\Pi$, where $\delta_\Pi$ is the indicator of $\Pi$. We use the limiting subdifferential $\partial\psi$ for policy stationarity. Standard policy optimization is recovered with $\phi\equiv0$ and $c=-\mu$, where $\mu$ is a training initial-state distribution.

The augmented Lagrangian is
$
\cL_\beta(\pi,V,\lambda)
=\phi(\pi)+c^\trans V
+\langle\lambda,R(\pi,V)\rangle
+\frac{\beta}{2}\|R(\pi,V)\|^2.
\label{eq:lagrangian}
$

\begin{definition}[Squared KKT residual]
For $z=(\pi,V,\lambda)$, let
\begin{align}
G(z)
={}&\dist^2\left(0,\partial\psi(\pi)+J_\pi(\pi,V)^\trans\lambda\right)
+
\|c+M_\pi^\trans\lambda\|^2
+\|R(\pi,V)\|^2,
\label{eq:kkt-residual}
\end{align}
where $J_\pi(\pi,V)=D_\pi R(\pi,V)$. For $\phi\equiv0$, $\partial\psi=N_\Pi$.
\end{definition}

\subsection{Controlled Markov sampling}

At iteration $k$, the algorithm freezes an exploratory behavior policy $\bar\pi^k$ and simulates the reset kernel
$K_k(s,s')=(1-\gamma)\rho(s')+\gamma P_{\bar\pi^k}(s,s').$
A reset indicator records whether a transition came from the environment or from $\rho$. Only environment transitions are used to estimate $P$. This distinction prevents the estimator from converging to $K_k$.

\begin{assumption}[Coverage and stochastic observations]
\label{ass:coverage-noise}
There are constants $\rho_{\min}>0$, $\alpha>0$, and $R_{\max}<\infty$ such that
$\rho(s)\geq\rho_{\min}, \quad \bar\pi^k(a\mid s)\geq\alpha, \quad |r(s,a)|\leq R_{\max}.$
A reward observation has the form $Y_{k,t}=r(S_{k,t},A_{k,t})+\zeta_{k,t}$, where
$\E[\zeta_{k,t}\mid\cG_{k,t}]=0, \quad \E[|\zeta_{k,t}|^4\mid\cG_{k,t}]\leq M_4.$
Here $\cG_{k,t}$ contains all past observations and the current state, action, and reset indicator, before the current reward and next state are drawn. Fresh reset coins, reset states, and action randomization are independent of the preceding history. On non-reset steps, the conditional next-state law is $P(\cdot\mid S_{k,t},A_{k,t})$. The current reward and next state may be dependent.
\end{assumption}

A batch of length $m_k$ produces action-conditional empirical rows $\widehat P_k(\cdot\mid s,a)$ and rewards $\widehat r_k(s,a)$. Every unvisited transition row is replaced by a fixed probability vector, and every unvisited reward row uses a fixed bounded value. Hence all empirical rows remain stochastic. For each candidate policy, define
$\widehat P_{\pi,k}=\sum_a\pi(a\mid s)\widehat P_k(\cdot\mid s,a), \quad \widehat r_{\pi,k}=\sum_a\pi(a\mid s)\widehat r_k(s,a),$
$\widehat M_{\pi,k}=I-\gamma\widehat P_{\pi,k}, \quad \widehat R_k(\pi,V)=\widehat M_{\pi,k}V-\widehat r_{\pi,k}.$
Row stochasticity gives
$\|M_\pi^{-1}\|_\infty\leq\frac{1}{1-\gamma}, \quad \|\widehat M_{\pi,k}^{-1}\|_\infty\leq\frac{1}{1-\gamma}.$
Finite-dimensional norm equivalence therefore yields a common spectral bound $\bar\kappa<\infty$.

\subsection{Markovian proximal Bellman-ADMM}

Let $\widehat\cL_{\beta,k}$ denote~\eqref{eq:lagrangian} with $R$ replaced by $\widehat R_k$. Algorithm~\ref{alg:bellman-admm} uses the same empirical model in the policy, value, and dual updates.

\begin{algorithm}[t]
\caption{Markovian Proximal Bellman-ADMM}
\label{alg:bellman-admm}
\begin{algorithmic}[1]
\STATE Initialize $(\pi^0,V^0,\lambda^0)$ and choose $\beta,\eta_\pi,\eta_V>0$
\FOR{$k=0,1,\ldots,T-1$}
\STATE Freeze $\bar\pi^k$ and collect a reset-controlled batch of length $m_k$
\STATE Construct $(\widehat P_k,\widehat r_k)$ from non-reset transitions
\STATE $\displaystyle
\pi^{k+1}\in\arg\min_{\pi\in\Pi}
\left\{\widehat\cL_{\beta,k}(\pi,V^k,\lambda^k)
+\frac{1}{2\eta_\pi}\|\pi-\pi^k\|^2\right\}$
\STATE $\displaystyle
V^{k+1}=\arg\min_V
\left\{\widehat\cL_{\beta,k}(\pi^{k+1},V,\lambda^k)
+\frac{1}{2\eta_V}\|V-V^k\|^2\right\}$
\STATE $\displaystyle
\lambda^{k+1}=\lambda^k+\beta\widehat R_k(\pi^{k+1},V^{k+1})$
\ENDFOR
\end{algorithmic}
\end{algorithm}

\begin{assumption}[Algorithmic margins]
\label{ass:algorithmic-margins}
The policy subproblem admits a global minimizer and is solved exactly. The value subproblem is solved exactly. The initial policy is deterministic and belongs to $\operatorname{dom}\psi$. There is $L_M<\infty$ such that
$\|M_{\pi'}-M_\pi\|\leq L_M\|\pi'-\pi\|.$
The reduced objective
$b(\pi)=\phi(\pi)+c^\trans M_\pi^{-1}r_\pi$
is bounded below on $\Pi$. The parameters satisfy
\begin{equation}
\beta>\max\left\{
\frac{60\bar\kappa^2}{\eta_V},
30\eta_\pi\bar\kappa^4L_M^2\|c\|^2
\right\},
\qquad
\tau=\frac{1}{8\eta_V}.
\label{eq:parameter-condition}
\end{equation}
The exact policy solve and a structured nonsmooth instance are described in Appendix~\ref{app:bellman-proofs}. The initialization satisfies $\E\|V^0\|^4+\E\|\lambda^0\|^4<\infty$.
\end{assumption}

\section{Resolvent Stability and Lyapunov Descent}
\label{sec:mechanism}

The value-step optimality condition is
$c+\widehat M_{k+1}^\trans\lambda^k +\beta\widehat M_{k+1}^\trans\widehat R_k(\pi^{k+1},V^{k+1}) +\eta_V^{-1}\Delta V_{k+1}=0,$
where $\widehat M_{k+1}=\widehat M_{\pi^{k+1},k}$ and $\Delta V_{k+1}=V^{k+1}-V^k$. Combining this relation with the dual update gives
$
c+\widehat M_{k+1}^\trans\lambda^{k+1}
+\eta_V^{-1}\Delta V_{k+1}=0.
$
Thus
$
\lambda^{k+1}
=-\widehat M_{k+1}^{-\trans}
\left(c+\eta_V^{-1}\Delta V_{k+1}\right).
$

Define the empirical-model error
$\eps_k =\sup_{\pi\in\Pi} \left( \|\widehat P_{\pi,k}-P_\pi\| +\|\widehat r_{\pi,k}-r_\pi\| \right).$

\begin{lemma}[Empirical Bellman moments]
\label{lem:model-moments}
There are deterministic constants $C_2,C_4<\infty$ such that
\begin{equation}
\E[\eps_k^2\mid\cF_k]\leq\frac{C_2}{m_k},
\qquad
\E[\eps_k^4\mid\cF_k]\leq\frac{C_4}{m_k^2}.
\label{eq:model-moments}
\end{equation}
\end{lemma}

\begin{lemma}[A priori fourth-moment stability]
\label{lem:moment-stability}
Under Assumptions~\ref{ass:coverage-noise} and~\ref{ass:algorithmic-margins}, there are constants $C_V^{(4)},C_\lambda^{(4)}<\infty$ such that
$\sup_{k\geq0}\E\|V^k\|^4\leq C_V^{(4)}, \quad \sup_{k\geq0}\E\|\lambda^k\|^4\leq C_\lambda^{(4)}.$
\end{lemma}

This argument establishes iterate stability directly from the empirical resolvent and the coupled value--dual recursion, providing the independent boundedness needed by the subsequent Lyapunov analysis.

\begin{lemma}[Bellman-resolvent multiplier control]
\label{lem:dual-control}
For $k\ge1$, there are deterministic constants $C_\pi,C_V,C_\eps>0$ such that
\begin{align}
\E\|\Delta\lambda_{k+1}\|^2
\leq{}&C_\pi\E\|\Delta\pi_{k+1}\|^2
+C_V\E\|\Delta V_{k+1}\|^2
+C_V\E\|\Delta V_k\|^2
+C_\eps\left(\frac{1}{m_k}+\frac{1}{m_{k-1}}\right).
\label{eq:dual-control}
\end{align}
One valid choice for the deterministic primal coefficients is
$C_\pi=5\bar\kappa^4L_M^2\|c\|^2, \quad C_V=\frac{5\bar\kappa^2}{\eta_V^2}.$
\end{lemma}

The lemma is the main structural result. It follows from~the value-dual identity, inverse-map stability, and the empirical-model moment bounds.

\paragraph{Structure-preserving stochasticization.}
Sharing an empirical Bellman operator preserves the multiplier identity under transition estimation. For $\widehat R=R+e$ and $\widehat J=J+G$, separate zero-mean errors do not eliminate the product term $G^\trans e$ in the augmented policy gradient. We therefore construct one complete empirical Bellman operator per batch and use it consistently in the residual, Jacobians, and dual update. Lemma~\ref{lem:surrogate-closure} shows that the resulting stochastic effects enter the true primal--dual system through the single operator error $\eps_k$.

\begin{lemma}[Structure-preserving surrogate closure]
\label{lem:surrogate-closure}
Let $e_k(\pi,V)=\widehat R_k(\pi,V)-R(\pi,V)$. On every path for which $\sum_k\eps_k^2<\infty$ and the a priori iterate bounds hold, there are finite constants such that four true-problem relations hold simultaneously. First, the primal steps satisfy
$\cL_\beta(\pi^k,V^k,\lambda^k)-\cL_\beta(\pi^{k+1},V^{k+1},\lambda^k)\ge a_\pi\|\Delta\pi_{k+1}\|^2+a_V\|\Delta V_{k+1}\|^2-C\eps_k^2$.
Second, the empirical value-dual identity becomes $c+M_{k+1}^\trans\lambda^{k+1}+\eta_V^{-1}\Delta V_{k+1}=r_{k+1}$ with $\|r_{k+1}\|\le C\eps_k$. Third, the multiplier update becomes $\Delta\lambda_{k+1}=\beta R_{k+1}+d_{k+1}$ with $\|d_{k+1}\|\le C\eps_k$. Fourth, the true KKT residual obeys $\KKT_{k+1}\le C(\|\Delta\pi_{k+1}\|+\|\Delta V_{k+1}\|+\|\Delta\lambda_{k+1}\|+\eps_k)$.
\end{lemma}

\begin{remark}[Shared empirical operators control product bias]
\label{rem:product-bias}
Separate unbiasedness of a random residual $\widehat R=R+e$ and a random Jacobian $\widehat J=J+G$ does not imply an unbiased augmented gradient. Indeed,
$\widehat J^\trans(\lambda+\beta\widehat R)-J^\trans(\lambda+\beta R)
=G^\trans(\lambda+\beta R)+\beta J^\trans e+\beta G^\trans e$,
and generally $\E[G^\trans e]\neq0$. Coupling $\widehat R$ and $\widehat J$ through one empirical Bellman surrogate turns $G^\trans e$ into a pathwise $O(\eps_k^2)$ perturbation. The stochastic analysis therefore proceeds through operator error without an unbiased-product or double-sampling condition.
\end{remark}

Define the corrected potential
$\Phi_k=\cL_\beta(\pi^k,V^k,\lambda^k)+\tau\|\Delta V_k\|^2, \quad \Delta V_0=0.$

\begin{lemma}[Expected Lyapunov descent]
\label{lem:lyapunov}
For $k\ge1$, there are constants $c_\pi,c_V,c_V',C_E>0$ such that
\begin{align}
\E[\Phi_k]-\E[\Phi_{k+1}]
\geq{}&c_\pi\E\|\Delta\pi_{k+1}\|^2
+c_V\E\|\Delta V_{k+1}\|^2
+c_V'\E\|\Delta V_k\|^2
-C_E\left(\frac{1}{m_k}+\frac{1}{m_{k-1}}\right).
\label{eq:expected-descent}
\end{align}
There are $\underline\Phi\in\R$ and $C_L>0$ such that
$
\E[\Phi_k]\geq\underline\Phi-\frac{C_L}{m_{k-1}}.
$
\end{lemma}

The empirical subproblems produce sufficient decrease for $\widehat\cL_{\beta,k}$. We transfer this decrease to $\cL_\beta$ by retaining the full empirical Bellman operator. Conditional Cauchy--Schwarz, Lemma~\ref{lem:model-moments}, and Lemma~\ref{lem:moment-stability} reduce the statistical contribution to $O(m_k^{-1})$.

\begin{theorem}[Controlled-Markov Bellman-ADMM convergence]
\label{thm:l1}
Assume the finite-MDP, reset-access, uniform-exploration, resolvent, and parameter conditions above, with no additional observation noise. If $\sum_km_k^{-1}<\infty$, then almost surely $\Delta\pi_k\to0$, $\Delta V_k\to0$, $\Delta\lambda_k\to0$, and $R(\pi^k,V^k)\to0$. Every accumulation point of $(\pi^k,V^k,\lambda^k)$ satisfies the KKT conditions of~\eqref{eq:problem}.
\end{theorem}   

\begin{theorem}[Structure-preserving stochastic KKT convergence]
\label{thm:as-convergence}
Assume the model, sampling, and algorithmic conditions above, with the fourth-moment noise condition replaced by conditional variance at most $\sigma_k^2$. Allow the average bias in each visited state-action reward row to have magnitude at most $b_k$. If
$
\sum_{k=0}^{\infty}\frac{1+\sigma_k^2}{m_k}<\infty,
\quad
\sum_{k=0}^{\infty}b_k^2<\infty,
$
then, almost surely,
$\Delta\pi_k\to0, \quad \Delta V_k\to0, \quad \Delta\lambda_k\to0, \quad R(\pi^k,V^k)\to0.$
The true KKT residual converges to zero. Every accumulation point satisfies the KKT conditions of~\eqref{eq:problem}.
\end{theorem}

\section{KKT Convergence and Complexity}
\label{sec:complexity}

The policy subproblem is solved at $(\pi^{k+1},V^k)$. We evaluate a companion iterate that preserves this first-order alignment. Define
$
\widetilde\lambda^{k+1}
=\lambda^k+\beta\widehat R_k(\pi^{k+1},V^k),
\quad
\widetilde z^{k+1}=(\pi^{k+1},V^k,\widetilde\lambda^{k+1}),
$
and $\widetilde G_{k+1}=G(\widetilde z^{k+1})$.

\begin{lemma}[Companion residual]
\label{lem:companion}
For $k\ge1$, there are deterministic constants $C_G,C_G'>0$ such that
\begin{align}
\E[\widetilde G_{k+1}]
\leq{}&C_G\E\left[
\|\Delta\pi_{k+1}\|^2
+\|\Delta V_{k+1}\|^2
+\|\Delta V_k\|^2
\right]
+C_G'\left(\frac{1}{m_k}+\frac{1}{m_{k-1}}\right).
\label{eq:companion-bound}
\end{align}
\end{lemma}

\begin{theorem}[Finite-time squared-KKT complexity]
\label{thm:finite-time}
Under Assumptions~\ref{ass:coverage-noise} and~\ref{ass:algorithmic-margins}, let $K$ be uniform on $\{0,\ldots,T-1\}$ and independent of the algorithmic randomness. There are constants $A,B>0$ that do not depend on $T$ or the batch schedule such that
$
\E[\widetilde G_{K+1}]
\leq
\frac{A}{T}
+\frac{B}{T}\sum_{k=0}^{T-1}\frac{1}{m_k}.
$
If $N=\sum_{k=0}^{T-1}m_k$ and $m_k=N/T$, then
$
\E[\widetilde G_{K+1}]
\leq \frac{A}{T}+\frac{BT}{N}.
$
Consequently,
$T(\eps)=O(\eps^{-1}), \quad N(\eps)=O(\eps^{-2})$
are sufficient for $\E[\widetilde G_{K+1}]\leq\eps$.
\end{theorem}
The theorem measures a squared residual. For the unsquared residual $\widetilde R=\sqrt{\widetilde G}$, Jensen's inequality gives $T=O(\eps^{-2})$ and $N=O(\eps^{-4})$.
\begin{corollary}[Iteration and Markov sample complexity]
\label{cor:complexity}
If $(1/T)\sum_{k<T}m_k^{-1}=O(T^{-1})$, then $\E\widetilde G_{K+1}=O(T^{-1})$. For fixed $T$ and total budget $N$, equal batches minimize $\sum_km_k^{-1}$ and give $\E\widetilde G_{K+1}\le A/T+BT/N$. Thus $T=O(\eps^{-1})$ and $N=O(\eps^{-2})$ suffice for squared residual $\eps$.
\end{corollary}

\section{From KKT Residual to Policy Performance}
\label{sec:performance}

We now set $\phi\equiv0$ and $c=-\mu$. Let
$V^\pi=M_\pi^{-1}r_\pi, \quad J_\nu(\pi)=\nu^\trans V^\pi, \quad F_\mu(\pi)=-J_\mu(\pi).$
Define the policy-stationarity residual, following the occupancy/gradient-domination viewpoint in~\cite{agarwal2021theory,huang2024occupancy,sherman2025pmd,barakat2025global},
$R_{\mathrm{pol}}(\pi) =\dist\left(0,\nabla F_\mu(\pi)+N_\Pi(\pi)\right).$

\paragraph{Residual-transfer chain.}
The policy-performance result rests on five explicit intermediate statements. Lemma~\ref{lem:adjoint-representation} gives the Bellman adjoint representation $\nabla F_\mu(\pi)=J_\pi(\pi,V^\pi)^\trans\lambda_\mu^\pi$ with $\lambda_\mu^\pi=M_\pi^{-\trans}\mu$. Lemma~\ref{lem:tabular-jacobian} gives the tabular Jacobian $\partial R_s(\pi,V)/\partial\pi(a\mid s)=-Q_V(s,a)$ and therefore $[J_\pi(\pi,V)^\trans\lambda]_{s,a}=-\lambda_sQ_V(s,a)$. Lemma~\ref{lem:value-error} gives $\|V-V^\pi\|\le\bar\kappa e_R$ from Bellman feasibility. Lemma~\ref{lem:adjoint-error} gives $\|\lambda-\lambda_\mu^\pi\|\le\bar\kappa e_V$ from value-stationarity error. Lemma~\ref{lem:adjoint-gradient-error} combines them into $\|J_\pi(\pi,V)^\trans\lambda-\nabla F_\mu(\pi)\|\le C_Ve_V+C_Re_R+C_{VR}e_Ve_R$.

\begin{lemma}[KKT residual controls policy stationarity]
\label{lem:kkt-policy}
Let $G=G(\pi,V,\lambda)$. If $G\leq1$, then
$R_{\mathrm{pol}}(\pi)\leq C_{\mathrm{stat}}\sqrt{G},$
where $C_{\mathrm{stat}}<\infty$ depends only on the discounted resolvent and bounded model quantities.
\end{lemma}

Let
$d_\mu^\pi=(1-\gamma)M_\pi^{-\trans}\mu$
be the normalized discounted state occupancy. For an optimal policy $\pi^\star$ under evaluation distribution $\nu$, define
$\cC_2(\pi^\star,\pi,\nu,\mu) =\left\| \frac{d_\nu^{\pi^\star}}{d_\mu^\pi} \right\|_2,$
with $0/0=0$ and value $+\infty$ when the numerator is positive and the denominator is zero.

\begin{theorem}[Occupancy-mismatch performance bound]
\label{thm:performance}
If
$\supp d_\nu^{\pi^\star}\subseteq\supp d_\mu^\pi,$
then
$
J_\nu(\pi^\star)-J_\nu(\pi)
\leq
\sqrt{2}\,\cC_2(\pi^\star,\pi,\nu,\mu)
R_{\mathrm{pol}}(\pi).
$
If $G\leq1$, then
$
J_\nu(\pi^\star)-J_\nu(\pi)
\leq
\sqrt{2}\,\cC_2 C_{\mathrm{stat}}\sqrt{G}.
$
Every covered exact KKT point is globally optimal under $\nu$.
\end{theorem}

\begin{corollary}[Full-support global optimality]
\label{cor:full-support-global}
If the training distribution satisfies $\mu(s)\ge\mu_{\min}>0$, then $J_\mu^\star-J_\mu(\pi)\le\sqrt{2}R_{\rm pol}(\pi)/((1-\gamma)\mu_{\min})$. Hence every exact policy-stationary point is globally optimal.
\end{corollary}

\begin{corollary}[KKT global optimality and finite-time policy bound]
\label{cor:kkt-global-finite}
Every exact KKT point satisfying the occupancy-cover condition in Theorem~\ref{thm:performance} is globally optimal under $\nu$. If along the random output iterate $K$ the mismatch coefficient is bounded by $\bar C_{\rm occ}$ and $\E G_K\le A/T+(B/T)\sum_{k<T}m_k^{-1}$, then
$\E[J_\nu^\star-J_\nu(\pi_K)]\le C\sqrt{A/T+(B/T)\sum_{k<T}m_k^{-1}}$.
\end{corollary}

\begin{proposition}[Coverage cannot be removed]
\label{prop:coverage-necessary}
There exists a three-state discounted MDP and a training distribution concentrated at the initial state for which a policy is first-order stationary but strictly globally suboptimal. In the construction, the current policy never visits a downstream state. Changing only the upstream action or only the downstream action is non-improving, while changing both raises return from $1$ to $9$. Hence occupancy coverage is a first-order identifiability condition.
\end{proposition}

\begin{proposition}[Separation between performance gap and squared policy stationarity]
\label{prop:pl-fails}
There is a three-stage chain and policies $\pi_\eps$ such that $J^\star-J(\pi_\eps)=\gamma^2\eps^3$ while $R_{\mathrm{pol}}(\pi_\eps)^2=(3/2)\gamma^4\eps^4$. Therefore no finite constant $C$ can satisfy $J^\star-J(\pi)\le C R_{\mathrm{pol}}(\pi)^2$ for all direct tabular policies.
\end{proposition}

For state $s$, define
$\bar r_s(\pi) =\dist\left(0,-Q^\pi(s,\cdot)+N_{\Delta(\cA)}(\pi_s)\right), \quad \delta_s^\pi=TV^\pi(s)-V^\pi(s).$

\begin{assumption}[Quadratic Bellman improvement]
\label{ass:quadratic-improvement}
There is $\mu_B>0$ such that
$\delta_s^\pi\leq\frac{1}{2\mu_B}\bar r_s(\pi)^2$
for every relevant state and policy.
\end{assumption}

\begin{theorem}[Quadratic performance conversion]
\label{thm:quadratic-performance}
Suppose Assumption~\ref{ass:quadratic-improvement} holds and
$\cC_{\mathrm{PL}} =\left\| \frac{d_\nu^{\pi^\star}}{(d_\mu^\pi)^2} \right\|_\infty <\infty.$
Then
$
J_\nu(\pi^\star)-J_\nu(\pi)
\leq
\frac{1-\gamma}{2\mu_B}\cC_{\mathrm{PL}}
R_{\mathrm{pol}}(\pi)^2.
\label{eq:quadratic-performance-bound}
$
If $G\leq1$, the right-hand side is at most
$\frac{1-\gamma}{2\mu_B}\cC_{\mathrm{PL}}C_{\mathrm{stat}}^2G.$
\end{theorem}

A uniform statewise action gap provides a concrete sufficient condition. If every suboptimal action satisfies $\max_bQ^\pi(s,b)-Q^\pi(s,a)\ge\Delta>0$ on the policy set of interest, then Assumption~\ref{ass:quadratic-improvement} holds with $\mu_B=\Delta/4$. Appendix~\ref{app:performance} proves this implication.

\begin{theorem}[RL-specific stronger conclusion]
\label{thm:l5-summary}
Suppose a random companion iterate satisfies the finite-time squared-KKT bound in Theorem~\ref{thm:finite-time}. Under a uniform occupancy-mismatch bound, its expected policy gap is $O(\sqrt{\E G})$. Hence $G\to0$ implies convergence to globally optimal policy performance and every covered exact KKT point is globally optimal. If the statewise quadratic Bellman-improvement condition and the stronger occupancy ratio bound hold, the expected policy gap is $O(\E G)$.
\end{theorem}

\section{Extensions to Other RL Formulations}
\label{sec:extensions}
The preceding theory raises a structural question: is multiplier stability specific to the tabular Bellman resolvent, or does it follow from a broader property of discounted RL operators? We show that the same mechanism persists across nonlinear policy parameterizations, projected Bellman equations, and explicit occupancy formulations whenever the induced operator is uniformly invertible and stable. In each case, the proof follows the same chain:
\emph{operator invertibility $\rightarrow$ explicit dual representation $\rightarrow$ multiplier control $\rightarrow$ Lyapunov descent}.
The stochastic occupancy formulation further shows that this principle survives empirical operator perturbations when the same structure-preserving operator is shared across primal and dual updates.

\subsection{Smooth Nonlinear Policies}
The operator-stability mechanism does not depend on direct tabular policy
coordinates. Let $\pi_\theta$ be smooth, define
$R(\theta,V)=M_\theta V-r_\theta$, and suppose
$\|M_\theta-M_{\theta'}\|\le L_M\|\theta-\theta'\|$
and $b_A(\theta)=\phi(\theta)+c^\trans M_\theta^{-1}r_\theta$
is lower bounded.\textbf{~\ref{lem:l6-euclidean-resolvent}} gives $\|M_\theta^{-1}\|_2\le\sqrt{|\cS|}/(1-\gamma)$. \textbf{~\ref{lem:l6-nonlinear-inverse}} gives $\|M_\theta^{-1}-M_{\theta'}^{-1}\|\le\kappa_M^2L_M\|\theta-\theta'\|$. \textbf{~\ref{lem:l6-exact-policy-descent}} gives exact parameter-step descent $\cL_\beta(\theta^k,V^k,\lambda^k)-\cL_\beta(\theta^{k+1},V^k,\lambda^k)\ge(2\eta_\theta)^{-1}\|\Delta\theta_{k+1}\|^2$. \textbf{~\ref{lem:l6-value-dual}} gives $\lambda^{k+1}=-M_{k+1}^{-\trans}(c+\eta_V^{-1}\Delta V_{k+1})$. \textbf{~\ref{lem:l6-dual-increment}} gives $\|\Delta\lambda_{k+1}\|^2\le A_A\|\Delta\theta_{k+1}\|^2+B_A\|\Delta V_{k+1}\|^2+B_A\|\Delta V_k\|^2$. \textbf{~\ref{lem:l6-lyap-lower}} gives $\Phi_k^A\ge b_A^{\inf}+(\beta/4)\|R_k\|^2+(\tau-\kappa_M^2/(\beta\eta_V^2))\|\Delta V_k\|^2$. \textbf{~\ref{thm:l6-a1}} shows that the corresponding positive coefficient conditions imply $\Delta\theta_k,\Delta V_k,\Delta\lambda_k,R_k\to0$ and KKT convergence.

For a computable proximal-gradient parameter update, define the smooth augmented term $h_k(\theta)$. \textbf{~\ref{lem:l6-pg-descent}} gives $\cL_\beta(\theta^k,V^k,\lambda^k)-\cL_\beta(\theta^{k+1},V^k,\lambda^k)\ge(1/(2\eta_{\theta,k})-L_{\theta,k}/2)\|\Delta\theta_{k+1}\|^2$. \textbf{~\ref{thm:l6-a2}} shows that the margin $1/(2\eta_{\theta,k})-L_{\theta,k}/2-A_A/\beta\ge\alpha_\theta>0$ yields Lyapunov descent and $\min_{k\le T}G_k^A=O(T^{-1})$. The step sizes also satisfy $\eta_{\theta,k}\ge\underline\eta>0$, with uniformly bounded $L_{\theta,k}$, $D_\theta M_\theta$, and $D_\theta r_\theta$. A uniform smoothness bound and the restarted backtracking rule in Appendix~\ref{app:nonlinear-policy} ensure these conditions.

\subsection{Projected Bellman Equations}
The same mechanism also survives value-function approximation. What is required for primal--dual stability is an invertible projected Bellman operator, rather than an exact state-value representation. Let $V_w=\Xi w$, fix positive diagonal $D$, and write $A_\theta=\Xi^\trans DM_\theta\Xi$, $b_\theta=\Xi^\trans Dr_\theta$. Assume $A_\theta$ is uniformly invertible and Lipschitz, with the transformed regularity and step-size conditions of Theorem~\ref{thm:l6-projected}. \textbf{~\ref{lem:l6-projected-dual}} gives the projected value-dual identity $\nu^{k+1}=-A_{k+1}^{-\trans}(c_w+\eta_w^{-1}\Delta w_{k+1})$. Theorem\textbf{~\ref{thm:l6-projected}} gives projected-KKT convergence and an $O(T^{-1})$ minimum squared residual. If the projected Bellman map contracts with modulus $q<1$, \textbf{~\ref{lem:l6-projected-value}} gives $\|\bar V^\theta-V^\theta\|\le\|(I-\Pi_D)V^\theta\|/(1-q)$. Proposition\textbf{~\ref{prop:l6-projected-counterexample}} shows that an $O(\eps)$ value error can coexist with an $O(\sqrt\eps)$ parameter-derivative error, so value approximation alone cannot control the policy gradient. \textbf{~\ref{lem:l6-adjoint-gradient}} instead gives $\|D_\theta R(\theta,V^\theta)^\trans\lambda^\theta-D_\theta R(\theta,\bar V^\theta)^\trans\widehat\lambda^\theta\|\le C_V\|V^\theta-\bar V^\theta\|+C_\lambda\|\lambda^\theta-\widehat\lambda^\theta\|$, while \textbf{~\ref{lem:l6-dual-projection}} gives $\|\lambda^\theta-\widehat\lambda^\theta\|\le C_{\rm dual}\inf_{z\in\operatorname{range}(D\Xi)}\|\lambda^\theta-z\|$.

\subsection{Explicit Occupancy Measures}
More importantly, the stability principle is not tied to the Bellman value formulation itself. Discounted occupancy flow provides a second RL operator with the same algebraic role. Let $x=(q,d)$, $b=((1-\gamma)\rho,0)$, and $K_\theta=\left[\begin{smallmatrix}-\gamma P^\trans&I\\I&-\Pi_\theta\end{smallmatrix}\right]$. \textbf{~\ref{lem:l6-occ-invertible}} proves $K_\theta$ invertible and $\|K_\theta^{-1}\|_{1,\oplus}\le2/(1-\gamma)$. Proposition\textbf{~\ref{prop:l6-occ-equivalence}} shows that $K_\theta x=b$ has the unique nonnegative solution $x^\theta=(q^\theta,d^\theta)$. Theorem\textbf{~\ref{thm:l6-occ-kkt-equivalence}} proves KKT equivalence between the constrained occupancy problem and the reduced policy objective. \textbf{~\ref{lem:l6-occ-dual}} gives the occupancy-dual identity $y^{k+1}=-K_{k+1}^{-\trans}(c_x+\eta_x^{-1}\Delta x_{k+1})$. \textbf{~\ref{thm:l6-occ-convergence}} gives deterministic occupancy-KKT convergence under the regularity, positive step-size lower bound, and strict margins in Theorem~\ref{thm:l6-occ-convergence}. Lemma\textbf{~\ref{lem:l6-occ-error}} gives the error bound $\|x-x^\theta\|\le\kappa_K\|K_\theta x-b\|$.

\subsection{Stochastic occupancy operator}
The stochastic occupancy formulation provides a second test of the structure-preserving principle from Section~4. The occupancy gradient contains $\widehat K^\trans\widehat K$, and sharing the empirical operator preserves the value--dual identity. We therefore construct one empirical transition matrix per iteration and share the induced $\widehat K_k$ across the policy, occupancy, and dual steps. \textbf{~\ref{lem:l6-regenerative}} proves a regenerative estimator with $\E_k\|\widehat P_k-P\|_F^2\le C_P/m_k$ and $\Pr_k(\|\widehat P_k-P\|_F\ge t)\le2|\cS|^2|\cA|e^{-c_Pm_kt^2}$. \textbf{~\ref{lem:l6-empirical-occ-invertible}} shows every empirical occupancy operator is automatically invertible with a deterministic uniform inverse bound, and $\|\widehat K_k^{-1}-K_\theta^{-1}\|\le\bar\kappa_K\kappa_K\delta_k$. \textbf{~\ref{lem:l6-empirical-dual}} gives $y^{k+1}=-\widehat K_k(\theta^{k+1})^{-\trans}(c_x+\eta_x^{-1}\Delta x_{k+1})$. \textbf{~\ref{lem:l6-empirical-dual-increment}} bounds $\Delta y_{k+1}$ by $\Delta\theta_{k+1}$, $\Delta x_{k+1}$, $\Delta x_k$, $\delta_k$, and $\delta_{k-1}$.

\textbf{~\ref{lem:l6-stochastic-lower}} gives the perturbed true Lyapunov lower bound $\Phi_k\ge B_C^{\inf}+(\beta/4)\|F_k\|^2+q_x\|\Delta x_k\|^2-C_0\delta_{k-1}^2$. Defining $\Psi_k=\Phi_k-B_C^{\inf}+1+C_0\delta_{k-1}^2$, \textbf{~\ref{lem:l6-policy-gradient-perturbation}} proves that the direct policy-gradient perturbation vanishes by block-row orthogonality, and hence bounds the policy-gradient perturbation by $|\langle\widehat g_{\theta,k}-g_{\theta,k},\Delta\theta_{k+1}\rangle|\le\epsilon\|\Delta\theta_{k+1}\|^2+C_\epsilon\delta_k^2\Psi_k$. Under the uniform bounds and parameter conditions in Appendix~\ref{app:occupancy-stochastic}, \textbf{~\ref{lem:l6-stochastic-descent}} gives, for $k\ge1$, the one-step recursion $\E_k\Psi_{k+1}\le(1+a/m_k)\Psi_k-c_\theta\E_k\|\Delta\theta_{k+1}\|^2-c_x\E_k\|\Delta x_{k+1}\|^2-c_x'\|\Delta x_k\|^2+b/m_k$. \textbf{~\ref{lem:l6-stochastic-kkt}} controls the true squared KKT residual by the increment energy plus $m_k^{-1}\Psi_k+m_k^{-1}+\delta_{k-1}^2$.

\textbf{~\ref{thm:l6-stochastic-finite}} Let $A_T=\sum_{k<T}m_k^{-1}$. Then $\sup_{1\le k\le T}\E\Psi_k\le e^{C_1A_T}(C_{\rm init}+C_2A_T)$ and, for an independent uniform output $K\in\{1,\ldots,T\}$, $\E G_K^{\rm true}\le C_3e^{C_1A_T}(1+A_T)(T^{-1}+A_T/T)$. Corollary\textbf{~\ref{cor:l6-equal-batch}} gives $\E G_K^{\rm true}=O(T^{-1}+T/N)$ for equal batches when $N\gtrsim T^2$. \textbf{~\ref{thm:l6-stochastic-as}} If $\sum_km_k^{-1}<\infty$, then $\Psi_k$ converges, $\sum_k(\|\Delta\theta_{k+1}\|^2+\|\Delta x_{k+1}\|^2+\|\Delta x_k\|^2)<\infty$, $\delta_k\to0$, and $G_k^{\rm true}\to0$ almost surely. Corollary\textbf{~\ref{cor:l6-performance}} gives the corresponding nonlinear-policy performance conversion under an additional parameterization-specific domination condition.
\paragraph{Unified principle and rate hierarchy.}
The Bellman resolvent $M_\pi^{-1}$, the projected Bellman operator
$A_\theta^{-1}$, and the occupancy operator $K_\theta^{-1}$ are three
realizations of the same stability mechanism. In each case, uniform
invertibility converts primal optimality into an explicit dual representation,
whose perturbation stability controls multiplier motion. Under stochastic
data, sharing one empirical operator across the coupled updates preserves this
representation pathwise and converts Markov dependence, observation noise,
and bias into operator perturbations. For the Bellman formulation, this yields
the squared-KKT rate
$
\mathbb{E}[\widetilde G_{K+1}]
=
O\!\left(
T^{-1}
+
T^{-1}\sum_{k<T}m_k^{-1}
\right),
$
while the policy-performance conversion gives
$J^\star-J(\pi)=O(\sqrt G)$ under occupancy coverage and $O(G)$ under the
stronger statewise quadratic improvement condition. The stochastic occupancy
formulation exhibits the same operator-stability mechanism with its
corresponding sampling rate. Taken together, these results identify
discounted-operator invertibility as a reusable structural source of
primal--dual stability in reinforcement learning.

\section{Numerical Mechanism Validation}
\label{sec:experiments}

We validate the structural predictions on eight independently generated finite discounted MDPs, with five trajectory seeds per instance. Figure~\ref{fig:main_mechanism} summarizes shared-operator stability and finite-time scaling. Additional three-way ablations in Appendix~\ref{app:operator_consistency} show that, at 320 steps per empirical model, sharing reduces the trailing true squared KKT residual by a factor of $8.27$ relative to separate operators, while using one third of the sampling budget. Sharing between the value and dual updates preserves the value--dual identity even with a separate policy model, identifying the consistency underlying multiplier control. Under $m_k=T$, hence $N=T^2$, the mean companion squared KKT residual exhibits slope $-0.87$, consistent with the $O(T^{-1})$ dependence in Theorem~\ref{thm:finite-time}. Further statistical controls and sensitivity results appear in Appendix~\ref{app:experiments}.

\begin{figure}[t]
    \centering

    \begin{minipage}[t]{0.32\linewidth}
        \centering
        \includegraphics[width=\linewidth]
        {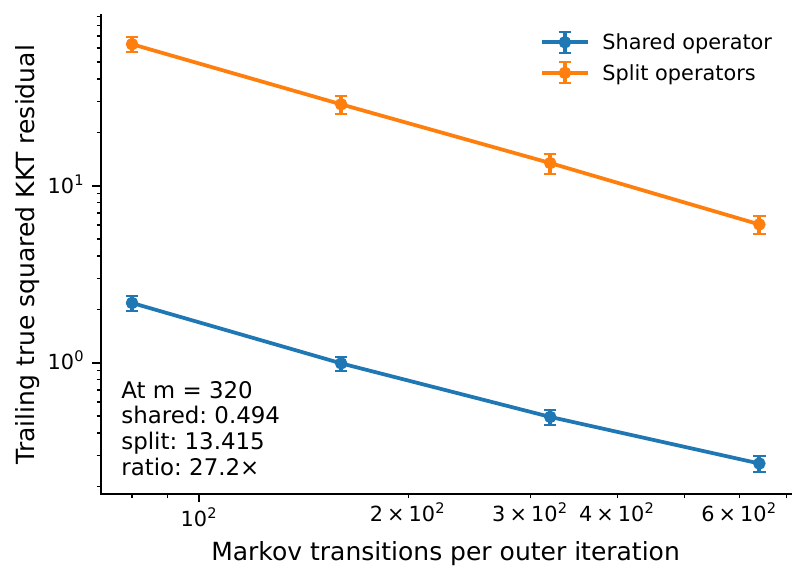}
        \vspace{-2mm}

        {\small\textbf{(a)} KKT residual}
    \end{minipage}
    \hfill
    \begin{minipage}[t]{0.32\linewidth}
        \centering
        \includegraphics[width=\linewidth]
        {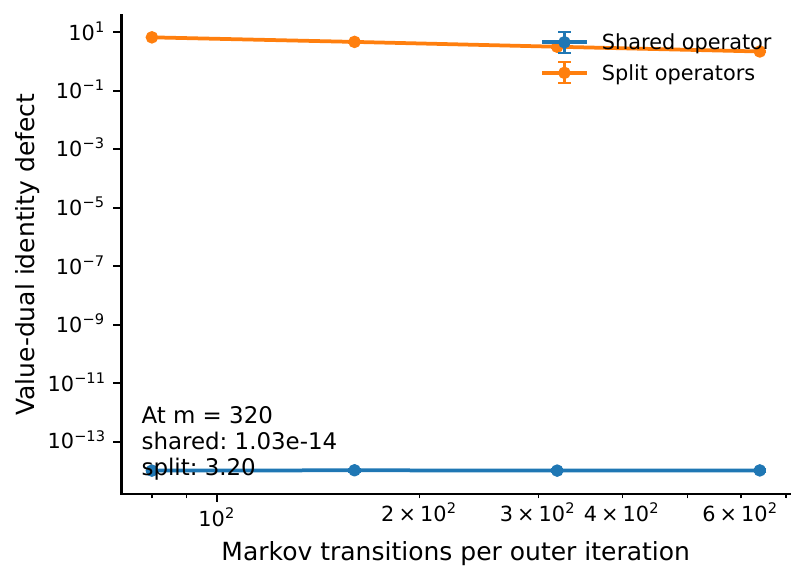}
        \vspace{-2mm}

        {\small\textbf{(b)} Value--dual identity}
    \end{minipage}
    \hfill
    \begin{minipage}[t]{0.32\linewidth}
        \centering
        \includegraphics[width=\linewidth]
        {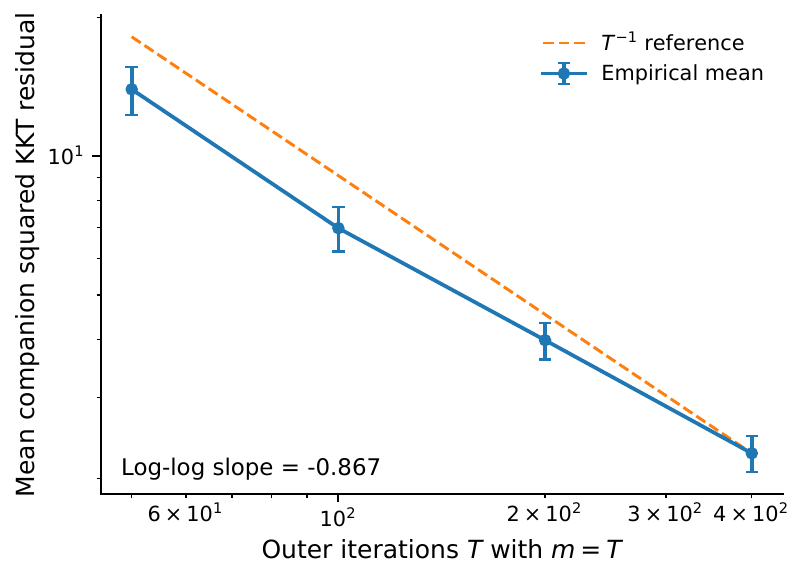}
        \vspace{-2mm}

        {\small\textbf{(c)} Finite-time scaling}
    \end{minipage}

    \caption{
    \textbf{Numerical validation of the operator-stability mechanism.}
    \textbf{(a)} A shared empirical Bellman operator yields substantially smaller true squared KKT residuals than an equal-budget split-operator construction.
    \textbf{(b)} Sharing the operator preserves the empirical value--dual identity to numerical precision, whereas inconsistent blockwise operators break this identity.
    \textbf{(c)} With $m_k=T$ and $N=T^2$, the companion squared KKT residual exhibits slope $-0.87$, close to the $T^{-1}$ scaling predicted by Theorem~\ref{thm:finite-time}.
    }
    \label{fig:main_mechanism}
\end{figure}

\section{Conclusion}
We established a Markovian nonconvex ADMM theory in which discounted RL operator structure
supplies primal-dual stability. The Bellman resolvent replaces the classical smooth-block mechanism,
the structure-preserving empirical Bellman model lifts the argument to controlled Markov and
stochastic data, the companion iterate yields finite-time KKT complexity, and occupancy geometry
converts stationarity into global policy quality. The nonlinear-policy, projected-Bellman,
deterministic occupancy, and stochastic occupancy results show that this mechanism extends beyond
the tabular Bellman formulation. Taken together, these results identify uniformly invertible
discounted RL operators as a reusable structural source of primal-dual stability.

\bibliography{ref}
\bibliographystyle{plain}

\appendix
\section{Proof Map, Conventions, and a Reusable Resolvent Lemma}
\label{app:map}

The appendix follows the same proof organization used in the reference ADMM analysis: the main text states the results and points to an appendix chain, while the appendix proves the lemmas in the order in which they are consumed by the convergence argument. In particular, the logical flow is
\[
\begin{aligned}
&\text{empirical-model control}
\Longrightarrow \text{iterate stability}
\Longrightarrow \text{dual control}\\
&\Longrightarrow \text{Lyapunov descent}
\Longrightarrow \text{KKT convergence and complexity}
\Longrightarrow \text{policy performance}.
\end{aligned}
\]
The extensions repeat the same architecture after replacing the Bellman operator by a nonlinear-policy Bellman operator, a projected Bellman matrix, or an occupancy operator.

We use $\Delta\pi_{k+1}=\pi^{k+1}-\pi^k$, $\Delta V_{k+1}=V^{k+1}-V^k$, and $\Delta\lambda_{k+1}=\lambda^{k+1}-\lambda^k$. We set $\eps_{-1}=0$ and $m_{-1}^{-1}=0$. The initialization satisfies $\pi^0\in\operatorname{dom}\psi$ and has the moments in Assumption~\ref{ass:algorithmic-margins}. Estimates obtained by subtracting consecutive value-dual identities, and their Lyapunov and residual consequences, are stated for $k\ge1$. The first update contributes a finite initialization constant to finite-time sums. The deterministic extensions use the same convention. All vector norms are Euclidean unless stated otherwise, and matrix norms are induced spectral norms. Constants denoted by $C$ may change from line to line but never depend on $k$, $T$, or the batch schedule.

\paragraph{Dependency map.}
The statistical Lemmas~\ref{lem:visit-tail}--\ref{lem:general-model-error} prove Lemma~\ref{lem:model-moments} and supply the stochastic error used by Theorems~\ref{thm:l1} and~\ref{thm:as-convergence}. Lemma~\ref{lem:inverse-stability} and Lemma~\ref{lem:moment-stability} are used in Lemma~\ref{lem:dual-control}. Lemmas~\ref{lem:surrogate-closure} and~\ref{lem:dual-control} are the two inputs to Lemma~\ref{lem:lyapunov}. Lemmas~\ref{lem:lyapunov} and~\ref{lem:companion} imply Theorem~\ref{thm:finite-time}. The policy-performance chain uses Lemmas~\ref{lem:adjoint-representation}--\ref{lem:simplex-geometry} before Theorems~\ref{thm:performance} and~\ref{thm:quadratic-performance}. Every extension result is used either to prove the next result in its subsection or to justify the operator-stability extension stated in Section~\ref{sec:extensions}.

\begin{lemma}[Inverse-map stability]
\label{lem:inverse-stability}
Let $M$ and $M'$ be invertible with $\|M^{-1}\|\le\kappa$ and $\|M'^{-1}\|\le\kappa$. Then
\[
\|M'^{-1}-M^{-1}\|\le\kappa^2\|M'-M\|.
\]
Consequently, if $\|M_{\pi'}-M_\pi\|\le L_M\|\pi'-\pi\|$, then
\[
\|M_{\pi'}^{-1}-M_\pi^{-1}\|
\le \kappa^2L_M\|\pi'-\pi\|.
\]
\end{lemma}

\begin{proof}
The resolvent identity gives
\[
M'^{-1}-M^{-1}=M'^{-1}(M-M')M^{-1}.
\]
Taking operator norms yields
\[
\|M'^{-1}-M^{-1}\|
\le \|M'^{-1}\|\,\|M'-M\|\,\|M^{-1}\|
\le \kappa^2\|M'-M\|.
\]
The policy-dependent bound follows by substituting the Lipschitz estimate for $M_\pi$.
\end{proof}

\paragraph{Role.}
Lemma~\ref{lem:inverse-stability} is the algebraic step that converts a primal policy displacement into a change of the inverse Bellman operator. It is used in Lemma~\ref{lem:dual-control}, in the nonlinear-policy extension, and in the empirical occupancy inverse perturbation bound.

\section{Controlled-Markov Empirical-Model Bounds}
\label{app:markov}

Fix an outer iteration $k$ and condition on $\cF_k$. During the batch the behavior policy is frozen, so the reset chain is time homogeneous. Let $B_t=0$ denote a reset and $B_t=1$ a genuine environment transition. For a state-action pair $(s,a)$ define
\[
Y_t^{s,a}=\one\{S_t=s,A_t=a,B_t=1\},
\qquad
N_{s,a}^{(k)}=\sum_{t=1}^{m_k}Y_t^{s,a}.
\]
The quantity $N_{s,a}^{(k)}$ is the denominator of the empirical conditional transition and reward estimators.

\begin{lemma}[Uniform effective-visitation lower tail]
\label{lem:visit-tail}
Under Assumption~\ref{ass:coverage-noise}, there are constants $p_0,c_0,C_0>0$, independent of $k$ and $m_k$, such that
\[
\Prob\!\left(
N_{s,a}^{(k)}<p_0m_k\mid\cF_k
\right)
\le C_0e^{-c_0m_k}
\]
for every state-action pair $(s,a)$.
\end{lemma}

\begin{proof}
Split the batch into $J=\lfloor m_k/2\rfloor$ nonoverlapping two-step blocks and let
\[
Q_j=\one\{B_{2j-1}=0,\ S_{2j}=s,\ A_{2j}=a,\ B_{2j}=1\}.
\]
If $Q_j=1$, then the second step of the block contributes one usable $(s,a)$ transition. Hence $N_{s,a}^{(k)}\ge\sum_{j=1}^JQ_j$. Given the history before block $j$, the reset draws a fresh state from $\rho$, the frozen behavior policy chooses action $a$, and the next reset coin equals one. Therefore
\[
\E[Q_j\mid\mathcal H_{j-1}]
=(1-\gamma)\rho(s)\bar\pi^k(a\mid s)\gamma
\ge q_\star:=\gamma(1-\gamma)\rho_{\min}\alpha>0.
\]
Set $D_j=Q_j-\E[Q_j\mid\mathcal H_{j-1}]$. Then $(D_j)$ is a martingale-difference sequence and $|D_j|\le1$. If $\sum_jQ_j<q_\star J/2$, then $\sum_jD_j\le-q_\star J/2$. Azuma--Hoeffding gives
\[
\Prob\!\left(
\sum_{j=1}^JQ_j<\frac{q_\star J}{2}
\,\middle|\,\cF_k
\right)
\le \exp\!\left(-\frac{q_\star^2J}{8}\right).
\]
Since $J\ge m_k/3$ for all sufficiently large $m_k$, the desired inequality follows with $p_0=q_\star/6$. The finitely many smaller batch sizes are absorbed by increasing $C_0$.
\end{proof}

\paragraph{Role.}
Lemma~\ref{lem:visit-tail} prevents the random denominators of all state-action conditional estimators from becoming too small. It is used in Lemma~\ref{lem:inverse-count} and the proof of Lemma~\ref{lem:model-moments}.

\begin{lemma}[Inverse moments of random visit counts]
\label{lem:inverse-count}
For every fixed $q\ge1$ there exists $C_q<\infty$ such that
\[
\E\!\left[
(N_{s,a}^{(k)})^{-q}\one\{N_{s,a}^{(k)}>0\}
\mid\cF_k
\right]
\le \frac{C_q}{m_k^q}.
\]
\end{lemma}

\begin{proof}
Write $N=N_{s,a}^{(k)}$ and split the expectation over $\{N\ge p_0m_k\}$ and $\{1\le N<p_0m_k\}$. On the first event, $N^{-q}\le(p_0m_k)^{-q}$. On the second event, $N^{-q}\le1$ and Lemma~\ref{lem:visit-tail} gives probability at most $C_0e^{-c_0m_k}$. Therefore
\[
\E[N^{-q}\one\{N>0\}\mid\cF_k]
\le (p_0m_k)^{-q}+C_0e^{-c_0m_k}.
\]
For fixed $q$, $e^{-c_0m}\le C'_qm^{-q}$ for all integers $m\ge1$. Combining the two terms proves the claim.
\end{proof}

\paragraph{Role.}
Lemma~\ref{lem:inverse-count} is used to convert martingale moment bounds with a random sample count into deterministic orders $m_k^{-1}$ and $m_k^{-2}$.

\begin{proof}[Proof of Lemma~\ref{lem:model-moments}]
We prove the transition and reward parts separately.

\textbf{Part (a): transition rows.}
For fixed $(s,a)$ define
\[
Z_t^{s,a}=Y_t^{s,a}\big(e_{S_{t+1}}-P(\cdot\mid s,a)\big).
\]
Because only non-reset transitions are used, $Y_t^{s,a}$ is measurable before $S_{t+1}$ is generated and, on $Y_t^{s,a}=1$, the next state is drawn from the true environment row. Hence $\E[Z_t^{s,a}\mid\mathcal H_t]=0$. The increments are bounded. The standard Burkholder--Rosenthal fourth-moment inequality for finite-dimensional martingales yields
\[
\E\!\left[
\left\|\sum_{t=1}^{m_k}Z_t^{s,a}\right\|^4
\,\middle|\,\cF_k
\right]
\le C_Zm_k^2.
\]
On the good-count event $\mathcal E_{s,a}^{(k)}=\{N_{s,a}^{(k)}\ge p_0m_k\}$,
\[
\widehat P_k(\cdot\mid s,a)-P(\cdot\mid s,a)
=\frac{1}{N_{s,a}^{(k)}}\sum_{t=1}^{m_k}Z_t^{s,a},
\]
so
\[
\E\!\left[
\|\widehat P_k(\cdot\mid s,a)-P(\cdot\mid s,a)\|^4
\one_{\mathcal E_{s,a}^{(k)}}
\,\middle|\,\cF_k
\right]
\le \frac{C}{m_k^2}.
\]
On the bad-count event two probability rows are uniformly bounded in norm, while Lemma~\ref{lem:visit-tail} gives an exponentially small probability. Thus the bad-event contribution is also $O(m_k^{-2})$. Consequently
\[
\E\!\left[
\|\widehat P_k(\cdot\mid s,a)-P(\cdot\mid s,a)\|^4
\mid\cF_k
\right]
\le \frac{C_P}{m_k^2}.
\]

\textbf{Part (b): reward rows.}
Write $m=m_k$, $I_t=Y_t^{s,a}$, $N=\sum_t I_t$, and
$S_m=\sum_{t=1}^m I_t\zeta_{k,t}$. The variables $I_t$ are predictable for the pre-observation filtration $\cG_{k,t}$. Thus $I_t\zeta_{k,t}$ is a martingale difference. Work throughout conditionally on $\cF_k$. The martingale fourth-moment inequality gives
\[
\E_k|S_m|^4\le C M_4 m^2.
\]
Let $p_\star$ be the visitation constant from Lemma~\ref{lem:visit-tail} and set $p=p_\star/4$. On $\{N\ge pm\}$,
\[
\E_k\left[|S_m/N|^4\one\{N\ge pm\}\right]
\le (pm)^{-4}\E_k|S_m|^4\le C M_4m^{-2}.
\]
To control small counts without conditioning on the full trajectory, split the batch into two contiguous halves of lengths $\lfloor m/2\rfloor$ and $\lceil m/2\rceil$. Let $N_1,N_2$ be their usable visit counts and let $\mathcal H$ contain all observations in the first half. For $m\ge4$, each half has length at least $m/3$, so Lemma~\ref{lem:visit-tail}, applied from any starting history, gives
\[
\Prob_k(N_1<pm)\le Ce^{-cm},\qquad
\Prob(N_2<pm\mid\mathcal H)\le Ce^{-cm}.
\]
On $\mathcal B=\{1\le N<pm\}$, Jensen's inequality gives
\[
|S_m/N|^4\le N^{-1}\sum_t I_t|\zeta_{k,t}|^4
\le\sum_t I_t|\zeta_{k,t}|^4.
\]
For a first-half index $t$, $I_t|\zeta_{k,t}|^4$ is $\mathcal H$-measurable and $\mathcal B\subset\{N_2<pm\}$. Hence
\[
\E_k[I_t|\zeta_{k,t}|^4\one_{\mathcal B}]
\le Ce^{-cm}\E_k[I_t|\zeta_{k,t}|^4]\le CM_4e^{-cm}.
\]
For a second-half index $t$, the event $\{N_1<pm\}$ belongs to $\cG_{k,t}$. Since $\mathcal B\subset\{N_1<pm\}$, the conditional moment assumption gives the same bound. Summing over $t$ yields
\[
\E_k[|S_m/N|^4\one_{\mathcal B}]
\le CM_4m e^{-cm}\le C M_4m^{-2}.
\]
For $N=0$, the bounded fallback reward contributes $Ce^{-cm}$. The finitely many $m<4$ are absorbed in the constant. Therefore
$\E_k|\widehat r_k(s,a)-r(s,a)|^4\le C/m_k^2$.
This argument allows the reward and next state at a given transition to be dependent.

\textbf{Part (c): policy-uniform model error.}
Because the state-action space is finite and $P_\pi,r_\pi$ are convex combinations of primitive rows,
\[
\eps_k
\le C_\Theta\left(
\max_{s,a}\|\widehat P_k(\cdot\mid s,a)-P(\cdot\mid s,a)\|
+
\max_{s,a}|\widehat r_k(s,a)-r(s,a)|
\right).
\]
Finite-dimensional norm equivalence and the bounds above give
$\E[\eps_k^4\mid\cF_k]\le C_4/m_k^2$. Conditional Jensen then yields
$\E[\eps_k^2\mid\cF_k]\le (\E[\eps_k^4\mid\cF_k])^{1/2}\le C_2/m_k$.
\end{proof}

\paragraph{Role.}
Lemma~\ref{lem:model-moments} supplies the deterministic $O(m_k^{-1})$ error terms in the expected Lyapunov and finite-time complexity arguments. Its fourth-moment part is exactly what permits Cauchy--Schwarz with the random iterates.

\begin{lemma}[General stochastic empirical-model error]
\label{lem:general-model-error}
Suppose an additional observation perturbation is decomposed as $\xi_{k,t}=b_{k,t}+\zeta_{k,t}$ with $\E[\zeta_{k,t}\mid\cG_{k,t}]=0$, $\E[\|\zeta_{k,t}\|^2\mid\cG_{k,t}]\le\sigma_k^2$, and the average bias in every visited state-action reward row bounded in magnitude by $b_k$. Then there are deterministic constants $C_0,C_\sigma,C_b$ such that
\[
\E[\eps_k^2\mid\cF_k]
\le \frac{C_0+C_\sigma\sigma_k^2}{m_k}+C_bb_k^2.
\]
Consequently,
\[
\sum_k\frac{1+\sigma_k^2}{m_k}<\infty,
\qquad
\sum_kb_k^2<\infty
\quad\Longrightarrow\quad
\sum_k\eps_k^2<\infty
\ \text{almost surely}.
\]
\end{lemma}

\begin{proof}
The transition part contributes $C/m_k$. For each reward row, use the predictable martingale sum $S_m=\sum_{t=1}^m I_t\zeta_{k,t}$ from Part (b) above. Orthogonality of martingale differences gives
$\E_k|S_m|^2\le m\sigma_k^2$. On $N\ge pm$ the squared average therefore contributes at most $C\sigma_k^2/m$. On $1\le N<pm$, repeat the two-half argument with exponent two. The first-half terms are controlled by the conditional lower-tail bound for the second half. The second-half terms are controlled by $\E[|\zeta_{k,t}|^2\mid\cG_{k,t}]\le\sigma_k^2$ and the first-half lower-tail event. Their total is at most $C\sigma_k^2m e^{-cm}\le C\sigma_k^2/m$. The bounded zero-count fallback contributes $C/m$. Adding the squared row-average bias and summing over the finite row set proves the bound.

Taking expectations and summing under the stated schedule gives $\sum_k\E\eps_k^2<\infty$. Since the summands are nonnegative, Tonelli's theorem implies $\sum_k\eps_k^2<\infty$ almost surely.
\end{proof}

\paragraph{Role.}
Lemma~\ref{lem:general-model-error} is the statistical input to Theorem~\ref{thm:as-convergence}. The summability condition controls both observation noise and row-average bias.

\section{Bellman-Resolvent Stability, Lyapunov Descent, and KKT Complexity}
\label{app:bellman-proofs}

\subsection{A priori iterate stability}

\begin{proof}[Proof of Lemma~\ref{lem:moment-stability}]
The empirical value first-order condition and the dual update give the exact identity
\begin{equation}
\lambda^{k+1}
=-\widehat M_{k+1}^{-\trans}
\left(c+\eta_V^{-1}\Delta V_{k+1}\right).
\label{eq:app-dual-rep}
\end{equation}
Hence
\begin{equation}
\|\lambda^{k+1}\|
\le \bar\kappa\|c\|
+\frac{\bar\kappa}{\eta_V}
\left(\|V^{k+1}\|+\|V^k\|\right).
\label{eq:app-lambda-rec}
\end{equation}
The dual update also gives
\[
\widehat M_{k+1}V^{k+1}
=\widehat r_{\pi^{k+1},k}+\beta^{-1}\Delta\lambda_{k+1},
\]
so
\begin{equation}
\|V^{k+1}\|
\le \bar\kappa\|\widehat r_{\pi^{k+1},k}\|
+\frac{\bar\kappa}{\beta}
\left(\|\lambda^{k+1}\|+\|\lambda^k\|\right).
\label{eq:app-value-rec}
\end{equation}
Lemma~\ref{lem:model-moments}, bounded true rewards, and the conditional fourth-moment assumption imply
$R_4:=\sup_k(\E\|\widehat r_{\pi^{k+1},k}\|^4)^{1/4}<\infty$.
Set
\[
X_k=(\E\|V^k\|^4)^{1/4},
\quad
Y_k=(\E\|\lambda^k\|^4)^{1/4},
\quad
a=\frac{\bar\kappa}{\eta_V},
\quad
b=\frac{\bar\kappa}{\beta}.
\]
Minkowski's inequality applied to \eqref{eq:app-lambda-rec}--\eqref{eq:app-value-rec} gives
\[
Y_{k+1}\le \bar\kappa\|c\|+a(X_{k+1}+X_k),
\qquad
X_{k+1}\le \bar\kappa R_4+b(Y_{k+1}+Y_k).
\]
Eliminating the next-iterate variables yields
\[
\begin{bmatrix}X_{k+1}\\Y_{k+1}\end{bmatrix}
\le
\mathsf A
\begin{bmatrix}X_k\\Y_k\end{bmatrix}
+d,
\qquad
\mathsf A=
\frac{1}{1-ab}
\begin{bmatrix}ab&b\\a&ab\end{bmatrix},
\]
where $d$ is a fixed nonnegative vector. Let $z=\sqrt{ab}=\bar\kappa/\sqrt{\beta\eta_V}$. The largest eigenvalue of $\mathsf A$ is
\[
\rho(\mathsf A)=\frac{ab+\sqrt{ab}}{1-ab}=\frac{z}{1-z}.
\]
Assumption~\ref{ass:algorithmic-margins} implies $\beta\eta_V>4\bar\kappa^2$, hence $z<1/2$ and $\rho(\mathsf A)<1$. Iterating the recursion gives
\[
\begin{bmatrix}X_k\\Y_k\end{bmatrix}
\le \mathsf A^k
\begin{bmatrix}X_0\\Y_0\end{bmatrix}
+\sum_{j=0}^{k-1}\mathsf A^jd.
\]
The matrix series converges. Thus $\sup_kX_k<\infty$ and $\sup_kY_k<\infty$, which is the claimed fourth-moment stability.
\end{proof}

\paragraph{Role.}
This lemma is proved before any Lyapunov decrease. It removes the potential circularity that would arise if bounded multipliers were assumed inside the empirical-to-true perturbation estimate.

\subsection{Multiplier control}

\begin{proof}[Proof of Lemma~\ref{lem:dual-control}]
Let $M_{k+1}=M_{\pi^{k+1}}$. Replacing $\widehat M_{k+1}$ in the value-dual identity by the true matrix gives
\begin{equation}
c+M_{k+1}^\trans\lambda^{k+1}
+\eta_V^{-1}\Delta V_{k+1}=r_{k+1},
\qquad
r_{k+1}=(M_{k+1}-\widehat M_{k+1})^\trans\lambda^{k+1}.
\label{eq:app-perturbed-value-dual}
\end{equation}
Since $\|M_{k+1}-\widehat M_{k+1}\|\le C\eps_k$, conditional Cauchy--Schwarz, Lemma~\ref{lem:model-moments}, and Lemma~\ref{lem:moment-stability} give
\begin{align}
\E\|r_{k+1}\|^2
&\le C\E[\eps_k^2\|\lambda^{k+1}\|^2]\\
&\le C(\E\eps_k^4)^{1/2}(\E\|\lambda^{k+1}\|^4)^{1/2}
\le \frac{C_r}{m_k}.
\label{eq:app-r-second}
\end{align}
Solve \eqref{eq:app-perturbed-value-dual} for $\lambda^{k+1}$ and subtract the corresponding identity at iteration $k$:
\begin{align}
\Delta\lambda_{k+1}
={}&-(M_{k+1}^{-\trans}-M_k^{-\trans})c
-\eta_V^{-1}M_{k+1}^{-\trans}\Delta V_{k+1}
+\eta_V^{-1}M_k^{-\trans}\Delta V_k
\nonumber\\
&+M_{k+1}^{-\trans}r_{k+1}
-M_k^{-\trans}r_k.
\label{eq:app-five-terms}
\end{align}
Lemma~\ref{lem:inverse-stability} gives
$\|(M_{k+1}^{-\trans}-M_k^{-\trans})c\|
\le \bar\kappa^2L_M\|c\|\|\Delta\pi_{k+1}\|$.
Using $\|\sum_{i=1}^5x_i\|^2\le5\sum_{i=1}^5\|x_i\|^2$ in \eqref{eq:app-five-terms},
\begin{align}
\|\Delta\lambda_{k+1}\|^2
\le{}&5\bar\kappa^4L_M^2\|c\|^2\|\Delta\pi_{k+1}\|^2
+\frac{5\bar\kappa^2}{\eta_V^2}
\left(\|\Delta V_{k+1}\|^2+\|\Delta V_k\|^2\right)
\nonumber\\
&+5\bar\kappa^2(\|r_{k+1}\|^2+\|r_k\|^2).
\end{align}
Taking expectations and using \eqref{eq:app-r-second} proves \eqref{eq:dual-control} with the stated $C_\pi$ and $C_V$ and a suitable $C_\eps$.
\end{proof}

\paragraph{Role.}
Lemma~\ref{lem:dual-control} is the replacement for the classical smooth-block multiplier estimate. It is the term that allows the positive primal decrease to dominate the dual ascent.

\subsection{Structure-preserving surrogate closure}

\begin{proof}[Proof of Lemma~\ref{lem:surrogate-closure}]
Fix a sample path for which $\sum_k\eps_k^2<\infty$ and the a priori iterate bounds hold. The empirical and true residuals satisfy
\[
e_k(\pi,V)=\widehat R_k(\pi,V)-R(\pi,V)
=(\widehat M_{\pi,k}-M_\pi)V-(\widehat r_{\pi,k}-r_\pi).
\]
On the bounded iterate set there is $C_e<\infty$ such that
\begin{equation}
\|e_k(\pi,V)\|+\|D_{\pi,V}e_k(\pi,V)\|
\le C_e\eps_k.
\label{eq:app-path-operator}
\end{equation}
We verify the four relations in the statement.

\textbf{Part (a): true primal decrease.}
Let $\Psi_k=\widehat\cL_{\beta,k}-\cL_\beta$. Expanding,
\[
\Psi_k=\langle\lambda,e_k\rangle
+\beta\langle R,e_k\rangle
+\frac\beta2\|e_k\|^2.
\]
Equation \eqref{eq:app-path-operator} and bounded iterates imply
$\|D_{\pi,V}\Psi_k\|\le C\eps_k$ along the two primal line segments. Optimality of the empirical policy subproblem gives
\[
\widehat\cL_{\beta,k}(\pi^k,V^k,\lambda^k)
-\widehat\cL_{\beta,k}(\pi^{k+1},V^k,\lambda^k)
\ge \frac{1}{2\eta_\pi}\|\Delta\pi_{k+1}\|^2.
\]
The perturbation difference is at most $C\eps_k\|\Delta\pi_{k+1}\|$, which by Young's inequality is bounded by $(4\eta_\pi)^{-1}\|\Delta\pi_{k+1}\|^2+C\eps_k^2$. Thus the true policy step decreases by at least $(4\eta_\pi)^{-1}\|\Delta\pi_{k+1}\|^2-C\eps_k^2$. The same argument for the value step gives $(4\eta_V)^{-1}\|\Delta V_{k+1}\|^2-C\eps_k^2$. Adding the two proves the first relation with $a_\pi=1/(4\eta_\pi)$ and $a_V=1/(4\eta_V)$.

\textbf{Part (b): perturbed value-dual identity.}
The exact empirical identity is the value-dual identity. Hence
\[
c+M_{k+1}^\trans\lambda^{k+1}+\eta_V^{-1}\Delta V_{k+1}
=(M_{k+1}-\widehat M_{k+1})^\trans\lambda^{k+1}=:r_{k+1}.
\]
The pathwise iterate bound and \eqref{eq:app-path-operator} give $\|r_{k+1}\|\le C\eps_k$.

\textbf{Part (c): inexact true-residual multiplier update.}
Since $\Delta\lambda_{k+1}=\beta\widehat R_k(\pi^{k+1},V^{k+1})$,
\[
\Delta\lambda_{k+1}=\beta R_{k+1}+d_{k+1},
\qquad
d_{k+1}=\beta e_k(\pi^{k+1},V^{k+1}),
\]
and $\|d_{k+1}\|\le C\eps_k$ by \eqref{eq:app-path-operator}.

\textbf{Part (d): true KKT residual.}
The empirical policy Fermat condition is
\begin{align*}
0\in{}&\partial\psi(\pi^{k+1})
+\widehat J_k(\pi^{k+1},V^k)^\trans
\big[\lambda^k+\beta\widehat R_k(\pi^{k+1},V^k)\big]
+\eta_\pi^{-1}\Delta\pi_{k+1}
.
\end{align*}
The bracket differs from $\lambda^{k+1}$ by
$\beta\widehat M_{k+1}(V^k-V^{k+1})$, and the true and empirical Jacobians differ by $O(\eps_k)$. The tabular Bellman Jacobian is affine in $V$, so replacing $V^k$ by $V^{k+1}$ costs $O(\|\Delta V_{k+1}\|)$. Hence the true policy residual is at most
$C(\|\Delta\pi_{k+1}\|+\|\Delta V_{k+1}\|+\eps_k)$.
The value residual follows from Part (b), and feasibility follows from Part (c). Combining the three components gives the stated true KKT bound.
\end{proof}

\paragraph{Role.}
This lemma is the interface between stochastic RL data and deterministic primal-dual analysis. It packages all sampling effects into one scalar operator error $\eps_k$.

\subsection{Expected Lyapunov descent and lower bound}

\begin{proof}[Proof of Lemma~\ref{lem:lyapunov}]
We prove the descent and lower bound separately.

\textbf{Part (a): expected primal decrease.}
The empirical policy and value subproblems give the same comparison inequalities as above. To justify the moment transfer, define
\[
U_k=1+\|V^k\|+\|\lambda^k\|,\quad
p_k=\max_{s,a}\|\widehat P_k(\cdot\mid s,a)-P(\cdot\mid s,a)\|,\quad
u_k=\max_{s,a}|\widehat r_k(s,a)-r(s,a)|.
\]
Then $U_k$ is $\cF_k$-measurable, $p_k$ is uniformly bounded, and Lemma~\ref{lem:model-moments} gives conditional second and fourth moments of orders $m_k^{-1}$ and $m_k^{-2}$ for both $p_k$ and $u_k$. Expanding $\Psi_k=\widehat\cL_{\beta,k}-\cL_\beta$ on the policy comparison segment, where $V=V^k$ and $\lambda=\lambda^k$, gives
\[
\sup_{\pi\in\Pi}\|D_\pi\Psi_k(\pi,V^k,\lambda^k)\|
\le C\left(p_kU_k^2+u_kU_k+u_k^2\right).
\]
Indeed, the Jacobian error is at most $C(p_k\|V^k\|+u_k)$ and the residual error has the same bound. Boundedness of $p_k$ absorbs its square. Conditioning before the batch therefore yields
\[
\E_k\sup_{\pi\in\Pi}\|D_\pi\Psi_k\|^2
\le \frac{C(1+U_k^4)}{m_k}.
\]
For the value segment, the explicit value solve gives
\[
V^{k+1}=(\beta\widehat M_{k+1}^\trans\widehat M_{k+1}+\eta_V^{-1}I)^{-1}
\left(\beta\widehat M_{k+1}^\trans\widehat r_{\pi^{k+1},k}
-\widehat M_{k+1}^\trans\lambda^k-c+\eta_V^{-1}V^k\right),
\]
so $\|V^{k+1}\|\le C(U_k+u_k)$. On the segment from $V^k$ to $V^{k+1}$,
\[
\|D_V\Psi_k\|\le C\{p_k(U_k+u_k)+u_k\},\qquad
\E_k\sup_{V\in[V^k,V^{k+1}]}\|D_V\Psi_k\|^2
\le C(1+U_k^2)/m_k.
\]
Taking expectations and applying Lemma~\ref{lem:moment-stability} controls both comparison-segment errors by $C/m_k$.
Young's inequality therefore yields
\begin{align}
\E[\cL_\beta(\pi^k,V^k,\lambda^k)-\cL_\beta(\pi^{k+1},V^k,\lambda^k)]
&\ge \frac{1}{4\eta_\pi}\E\|\Delta\pi_{k+1}\|^2-\frac{C}{m_k},
\label{eq:app-policy-dec}\\
\E[\cL_\beta(\pi^{k+1},V^k,\lambda^k)-\cL_\beta(\pi^{k+1},V^{k+1},\lambda^k)]
&\ge \frac{1}{4\eta_V}\E\|\Delta V_{k+1}\|^2-\frac{C}{m_k}.
\label{eq:app-value-dec}
\end{align}
Let $a_\pi=1/(4\eta_\pi)$ and $a_V=1/(4\eta_V)$.

\textbf{Part (b): dual ascent.}
By Lemma~\ref{lem:surrogate-closure}, $\Delta\lambda_{k+1}=\beta R_{k+1}+d_{k+1}$. Therefore
\begin{align}
&\cL_\beta(\pi^{k+1},V^{k+1},\lambda^{k+1})
-\cL_\beta(\pi^{k+1},V^{k+1},\lambda^k)\\
&\quad=\langle\Delta\lambda_{k+1},R_{k+1}\rangle
=\frac1\beta\|\Delta\lambda_{k+1}\|^2
-\frac1\beta\langle\Delta\lambda_{k+1},d_{k+1}\rangle\\
&\quad\le \frac{3}{2\beta}\|\Delta\lambda_{k+1}\|^2
+\frac{1}{2\beta}\|d_{k+1}\|^2.
\label{eq:app-dual-ascent}
\end{align}
The second moment of $d_{k+1}$ is $O(m_k^{-1})$. Substituting Lemma~\ref{lem:dual-control} into \eqref{eq:app-dual-ascent} and combining with \eqref{eq:app-policy-dec}--\eqref{eq:app-value-dec} gives
\begin{align*}
\E[\cL_\beta^k]-\E[\cL_\beta^{k+1}]
\ge{}&\left(a_\pi-\frac{3C_\pi}{2\beta}\right)\E\|\Delta\pi_{k+1}\|^2
+\left(a_V-D_V\right)\E\|\Delta V_{k+1}\|^2
-D_V\E\|\Delta V_k\|^2\\
&-C_E\left(m_k^{-1}+m_{k-1}^{-1}\right),
\end{align*}
where $D_V=3C_V/(2\beta)$. Adding $\tau\|\Delta V_k\|^2$ to the potential gives coefficients
\[
c_\pi=a_\pi-\frac{3C_\pi}{2\beta},
\quad
c_V=a_V-D_V-\tau,
\quad
c_V'=\tau-D_V.
\]
Assumption~\ref{ass:algorithmic-margins} makes all three positive. This proves \eqref{eq:expected-descent}.

\textbf{Part (c): lower bound.}
From the true perturbed value-dual identity,
$\lambda^k=-M_k^{-\trans}(c+\eta_V^{-1}\Delta V_k-r_k)$, while $V^k=M_k^{-1}(r_{\pi^k}+R_k)$. Substituting both into $\cL_\beta^k$ gives
\begin{align*}
\cL_\beta^k
={}&b(\pi^k)
-\eta_V^{-1}\langle\Delta V_k,M_k^{-1}R_k\rangle
+\langle r_k,M_k^{-1}R_k\rangle
+\frac\beta2\|R_k\|^2.
\end{align*}
Two applications of Young's inequality yield
\[
\Phi_k
\ge b_{\min}
+\frac\beta4\|R_k\|^2
+\left(\tau-\frac{2\bar\kappa^2}{\beta\eta_V^2}\right)\|\Delta V_k\|^2
-\frac{2\bar\kappa^2}{\beta}\|r_k\|^2.
\]
The parameter condition makes the coefficient of $\|\Delta V_k\|^2$ positive. Taking expectations and using \eqref{eq:app-r-second} gives
$\E\Phi_k\ge\underline\Phi-C_Lm_{k-1}^{-1}$.
\end{proof}

\paragraph{Role.}
Lemma~\ref{lem:lyapunov} is the summability engine for both the asymptotic and finite-time results.

\subsection{Controlled-Markov convergence}

\begin{proof}[Proof of Theorem~\ref{thm:l1}]
With no additional observation noise, Lemma~\ref{lem:model-moments} gives $\E\eps_k^2\le C/m_k$. Hence $\sum_km_k^{-1}<\infty$ implies $\sum_k\E\eps_k^2<\infty$, and Tonelli gives $\sum_k\eps_k^2<\infty$ almost surely.

Fix a path in this probability-one event. The empirical rewards are pathwise bounded because the true rewards are bounded and $\eps_k\to0$. The same two-dimensional recursion used in Lemma~\ref{lem:moment-stability}, now pathwise, gives bounded $V^k$ and $\lambda^k$. Lemma~\ref{lem:surrogate-closure} and the pathwise version of the Lyapunov calculation then yield
\[
\Phi_k-\Phi_{k+1}
\ge c_\pi\|\Delta\pi_{k+1}\|^2
+c_V\|\Delta V_{k+1}\|^2
+c_V'\|\Delta V_k\|^2
-C(\eps_k^2+\eps_{k-1}^2).
\]
The lower-bound calculation in Lemma~\ref{lem:lyapunov} remains valid pathwise up to a summable $O(\eps_{k-1}^2)$ term. Summing over $k$ shows
\[
\sum_k\|\Delta\pi_{k+1}\|^2<\infty,
\qquad
\sum_k\|\Delta V_{k+1}\|^2<\infty.
\]
The pathwise version of Lemma~\ref{lem:dual-control} and $\sum_k\eps_k^2<\infty$ give $\sum_k\|\Delta\lambda_{k+1}\|^2<\infty$. Thus all increments converge to zero. Lemma~\ref{lem:surrogate-closure}, Part (c), gives
$R_{k+1}=\beta^{-1}(\Delta\lambda_{k+1}-d_{k+1})\to0$.
Part (d) then gives KKT-residual convergence.

The policy space is compact and the value-dual iterates are bounded, so accumulation points exist. Fix a convergent subsequence $z^{k_j}\to z^\star$. Vanishing increments imply $z^{k_j-1}\to z^\star$. The descent inequality bounds $\Phi_k$ above, so $\phi(\pi^k)$ is bounded above on this path. Lower semicontinuity gives $\phi(\pi^\star)<\infty$.

Write $H_j(\pi)$ for the smooth part of $\widehat\cL_{\beta,k_j-1}(\pi,V^{k_j-1},\lambda^{k_j-1})$. Model-error convergence and bounded iterates imply that $H_j$ converges uniformly on $\Pi$ to a continuous function. Exact policy minimization, compared with $\pi^\star$, gives
\[
\phi(\pi^{k_j})+H_j(\pi^{k_j})
+\frac{\|\pi^{k_j}-\pi^{k_j-1}\|^2}{2\eta_\pi}
\le \phi(\pi^\star)+H_j(\pi^\star)
+\frac{\|\pi^\star-\pi^{k_j-1}\|^2}{2\eta_\pi}.
\]
Taking the upper limit and using lower semicontinuity proves
$\psi(\pi^{k_j})\to\psi(\pi^\star)$. The policy Fermat condition supplies
$w_j\in\partial\psi(\pi^{k_j})$ with
$w_j\to-J_\pi(\pi^\star,V^\star)^\trans\lambda^\star$.
The limiting subdifferential is closed along this function-value-convergent sequence. Hence
$0\in\partial\psi(\pi^\star)+J_\pi(\pi^\star,V^\star)^\trans\lambda^\star$.
The other two KKT conditions follow by continuity. This proves the accumulation-point claim without a subdifferential sum rule.
\end{proof}

\paragraph{Role.}
Theorem~\ref{thm:l1} is the first complete Markovian nonconvex ADMM convergence theorem in the paper. It validates that Bellman-resolvent stability survives controlled Markov sampling.

\subsection{Companion residual}

\begin{proof}[Proof of Lemma~\ref{lem:companion}]
The policy-subproblem Fermat condition is
\begin{align}
0\in{}&\partial\psi(\pi^{k+1})
+\widehat J_k(\pi^{k+1},V^k)^\trans
\left[\lambda^k+\beta\widehat R_k(\pi^{k+1},V^k)\right]
+\eta_\pi^{-1}\Delta\pi_{k+1}
.
\label{eq:app-policy-fermat}
\end{align}
The bracket is $\widetilde\lambda^{k+1}$. Hence the true policy residual at $\widetilde z^{k+1}$ is bounded by
\[
e_{\pi,k+1}
\le \eta_\pi^{-1}\|\Delta\pi_{k+1}\|
+\|J_\pi(\pi^{k+1},V^k)-\widehat J_k(\pi^{k+1},V^k)\|
\|\widetilde\lambda^{k+1}\|.
\]
Use the conditional variables $U_k,p_k,u_k$ from the preceding descent proof. The empirical companion multiplier satisfies $\|\widetilde\lambda^{k+1}\|\le C(U_k+u_k)$. Consequently
\[
\|(J_\pi-\widehat J_k)^\trans\widetilde\lambda^{k+1}\|
\le C(p_kU_k^2+u_kU_k+u_k^2).
\]
Conditioning on $\cF_k$ and then using the fourth-moment stability yields
\begin{equation}
\E e_{\pi,k+1}^2
\le C\E\|\Delta\pi_{k+1}\|^2+\frac{C}{m_k}.
\label{eq:app-comp-policy}
\end{equation}
The companion and final multipliers satisfy
\[
\lambda^{k+1}-\widetilde\lambda^{k+1}
=\beta\widehat M_{k+1}\Delta V_{k+1}.
\]
Substitution into the empirical value-dual identity gives
\[
c+\widehat M_{k+1}^\trans\widetilde\lambda^{k+1}
=-\left(\eta_V^{-1}I+\beta\widehat M_{k+1}^\trans\widehat M_{k+1}\right)\Delta V_{k+1}.
\]
Replacing $\widehat M_{k+1}$ by $M_{k+1}$ and applying the fourth-moment bounds yields
\begin{equation}
\E\|c+M_{k+1}^\trans\widetilde\lambda^{k+1}\|^2
\le C\E\|\Delta V_{k+1}\|^2+\frac{C}{m_k}.
\label{eq:app-comp-value}
\end{equation}
Finally,
\[
\widehat R_k(\pi^{k+1},V^k)
=\beta^{-1}\Delta\lambda_{k+1}-\widehat M_{k+1}\Delta V_{k+1}.
\]
Thus
\[
\|R(\pi^{k+1},V^k)\|
\le \beta^{-1}\|\Delta\lambda_{k+1}\|
+\|\widehat M_{k+1}\|\|\Delta V_{k+1}\|
+\|e_k(\pi^{k+1},V^k)\|.
\]
Squaring, taking expectations, and applying Lemma~\ref{lem:dual-control} controls feasibility by the three increment energies plus $m_k^{-1}+m_{k-1}^{-1}$. Adding this to \eqref{eq:app-comp-policy} and \eqref{eq:app-comp-value} proves \eqref{eq:companion-bound}.
\end{proof}

\paragraph{Role.}
The companion iterate removes endpoint products between random multipliers and the latest value increment. This is the technical step that makes a deterministic-constant fourth-moment complexity bound possible.

\subsection{Finite-time complexity}

\begin{proof}[Proof of Theorem~\ref{thm:finite-time}]
Sum \eqref{eq:expected-descent} from $k=1$ to $T-1$ and let $c_\star=\min\{c_\pi,c_V,c_V'\}$. The first update has uniformly bounded fourth moments by the explicit value solve and the empirical-model bounds. Its policy comparison with the deterministic $\pi^0\in\operatorname{dom}\psi$ also gives $\sup_{m_0}\E\Phi_1<\infty$, and the companion Fermat condition gives $\sup_{m_0}\E\widetilde G_1<\infty$. Absorb these first-step quantities into the initial constant. The lower bound in Lemma~\ref{lem:lyapunov} absorbs the terminal potential and gives constants $A_0,A_1$ such that
\[
\sum_{k=0}^{T-1}\E\!
\left[
\|\Delta\pi_{k+1}\|^2
+\|\Delta V_{k+1}\|^2
+\|\Delta V_k\|^2
\right]
\le A_0+A_1\sum_{k=0}^{T-1}m_k^{-1}.
\]
Apply Lemma~\ref{lem:companion}, sum over $k$, and divide by $T$:
\[
\frac1T\sum_{k=0}^{T-1}\E\widetilde G_{k+1}
\le \frac{A}{T}+\frac{B}{T}\sum_{k=0}^{T-1}m_k^{-1}.
\]
Since $K$ is uniform and independent, the left side equals $\E\widetilde G_{K+1}$, proving Theorem~\ref{thm:finite-time}.

For a fixed total budget $N=\sum_{k<T}m_k$, convexity of $x\mapsto1/x$ gives
\[
\sum_{k=0}^{T-1}m_k^{-1}\ge \frac{T^2}{N},
\]
with equality for equal batches $m_k=N/T$. Substitution gives the equal-batch bound in Theorem~\ref{thm:finite-time}. Taking $T\asymp\sqrt N$ yields $\E\widetilde G_{K+1}=O(N^{-1/2})$. To make the squared residual at most $\eps$, it is sufficient to choose $T=O(\eps^{-1})$ and $N=O(\eps^{-2})$.
\end{proof}

\paragraph{Role.}
Theorem~\ref{thm:finite-time} is the finite-time optimization guarantee that the policy-performance analysis later converts into a return-gap rate.

\begin{proof}[Proof of Corollary~\ref{cor:complexity}]
If $(1/T)\sum_{k<T}m_k^{-1}=O(T^{-1})$, Theorem~\ref{thm:finite-time} gives $\E\widetilde G_{K+1}=O(T^{-1})$. For fixed $N$ the equal-batch conclusion follows from Jensen's inequality as in the theorem proof. Setting $A/T\le\eps/2$ and $BT/N\le\eps/2$ gives $T=O(\eps^{-1})$ and $N=O(\eps^{-2})$.
\end{proof}

\paragraph{Role.}
The corollary translates the theorem into the conventional iteration and sample-complexity language used in the abstract and introduction.

\subsection{General stochastic observations}

\begin{proof}[Proof of Theorem~\ref{thm:as-convergence}]
Lemma~\ref{lem:general-model-error} and the summability conditions in Theorem~\ref{thm:as-convergence} imply $\sum_k\eps_k^2<\infty$ almost surely. Fix a path in this event. Since $\eps_k\to0$ and the true reward is bounded, the empirical rewards are pathwise bounded. The pathwise two-dimensional recursion in the proof of Lemma~\ref{lem:moment-stability} therefore gives bounded $V^k$ and $\lambda^k$.

Lemma~\ref{lem:surrogate-closure} and the pathwise Lyapunov calculation give
\[
\Phi_k-\Phi_{k+1}
\ge c_\pi\|\Delta\pi_{k+1}\|^2
+c_V\|\Delta V_{k+1}\|^2
+c_V'\|\Delta V_k\|^2
-C(\eps_k^2+\eps_{k-1}^2).
\]
The lower bound differs from a fixed constant by at most $C\eps_{k-1}^2$. Summation yields square summability of the policy and value increments. The pathwise five-term dual estimate then gives square summability of $\Delta\lambda_{k+1}$. Hence all increments vanish. The inexact dual update gives feasibility, and the true KKT bound in Lemma~\ref{lem:surrogate-closure} gives stationarity. Accumulation-point KKT optimality follows from the exact policy comparison and function-value convergence argument in the proof of Theorem~\ref{thm:l1}.
\end{proof}

\paragraph{Role.}
Theorem~\ref{thm:as-convergence} extends the controlled-Markov result to centered observation noise and decaying bias without imposing pathwise bounded noise.

\subsection{Exact policy subproblems}
For fixed $(V^k,\lambda^k)$, the empirical Bellman residual is affine in $\pi$. Its squared norm plus the proximal term is a strongly convex quadratic. When $\phi\equiv0$, the policy step is therefore a simplex-constrained quadratic program, separable across states. More generally, consider the nonsmooth, nonconvex regularizer
$\phi(\pi)=a\sum_s\|\pi_s\|_0$ with $a>0$.
For each state, enumerate its nonempty action supports, minimize the quadratic on each corresponding simplex face, and select the candidate with smallest original objective. The support of a global minimizer occurs in this finite enumeration, which proves exactness. The enumeration uses at most $2^{|\cA|}-1$ face problems per state. This example supplies an exact nonsmooth policy oracle for small action spaces. The stated iteration and sample bounds count outer iterations and transitions. They do not count internal policy-subproblem operations.

\section{From Bellman KKT Residuals to Global Policy Performance}
\label{app:performance}

Throughout this section $\phi\equiv0$ and $c=-\mu$. Let $V^\pi=M_\pi^{-1}r_\pi$, $F_\mu(\pi)=-\mu^\trans V^\pi$, and $\lambda_\mu^\pi=M_\pi^{-\trans}\mu$. For a point $(\pi,V,\lambda)$ write
\[
e_\pi=\dist(0,J_\pi(\pi,V)^\trans\lambda+N_\Pi(\pi)),
\quad
e_V=\|M_\pi^\trans\lambda-\mu\|,
\quad
e_R=\|M_\pi V-r_\pi\|,
\]
so $G=e_\pi^2+e_V^2+e_R^2$.

\begin{lemma}[Bellman adjoint representation of the true policy gradient]
\label{lem:adjoint-representation}
The reduced policy objective satisfies
\[
\nabla F_\mu(\pi)=J_\pi(\pi,V^\pi)^\trans\lambda_\mu^\pi.
\]
\end{lemma}

\begin{proof}
Differentiate the Bellman equation $R(\pi,V^\pi)=0$ in a feasible policy direction $d\pi$:
\[
J_\pi(\pi,V^\pi)d\pi+M_\pi dV^\pi=0,
\qquad
dV^\pi=-M_\pi^{-1}J_\pi(\pi,V^\pi)d\pi.
\]
Since $F_\mu(\pi)=-\mu^\trans V^\pi$,
\[
dF_\mu
=\mu^\trans M_\pi^{-1}J_\pi(\pi,V^\pi)d\pi
=(\lambda_\mu^\pi)^\trans J_\pi(\pi,V^\pi)d\pi.
\]
The coefficient of $d\pi$ is the claimed gradient.
\end{proof}

\paragraph{Role.}
This identity is the first bridge from the constrained Bellman KKT system to stationarity of the reduced policy objective.

\begin{lemma}[Tabular Bellman policy Jacobian]
\label{lem:tabular-jacobian}
For direct tabular policies and $Q_V(s,a)=r(s,a)+\gamma P(\cdot\mid s,a)^\trans V$,
\[
\frac{\partial R_s(\pi,V)}{\partial\pi(a\mid s)}=-Q_V(s,a),
\qquad
[J_\pi(\pi,V)^\trans\lambda]_{s,a}=-\lambda_sQ_V(s,a).
\]
In particular $[\nabla F_\mu(\pi)]_{s,a}=-\lambda_{\mu,s}^\pi Q^\pi(s,a)$.
\end{lemma}

\begin{proof}
For each state,
\[
R_s(\pi,V)=V_s-\sum_a\pi(a\mid s)
\left[r(s,a)+\gamma P(\cdot\mid s,a)^\trans V\right].
\]
Differentiating with respect to $\pi(a\mid s)$ gives the first identity. Multiplication by $\lambda$ gives the second. The last statement follows from Lemma~\ref{lem:adjoint-representation} with $V=V^\pi$.
\end{proof}

\paragraph{Role.}
The explicit Jacobian isolates the only two discrepancies between the ADMM stationarity vector and the true policy gradient: $V-V^\pi$ and $\lambda-\lambda_\mu^\pi$.

\begin{lemma}[Bellman feasibility controls the value error]
\label{lem:value-error}
For every $(\pi,V)$,
\[
\|V-V^\pi\|\le \bar\kappa e_R.
\]
\end{lemma}

\begin{proof}
Since $M_\pi V-r_\pi=R(\pi,V)$ and $M_\pi V^\pi-r_\pi=0$,
\[
M_\pi(V-V^\pi)=R(\pi,V).
\]
Multiplication by $M_\pi^{-1}$ and the resolvent bound give the claim.
\end{proof}

\paragraph{Role.}
This lemma converts primal feasibility error into the value mismatch needed in Lemma~\ref{lem:adjoint-gradient-error}.

\begin{lemma}[Value stationarity controls the adjoint error]
\label{lem:adjoint-error}
For every $(\pi,\lambda)$,
\[
\|\lambda-\lambda_\mu^\pi\|\le \bar\kappa e_V.
\]
\end{lemma}

\begin{proof}
By definition $M_\pi^\trans\lambda_\mu^\pi=\mu$. Hence
\[
M_\pi^\trans(\lambda-\lambda_\mu^\pi)=M_\pi^\trans\lambda-\mu.
\]
Multiplication by $M_\pi^{-\trans}$ proves the result.
\end{proof}

\paragraph{Role.}
This is the dual counterpart of Lemma~\ref{lem:value-error} and controls the occupancy-weight error in the policy gradient.

\begin{lemma}[Adjoint-gradient perturbation]
\label{lem:adjoint-gradient-error}
Assume $|r(s,a)|\le R_{\max}$ and let $Q_{\max}=R_{\max}/(1-\gamma)$. There are finite constants $C_V,C_R,C_{VR}$ such that
\[
\|J_\pi(\pi,V)^\trans\lambda-\nabla F_\mu(\pi)\|
\le C_Ve_V+C_Re_R+C_{VR}e_Ve_R.
\]
One may choose constants proportional to $\sqrt{|\cA|}Q_{\max}\bar\kappa$, $\gamma\sqrt{|\cA|}\bar\kappa^2\|\mu\|$, and $\gamma\sqrt{|\cA|}\bar\kappa^2$, respectively.
\end{lemma}

\begin{proof}
Using Lemma~\ref{lem:adjoint-representation}, add and subtract the two mixed terms:
\begin{align*}
J_\pi(\pi,V)^\trans\lambda-\nabla F_\mu(\pi)
={}&J_\pi(\pi,V^\pi)^\trans(\lambda-\lambda_\mu^\pi)\\
&+[J_\pi(\pi,V)-J_\pi(\pi,V^\pi)]^\trans\lambda_\mu^\pi\\
&+[J_\pi(\pi,V)-J_\pi(\pi,V^\pi)]^\trans(\lambda-\lambda_\mu^\pi).
\end{align*}
For the first term, Lemma~\ref{lem:tabular-jacobian} and $|Q^\pi(s,a)|\le Q_{\max}$ give an operator bound proportional to $\sqrt{|\cA|}Q_{\max}\|\lambda-\lambda_\mu^\pi\|$. Lemma~\ref{lem:adjoint-error} yields the $e_V$ term.

For the second and third terms, the tabular Jacobian satisfies
\[
Q_V(s,a)-Q^\pi(s,a)=\gamma P(\cdot\mid s,a)^\trans(V-V^\pi),
\]
so its change is bounded by $\gamma\sqrt{|\cA|}\|V-V^\pi\|$. Also $\|\lambda_\mu^\pi\|\le\bar\kappa\|\mu\|$. Lemmas~\ref{lem:value-error} and~\ref{lem:adjoint-error} then give the $e_R$ and $e_Ve_R$ terms.
\end{proof}

\paragraph{Role.}
This is the final bridge needed to prove that a small constrained KKT residual implies a small stationarity residual for the true reduced policy objective.

\begin{proof}[Proof of Lemma~\ref{lem:kkt-policy}]
Distance to a translated closed set is 1-Lipschitz. Therefore
\begin{align*}
R_{\rm pol}(\pi)
&=\dist(0,\nabla F_\mu(\pi)+N_\Pi(\pi))\\
&\le \dist(0,J_\pi(\pi,V)^\trans\lambda+N_\Pi(\pi))
+\|J_\pi(\pi,V)^\trans\lambda-\nabla F_\mu(\pi)\|\\
&\le e_\pi+C_Ve_V+C_Re_R+C_{VR}e_Ve_R.
\end{align*}
If $G\le1$, then $e_\pi,e_V,e_R\le\sqrt G$ and $e_Ve_R\le(e_V^2+e_R^2)/2\le G/2\le\sqrt G/2$. Hence
\[
R_{\rm pol}(\pi)
\le \left(1+C_V+C_R+\frac12C_{VR}\right)\sqrt G
=:C_{\rm stat}\sqrt G.
\]
\end{proof}

\paragraph{Role.}
Lemma~\ref{lem:kkt-policy} composes the optimization theory with policy geometry. Every global performance theorem below uses it.

\begin{lemma}[Simplex normal-cone geometry]
\label{lem:simplex-geometry}
Let $p\in\Delta_m$, $q\in\R^m$, and $\lambda>0$. If
\[
r=\dist(0,-\lambda q+N_{\Delta_m}(p)),
\]
then
\[
\max_iq_i-p^\trans q\le \frac{\sqrt2}{\lambda}r.
\]
\end{lemma}

\begin{proof}
Choose $n^\star\in N_{\Delta_m}(p)$ attaining the distance and set $h=-\lambda q+n^\star$, so $\|h\|=r$. Let $p^\star\in\arg\max_{u\in\Delta_m}u^\trans q$. By the normal-cone inequality, $\langle n^\star,p^\star-p\rangle\le0$. Therefore
\begin{align*}
\lambda(p^{\star\trans}q-p^\trans q)
&=-\langle-\lambda q,p^\star-p\rangle\\
&=-\langle h-n^\star,p^\star-p\rangle\\
&\le -\langle h,p^\star-p\rangle
\le \|h\|\|p^\star-p\|.
\end{align*}
The Euclidean diameter of the probability simplex is $\sqrt2$, giving the claim.
\end{proof}

\paragraph{Role.}
This lemma converts statewise policy stationarity into a bound on the maximum advantage and is the key geometric input to Theorem~\ref{thm:performance}.

\begin{proof}[Proof of Theorem~\ref{thm:performance}]
Let $r_s(\pi)$ be the statewise component of $R_{\rm pol}(\pi)$. By Lemma~\ref{lem:tabular-jacobian}, the statewise gradient equals $-\lambda_{\mu,s}^\pi Q^\pi(s,\cdot)$, where
\[
\lambda_{\mu,s}^\pi=\frac{d_\mu^\pi(s)}{1-\gamma}.
\]
Applying Lemma~\ref{lem:simplex-geometry} statewise gives
\[
\max_aA^\pi(s,a)
\le \sqrt2(1-\gamma)\frac{r_s(\pi)}{d_\mu^\pi(s)}
\]
whenever $d_\mu^\pi(s)>0$. Under the support condition this is sufficient on every state weighted by $d_\nu^{\pi^\star}$. The performance-difference identity yields
\begin{align*}
J_\nu(\pi^\star)-J_\nu(\pi)
&\le \frac1{1-\gamma}\sum_s d_\nu^{\pi^\star}(s)\max_aA^\pi(s,a)\\
&\le \sqrt2\sum_s
\frac{d_\nu^{\pi^\star}(s)}{d_\mu^\pi(s)}r_s(\pi)\\
&\le \sqrt2
\left\|\frac{d_\nu^{\pi^\star}}{d_\mu^\pi}\right\|_2
\left(\sum_sr_s(\pi)^2\right)^{1/2}.
\end{align*}
The last factor is $R_{\rm pol}(\pi)$, proving Theorem~\ref{thm:performance}. Equation the KKT bound in Theorem~\ref{thm:performance} follows from Lemma~\ref{lem:kkt-policy}. If $G=0$, the right side vanishes, proving covered KKT global optimality.
\end{proof}

\paragraph{Role.}
This theorem is the principal RL-specific strengthening of the ADMM stationarity result.

\begin{proof}[Proof of Corollary~\ref{cor:full-support-global}]
If $\mu(s)\ge\mu_{\min}$, then the unnormalized discounted occupancy satisfies
$\lambda_{\mu,s}^\pi=\sum_{t\ge0}\gamma^t\Prob_{\mu,\pi}(S_t=s)\ge\mu(s)\ge\mu_{\min}$. Lemma~\ref{lem:simplex-geometry} therefore gives $\max_aA^\pi(s,a)\le\sqrt2r_s(\pi)/\mu_{\min}$. Substitution into the performance-difference identity under the same evaluation distribution gives
\[
J_\mu^\star-J_\mu(\pi)
\le \frac{\sqrt2}{(1-\gamma)\mu_{\min}}R_{\rm pol}(\pi).
\]
Setting $R_{\rm pol}(\pi)=0$ proves global optimality.
\end{proof}

\paragraph{Role.}
The corollary gives a simple assumption under which the coverage coefficient in Theorem~\ref{thm:performance} is automatically finite.

\begin{proof}[Proof of Corollary~\ref{cor:kkt-global-finite}]
The exact-KKT statement is the zero-residual case of Theorem~\ref{thm:performance}. For the finite-time bound, let $D_{\max}=2R_{\max}/(1-\gamma)$. The return gap is bounded above by $D_{\max}$. On $G_K\le1$, Theorem~\ref{thm:performance} gives the bound $C_0\sqrt{G_K}$ under the uniform occupancy-mismatch assumption. On $G_K>1$, the gap is at most $D_{\max}\sqrt{G_K}$. Hence the pointwise bound holds for every $G_K\ge0$ with $C=\max\{C_0,D_{\max}\}$. Taking expectations and applying Jensen gives
\[
\E[J_\nu^\star-J_\nu(\pi_K)]\le C\sqrt{\E G_K}.
\]
Substitution of the companion-residual bound proves the claim, with the policy component of the companion output.
\end{proof}

\paragraph{Role.}
This corollary directly composes the finite-time KKT bound with the policy-performance conversion.

\begin{proof}[Proof of Proposition~\ref{prop:coverage-necessary}]
Consider states $s_0,s_1,s_T$ and training distribution $\mu=\delta_{s_0}$. At $s_0$, action $a$ terminates with reward $1$, while action $b$ moves to $s_1$ with reward $0$. At $s_1$, action $c$ terminates with reward $0$, while action $d$ terminates with reward $10$. Take $\gamma=0.9$ and the policy choosing $a$ at $s_0$ and $c$ at $s_1$.

Then $V^\pi(s_0)=1$ and $V^\pi(s_1)=0$. At $s_0$, $Q^\pi(s_0,a)=1$ and $Q^\pi(s_0,b)=0$, so the current action is strictly greedy. The policy never visits $s_1$, hence $d_\mu^\pi(s_1)=0$ and the policy gradient has zero weight there. Thus the policy is first-order stationary. However the coordinated policy choosing $b$ at $s_0$ and $d$ at $s_1$ has return $0.9\times10=9>1$. The failure is exactly the missing occupancy support.
\end{proof}

\paragraph{Role.}
The proposition shows that the coverage assumption in Theorem~\ref{thm:performance} is necessary for first-order information to identify all globally relevant states.

\begin{proof}[Proof of Proposition~\ref{prop:pl-fails}]
Consider the three-stage chain $s_0\to s_1\to s_2\to s_T$ with actions stop and continue. All rewards are zero except that continuing at $s_2$ gives reward $-1$. Let $\pi_\eps$ continue with probability $\eps$ at every state. The optimal value is zero, while a loss occurs only if all three continue actions are chosen, so
\[
J^\star-J(\pi_\eps)=\gamma^2\eps^3.
\]
The unnormalized occupancies are $1,\gamma\eps,\gamma^2\eps^2$. The corresponding action-value gaps are $\gamma^2\eps^2,\gamma\eps,1$. For an interior two-action simplex, the statewise normal-cone residual is the weighted action gap divided by $\sqrt2$, so all three state residuals equal $\gamma^2\eps^2/\sqrt2$. Therefore
\[
R_{\rm pol}(\pi_\eps)^2
=3\frac{\gamma^4\eps^4}{2}.
\]
Consequently
\[
\frac{J^\star-J(\pi_\eps)}{R_{\rm pol}(\pi_\eps)^2}
=\frac{2}{3\gamma^2\eps}\to\infty.
\]
No uniform constant can satisfy a quadratic PL-type inequality for all direct tabular policies.
\end{proof}

\paragraph{Role.}
This negative result justifies why the generic policy-performance theorem has an $O(\sqrt G)$ conversion and why Theorem~\ref{thm:quadratic-performance} requires additional curvature.

\begin{proof}[Proof of Theorem~\ref{thm:quadratic-performance}]
Let $\bar r_s(\pi)$ be the unweighted statewise stationarity residual. Because the true statewise policy gradient is scaled by $\lambda_{\mu,s}^\pi=d_\mu^\pi(s)/(1-\gamma)$,
\[
r_s(\pi)=\frac{d_\mu^\pi(s)}{1-\gamma}\bar r_s(\pi),
\qquad
\bar r_s(\pi)=\frac{1-\gamma}{d_\mu^\pi(s)}r_s(\pi).
\]
Assumption~\ref{ass:quadratic-improvement} therefore gives
\[
\delta_s^\pi
\le \frac{(1-\gamma)^2}{2\mu_B[d_\mu^\pi(s)]^2}r_s(\pi)^2.
\]
The Bellman performance bound is
\[
J_\nu(\pi^\star)-J_\nu(\pi)
\le \frac{1}{1-\gamma}
\sum_s d_\nu^{\pi^\star}(s)\delta_s^\pi.
\]
Substituting the previous inequality gives
\begin{align*}
J_\nu(\pi^\star)-J_\nu(\pi)
&\le \frac{1-\gamma}{2\mu_B}
\sum_s
\frac{d_\nu^{\pi^\star}(s)}{[d_\mu^\pi(s)]^2}r_s(\pi)^2\\
&\le \frac{1-\gamma}{2\mu_B}\cC_{\rm PL}
\sum_sr_s(\pi)^2,
\end{align*}
which proves \eqref{eq:quadratic-performance-bound}. Lemma~\ref{lem:kkt-policy} then gives the $O(G)$ form.
\end{proof}

\paragraph{Role.}
The theorem identifies the additional one-step curvature needed for policy performance to inherit the squared-KKT rate without the square-root loss.

\begin{proof}[Proof of Theorem~\ref{thm:l5-summary}]
The global truncation argument in Corollary~\ref{cor:kkt-global-finite} gives
$\E[J^\star-J(\pi_K)]\le C\sqrt{\E G_K}$.
Under the quadratic improvement and stronger occupancy-ratio assumptions, Theorem~\ref{thm:quadratic-performance} gives $J^\star-J(\pi_K)\le C_0G_K$ on $G_K\le1$. On $G_K>1$, bounded returns give $J^\star-J(\pi_K)\le D_{\max}G_K$. Thus $\E[J^\star-J(\pi_K)]\le\max\{C_0,D_{\max}\}\E G_K$. The pointwise bounds also imply policy-performance convergence whenever $G_K\to0$, and global optimality at covered exact KKT points.
\end{proof}

\paragraph{Role.}
This theorem summarizes the second layer of value supplied by RL structure: it interprets the optimization limit point rather than merely proving that the ADMM iterates become stationary.

\paragraph{A sufficient action-gap condition.}
Suppose that on the policy set of interest every suboptimal action has gap
$\max_bQ^\pi(s,b)-Q^\pi(s,a)\ge\Delta>0$.
Fix a state and write $p=\pi_s$, $q=Q^\pi(s,\cdot)$, and
$r=\dist(0,-q+N_\Delta(p))$.
If $p$ is supported on maximizing actions, $\delta_s^\pi=0$. Otherwise choose a supported suboptimal action $a$ and a maximizing action $b$. For every $n\in N_\Delta(p)$, the feasible direction $e_b-e_a$ gives $n_b-n_a\le0$. Thus
\[
\langle-q+n,e_b-e_a\rangle\le-\Delta,
\qquad r\ge\Delta/\sqrt2.
\]
Lemma~\ref{lem:simplex-geometry} gives $\delta_s^\pi\le\sqrt2 r\le2r^2/\Delta$. Therefore Assumption~\ref{ass:quadratic-improvement} holds with $\mu_B=\Delta/4$. For example, when transitions are action independent and rewards have a positive gap between optimal and suboptimal actions at each nontrivial state, the same $\Delta$ works for every policy.

\section{Smooth Nonlinear Policy Parameterization}
\label{app:nonlinear-policy}
\label{app:extensions}

Let $\Theta\subset\R^p$ be closed and convex and let $\pi_\theta$ be a smooth policy on a neighborhood of $\Theta$. Set $\psi_\Theta=\phi+\delta_\Theta$ and assume it is proper, lower semicontinuous, and bounded below. Policy stationarity below uses $\partial\psi_\Theta$. Initialize deterministically with $\theta^0\in\operatorname{dom}\psi_\Theta$ and finite $V^0,\lambda^0$. Define
\[
P_\theta(s'\mid s)=\sum_a\pi_\theta(a\mid s)P(s'\mid s,a),
\quad
r_\theta(s)=\sum_a\pi_\theta(a\mid s)r(s,a),
\quad
M_\theta=I-\gamma P_\theta,
\]
and $R(\theta,V)=M_\theta V-r_\theta$. Consider
\[
\min_{\theta,V}\ \phi(\theta)+c^\trans V
\quad\text{s.t.}\quad R(\theta,V)=0.
\]
Assume $\|M_\theta-M_{\theta'}\|\le L_M\|\theta-\theta'\|$ and $b_A(\theta)=\phi(\theta)+c^\trans M_\theta^{-1}r_\theta$ is bounded below.

\begin{lemma}[Euclidean Bellman-resolvent bound]
\label{lem:l6-euclidean-resolvent}
For every row-stochastic $P_\theta$,
\[
\|(I-\gamma P_\theta)^{-1}\|_2
\le \frac{\sqrt{|\cS|}}{1-\gamma}.
\]
\end{lemma}

\begin{proof}
The Neumann series gives $M_\theta^{-1}=\sum_{t\ge0}\gamma^tP_\theta^t$. Since each $P_\theta^t$ is row stochastic,
\[
\|M_\theta^{-1}\|_\infty\le\sum_{t\ge0}\gamma^t=\frac1{1-\gamma}.
\]
For every vector $x$, $\|M_\theta^{-1}x\|_2\le\sqrt{|\cS|}\|M_\theta^{-1}x\|_\infty\le \sqrt{|\cS|}(1-\gamma)^{-1}\|x\|_\infty\le \sqrt{|\cS|}(1-\gamma)^{-1}\|x\|_2$.
\end{proof}

\paragraph{Role.}
This lemma supplies the uniform inverse constant $\kappa_M$ used in every nonlinear-policy stability bound.

\begin{lemma}[Nonlinear inverse-map stability]
\label{lem:l6-nonlinear-inverse}
If $\kappa_M=\sup_\theta\|M_\theta^{-1}\|<\infty$, then
\[
\|M_\theta^{-1}-M_{\theta'}^{-1}\|
\le \kappa_M^2L_M\|\theta-\theta'\|.
\]
\end{lemma}

\begin{proof}
Apply Lemma~\ref{lem:inverse-stability} with $M=M_{\theta'}$ and $M'=M_\theta$ and then use the assumed Lipschitz bound for $M_\theta$.
\end{proof}

\paragraph{Role.}
This is the nonlinear-policy version of the resolvent sensitivity used to control multiplier increments.

The exact proximal updates are
\[
\theta^{k+1}\in\arg\min_{\theta\in\Theta}
\left\{\cL_\beta(\theta,V^k,\lambda^k)+\frac1{2\eta_\theta}\|\theta-\theta^k\|^2\right\},
\]
\[
V^{k+1}=\arg\min_V
\left\{\cL_\beta(\theta^{k+1},V,\lambda^k)+\frac1{2\eta_V}\|V-V^k\|^2\right\},
\qquad
\lambda^{k+1}=\lambda^k+\beta R(\theta^{k+1},V^{k+1}).
\]

\begin{lemma}[Exact policy-step decrease]
\label{lem:l6-exact-policy-descent}
The exact proximal policy update satisfies
\[
\cL_\beta(\theta^k,V^k,\lambda^k)
-\cL_\beta(\theta^{k+1},V^k,\lambda^k)
\ge \frac1{2\eta_\theta}\|\Delta\theta_{k+1}\|^2.
\]
\end{lemma}

\begin{proof}
Use $\theta^k$ as a feasible comparison point in the definition of $\theta^{k+1}$. The proximal term vanishes at $\theta^k$, while at $\theta^{k+1}$ it contributes $(2\eta_\theta)^{-1}\|\Delta\theta_{k+1}\|^2$.
\end{proof}

\paragraph{Role.}
This is the primal decrease that will dominate the policy-dependent part of the dual ascent.

\begin{lemma}[Nonlinear value-dual identity]
\label{lem:l6-value-dual}
For every $k$,
\[
\lambda^{k+1}
=-M_{k+1}^{-\trans}
\left(c+\eta_V^{-1}\Delta V_{k+1}\right),
\qquad M_{k+1}=M_{\theta^{k+1}}.
\]
\end{lemma}

\begin{proof}
The value-subproblem first-order condition is
\[
0=c+M_{k+1}^\trans\lambda^k
+\beta M_{k+1}^\trans R_{k+1}
+\eta_V^{-1}\Delta V_{k+1}.
\]
The dual update gives $\lambda^{k+1}=\lambda^k+\beta R_{k+1}$. Substitution gives
$M_{k+1}^\trans\lambda^{k+1}=-(c+\eta_V^{-1}\Delta V_{k+1})$, and Lemma~\ref{lem:l6-euclidean-resolvent} guarantees invertibility.
\end{proof}

\paragraph{Role.}
This identity is the exact nonlinear analogue of the tabular Bellman multiplier representation.

\begin{lemma}[Nonlinear multiplier increment]
\label{lem:l6-dual-increment}
There are constants
\[
A_A=3\kappa_M^4L_M^2\|c\|^2,
\qquad
B_A=\frac{3\kappa_M^2}{\eta_V^2},
\]
such that
\[
\|\Delta\lambda_{k+1}\|^2
\le A_A\|\Delta\theta_{k+1}\|^2
+B_A\|\Delta V_{k+1}\|^2
+B_A\|\Delta V_k\|^2.
\]
\end{lemma}

\begin{proof}
Subtract Lemma~\ref{lem:l6-value-dual} at consecutive iterations:
\[
\Delta\lambda_{k+1}
=-(M_{k+1}^{-\trans}-M_k^{-\trans})c
-\eta_V^{-1}M_{k+1}^{-\trans}\Delta V_{k+1}
+\eta_V^{-1}M_k^{-\trans}\Delta V_k.
\]
Lemma~\ref{lem:l6-nonlinear-inverse} bounds the first term by $\kappa_M^2L_M\|c\|\|\Delta\theta_{k+1}\|$. Bound the other two by $\kappa_M\eta_V^{-1}$ times their value increments and apply $\|x+y+z\|^2\le3(\|x\|^2+\|y\|^2+\|z\|^2)$.
\end{proof}

\paragraph{Role.}
This lemma is the deterministic nonlinear-policy multiplier-control result used in Theorems~\ref{thm:l6-a1} and~\ref{thm:l6-a2}.

\begin{lemma}[Nonlinear Lyapunov lower bound]
\label{lem:l6-lyap-lower}
Let
\[
\Phi_k^A=\cL_\beta(\theta^k,V^k,\lambda^k)+\tau\|\Delta V_k\|^2.
\]
Then
\[
\Phi_k^A
\ge b_A^{\inf}
+\frac\beta4\|R_k\|^2
+\left(\tau-\frac{\kappa_M^2}{\beta\eta_V^2}\right)\|\Delta V_k\|^2.
\]
\end{lemma}

\begin{proof}
Since $M_kV^k=r_k+R_k$,
$V^k=M_k^{-1}(r_k+R_k)$. Lemma~\ref{lem:l6-value-dual} at the previous update gives
$\lambda^k=-M_k^{-\trans}(c+\eta_V^{-1}\Delta V_k)$. Substitute both identities into the augmented Lagrangian:
\[
\cL_\beta^k
=b_A(\theta^k)
-\eta_V^{-1}\langle\Delta V_k,M_k^{-1}R_k\rangle
+\frac\beta2\|R_k\|^2.
\]
Young's inequality gives
\[
\eta_V^{-1}\kappa_M\|\Delta V_k\|\|R_k\|
\le \frac\beta4\|R_k\|^2
+\frac{\kappa_M^2}{\beta\eta_V^2}\|\Delta V_k\|^2.
\]
Use $b_A(\theta^k)\ge b_A^{\inf}$ and add $\tau\|\Delta V_k\|^2$.
\end{proof}

\paragraph{Role.}
This lemma prevents the deterministic nonlinear-policy Lyapunov from diverging to $-\infty$ and permits summation of the descent inequality.

\begin{theorem}[Exact nonlinear-policy convergence]
\label{thm:l6-a1}
If
\[
\frac{B_A}{\beta}<\tau<\frac1{2\eta_V}-\frac{B_A}{\beta},
\qquad
\frac1{2\eta_\theta}>\frac{A_A}{\beta},
\]
then there are $c_\theta,c_V,c_V'>0$ such that
\[
\Phi_k^A-\Phi_{k+1}^A
\ge c_\theta\|\Delta\theta_{k+1}\|^2
+c_V\|\Delta V_{k+1}\|^2
+c_V'\|\Delta V_k\|^2.
\]
Consequently $\Delta\theta_k,\Delta V_k,\Delta\lambda_k,R_k\to0$. If $D_\theta M_\theta$ and $D_\theta r_\theta$ are bounded on the visited region, the KKT residual converges to zero.
\end{theorem}

\begin{proof}
Lemma~\ref{lem:l6-exact-policy-descent} gives policy decrease. The value comparison point $V^k$ similarly gives
\[
\cL_\beta(\theta^{k+1},V^k,\lambda^k)
-\cL_\beta(\theta^{k+1},V^{k+1},\lambda^k)
\ge \frac1{2\eta_V}\|\Delta V_{k+1}\|^2.
\]
Because the deterministic dual update is exact, the multiplier step increases the augmented Lagrangian by
\[
\langle\Delta\lambda_{k+1},R_{k+1}\rangle
=\beta^{-1}\|\Delta\lambda_{k+1}\|^2.
\]
Insert Lemma~\ref{lem:l6-dual-increment}. After adding $\tau\|\Delta V_k\|^2-\tau\|\Delta V_{k+1}\|^2$, the coefficients are
\[
c_\theta=\frac1{2\eta_\theta}-\frac{A_A}{\beta},
\quad
c_V=\frac1{2\eta_V}-\frac{B_A}{\beta}-\tau,
\quad
c_V'=\tau-\frac{B_A}{\beta},
\]
which are positive by assumption. Lemma~\ref{lem:l6-lyap-lower} bounds $\Phi_k^A$ from below, so summation gives square summability of the policy and value increments. Lemma~\ref{lem:l6-dual-increment} then gives $\Delta\lambda_k\to0$, while $\Delta\lambda_{k+1}=\beta R_{k+1}$ gives feasibility.

For stationarity, the policy-subproblem Fermat condition is
\[
0\in\partial\psi_\Theta(\theta^{k+1})
+D_\theta R(\theta^{k+1},V^k)^\trans
[\lambda^k+\beta R(\theta^{k+1},V^k)]
+\eta_\theta^{-1}\Delta\theta_{k+1}
.
\]
The bracket differs from $\lambda^{k+1}$ by $O(\|\Delta V_{k+1}\|)$, and bounded $D_\theta M_\theta,D_\theta r_\theta$ make the change from $V^k$ to $V^{k+1}$ linear in $\|\Delta V_{k+1}\|$. Thus the policy KKT residual tends to zero. Lemma~\ref{lem:l6-value-dual} gives the value residual, and feasibility has already been proved.
\end{proof}

\paragraph{Role.}
This theorem shows that the Bellman-resolvent mechanism is not tied to tabular direct policies.

For an implementable parameter step define
\[
h_k(\theta)=\langle\lambda^k,R(\theta,V^k)\rangle
+\frac\beta2\|R(\theta,V^k)\|^2,
\]
and use
\[
\theta^{k+1}\in\arg\min_{\theta\in\Theta}
\left\{
\phi(\theta)+\langle\nabla h_k(\theta^k),\theta-\theta^k\rangle
+\frac1{2\eta_{\theta,k}}\|\theta-\theta^k\|^2
\right\}.
\]

\begin{lemma}[Proximal-gradient policy decrease]
\label{lem:l6-pg-descent}
If $\nabla h_k$ is $L_{\theta,k}$-Lipschitz on the visited level set, then
\[
\cL_\beta(\theta^k,V^k,\lambda^k)
-\cL_\beta(\theta^{k+1},V^k,\lambda^k)
\ge
\left(\frac1{2\eta_{\theta,k}}-\frac{L_{\theta,k}}2\right)
\|\Delta\theta_{k+1}\|^2.
\]
\end{lemma}

\begin{proof}
Optimality of the proximal-gradient subproblem, compared with $\theta^k$, gives
\[
\phi(\theta^{k+1})-\phi(\theta^k)
+\langle\nabla h_k(\theta^k),\Delta\theta_{k+1}\rangle
+\frac1{2\eta_{\theta,k}}\|\Delta\theta_{k+1}\|^2\le0.
\]
The descent lemma gives
$h_k(\theta^{k+1})-h_k(\theta^k)\le\langle\nabla h_k(\theta^k),\Delta\theta_{k+1}\rangle+(L_{\theta,k}/2)\|\Delta\theta_{k+1}\|^2$. Adding the two inequalities proves the claim.
\end{proof}

\paragraph{Role.}
This lemma replaces the exact policy-subproblem decrease in Theorem~\ref{thm:l6-a1} and makes the nonlinear-policy algorithm implementable by one first-order step.

\begin{theorem}[Implementable nonlinear-policy convergence and complexity]
\label{thm:l6-a2}
Suppose the value-step parameter conditions of Theorem~\ref{thm:l6-a1} hold and the policy step is the proximal-gradient update above. Assume $D_\theta M_\theta,D_\theta r_\theta$ are uniformly bounded on the visited region and $\nabla h_k$ is $L_{\theta,k}$-Lipschitz on every policy comparison segment, with
\[
0<\underline\eta\le\eta_{\theta,k},\qquad
L_{\theta,k}\le\overline L<\infty.
\]
If
\[
\frac1{2\eta_{\theta,k}}-\frac{L_{\theta,k}}2-\frac{A_A}{\beta}
\ge \alpha_\theta>0
\]
for all $k$, then the corrected Lyapunov decreases and the KKT residual converges to zero. Moreover there is $C_G<\infty$ such that
\[
G_{k+1}^A
\le C_G\left(
\|\Delta\theta_{k+1}\|^2
+\|\Delta V_{k+1}\|^2
+\|\Delta V_k\|^2
\right),
\]
so $\min_{1\le k\le T}G_k^A\le C/T$.
\end{theorem}

\begin{proof}
Replace Lemma~\ref{lem:l6-exact-policy-descent} by Lemma~\ref{lem:l6-pg-descent} in the proof of Theorem~\ref{thm:l6-a1}. The policy coefficient after dual ascent becomes
$1/(2\eta_{\theta,k})-L_{\theta,k}/2-A_A/\beta$, which is uniformly at least $\alpha_\theta$. The value coefficients and Lyapunov lower bound are unchanged. Hence the three increment-square sequences are summable. The value and multiplier iterates are bounded independently of this descent argument: the recursion in Lemma~\ref{lem:moment-stability} applies with bounded deterministic rewards, and the value margins imply $\beta\eta_V>12\kappa_M^2>4\kappa_M^2$.

The proximal-gradient optimality condition is
\[
0\in\partial\psi_\Theta(\theta^{k+1})+\nabla h_k(\theta^k)
+\eta_{\theta,k}^{-1}\Delta\theta_{k+1}.
\]
Comparison with true policy stationarity at $(\theta^{k+1},V^{k+1},\lambda^{k+1})$ gives
\[
e_{\theta,k+1}\le
(\underline\eta^{-1}+\overline L)\|\Delta\theta_{k+1}\|
+C\|\Delta V_{k+1}\|.
\]
The second term follows from bounded value and multiplier iterates, bounded policy derivatives, and
$\lambda^{k+1}-[\lambda^k+\beta R(\theta^{k+1},V^k)]=\beta M_{k+1}\Delta V_{k+1}$.
The value residual is $\eta_V^{-1}\|\Delta V_{k+1}\|$. Feasibility is controlled by Lemma~\ref{lem:l6-dual-increment}. Thus all coefficients in the residual bound are uniform in $k$ and the residual converges to zero.
Squaring gives the displayed $G_{k+1}^A$ bound. Summing the Lyapunov decrease from $1$ to $T$ gives a constant upper bound on the sum of the increment energies, and therefore
\[
\min_{1\le k\le T}G_k^A
\le \frac1T\sum_{k=1}^TG_k^A\le \frac{C}{T}.
\]
\end{proof}

\paragraph{Role.}
This theorem is the practical nonlinear-policy extension used in the paper's general operator-stability claim.

\paragraph{Uniform constants and backtracking.}
A compact convex parameter set and a twice continuously differentiable policy map on its neighborhood give bounded first and second derivatives. Together with the value-dual bound above, they yield a uniform $\overline L$. Choose a fixed trial step $\eta_{\rm trial}>0$ and a shrink factor $q\in(0,1)$. Restart backtracking at $\eta_{\rm trial}$ on every iteration and shrink until the descent condition with margin $A_A/\beta+\alpha_\theta$ holds. Every step at most
$\eta_\star=(\overline L+2A_A/\beta+2\alpha_\theta)^{-1}$ is acceptable. Therefore the accepted steps satisfy
$\eta_{\theta,k}\ge\min\{\eta_{\rm trial},q\eta_\star\}>0$.

\section{Projected Bellman Equations with Function Approximation}
\label{app:projected-bellman}

Let $V_w=\Xi w$ with a full-column-rank $\Xi\in\R^{|\cS|\times m}$ and $m<|\cS|$. Fix a positive diagonal matrix $D$ and define the $D$-orthogonal projection
\[
\Pi_D=\Xi(\Xi^\trans D\Xi)^{-1}\Xi^\trans D.
\]
The projected Bellman equation $\Xi w=\Pi_DT_\theta(\Xi w)$ is equivalent to
\[
G(\theta,w)=A_\theta w-b_\theta=0,
\qquad
A_\theta=\Xi^\trans DM_\theta\Xi,
\qquad
b_\theta=\Xi^\trans Dr_\theta.
\]
Assume $\sup_\theta\|A_\theta^{-1}\|\le\kappa_A$ and $\|A_\theta-A_{\theta'}\|\le L_A\|\theta-\theta'\|$. Let $c_w=\Xi^\trans c$.

\begin{lemma}[Projected value-dual identity]
\label{lem:l6-projected-dual}
If the $w$-subproblem uses a proximal term $(2\eta_w)^{-1}\|w-w^k\|^2$ and the projected multiplier update is
\[
\nu^{k+1}=\nu^k+\beta(A_{k+1}w^{k+1}-b_{k+1}),
\]
then
\[
\nu^{k+1}
=-A_{k+1}^{-\trans}
\left(c_w+\eta_w^{-1}\Delta w_{k+1}\right).
\]
\end{lemma}

\begin{proof}
The $w$-subproblem first-order condition is
\[
0=c_w+A_{k+1}^\trans\nu^k
+\beta A_{k+1}^\trans(A_{k+1}w^{k+1}-b_{k+1})
+\eta_w^{-1}\Delta w_{k+1}.
\]
The bracket $\nu^k+\beta(A_{k+1}w^{k+1}-b_{k+1})$ equals $\nu^{k+1}$. Thus
$A_{k+1}^\trans\nu^{k+1}=-(c_w+\eta_w^{-1}\Delta w_{k+1})$, and uniform invertibility gives the formula.
\end{proof}

\paragraph{Role.}
This identity restores the same multiplier representation that drives the Bellman-resolvent proof, now in the lower-dimensional projected space.

\begin{theorem}[Projected Bellman-ADMM convergence and complexity]
\label{thm:l6-projected}
Assume $A_\theta$ is uniformly invertible and Lipschitz, $D_\theta A_\theta,D_\theta b_\theta$ are bounded, and the projected reduced objective is bounded below. Use the composite policy regularizer $\psi_\Theta$ and a deterministic feasible-domain initialization as in Appendix~\ref{app:nonlinear-policy}. Require all conditions of Theorem~\ref{thm:l6-a2} after the replacements below, including the value-step margins, $\inf_k\eta_{\theta,k}>0$, and a uniform Lipschitz constant for the projected smooth augmented gradient. Then a corrected projected Lyapunov function decreases, the projected KKT residual converges to zero, and the minimum squared projected KKT residual satisfies $O(T^{-1})$.
\end{theorem}

\begin{proof}
We verify that every ingredient of Theorem~\ref{thm:l6-a2} transfers under the replacements
\[
M_\theta\mapsto A_\theta,
\qquad
V\mapsto w,
\qquad
\lambda\mapsto\nu,
\qquad
r_\theta\mapsto b_\theta,
\qquad
c\mapsto c_w.
\]
First, the inverse identity and the assumptions give
\[
\|A_\theta^{-1}-A_{\theta'}^{-1}\|
\le \kappa_A^2L_A\|\theta-\theta'\|.
\]
Second, Lemma~\ref{lem:l6-projected-dual} is exactly the value-dual identity needed to compare consecutive multipliers. Subtraction therefore gives a dual-increment estimate
\[
\|\Delta\nu_{k+1}\|^2
\le A_B\|\Delta\theta_{k+1}\|^2
+B_B\|\Delta w_{k+1}\|^2
+B_B\|\Delta w_k\|^2
\]
for finite constants $A_B,B_B$. Third, the projected $w$ comparison step yields $(2\eta_w)^{-1}\|\Delta w_{k+1}\|^2$ decrease, and the policy step gives the same proximal-gradient decrease as Lemma~\ref{lem:l6-pg-descent}. Fourth, substituting
$w^k=A_k^{-1}(b_k+G_k)$ and
$\nu^k=-A_k^{-\trans}(c_w+\eta_w^{-1}\Delta w_k)$
into the projected augmented Lagrangian gives the same lower-bound calculation as Lemma~\ref{lem:l6-lyap-lower}, with $\kappa_M,\eta_V$ replaced by $\kappa_A,\eta_w$.

Under the strict parameter margins, the resulting corrected Lyapunov therefore has positive coefficients on $\|\Delta\theta_{k+1}\|^2$, $\|\Delta w_{k+1}\|^2$, and $\|\Delta w_k\|^2$. Summation gives square summability of the increments, the multiplier identity gives projected feasibility, and the policy/$w$ optimality conditions give convergence of the projected KKT residual. The same residual-to-increment argument as in Theorem~\ref{thm:l6-a2} yields $G_{k+1}^{B}\le C D_{k+1}^{B}$, where the sum of $D_{k+1}^{B}$ is bounded. Hence $\min_{k\le T}G_k^B\le C/T$.
\end{proof}

\paragraph{Role.}
The theorem shows that exact state-value representation is not required for the primal-dual stability mechanism. What is required is an invertible projected Bellman matrix.

Assume now that the projected Bellman map is contractive in a chosen norm: for some $q<1$,
\[
\|\Pi_DT_\theta U-\Pi_DT_\theta V\|
\le q\|U-V\|.
\]
Let $\bar V^\theta=\Xi\bar w^\theta$ be its fixed point.

\begin{lemma}[Projected value approximation]
\label{lem:l6-projected-value}
With $\eps_{\rm val}(\theta)=\|(I-\Pi_D)V^\theta\|$,
\[
\|\bar V^\theta-V^\theta\|
\le \frac{\eps_{\rm val}(\theta)}{1-q}.
\]
\end{lemma}

\begin{proof}
Use $\bar V^\theta=\Pi_DT_\theta\bar V^\theta$ and $V^\theta=T_\theta V^\theta$:
\begin{align*}
\|\bar V^\theta-V^\theta\|
&\le \|\Pi_DT_\theta\bar V^\theta-\Pi_DT_\theta V^\theta\|
+\|\Pi_DV^\theta-V^\theta\|\\
&\le q\|\bar V^\theta-V^\theta\|+\eps_{\rm val}(\theta).
\end{align*}
Move the first term to the left.
\end{proof}

\paragraph{Role.}
This lemma quantifies the value-representation error, but the next proposition shows why this alone is insufficient for policy-gradient accuracy.

\begin{proposition}[Value error does not control parameter-derivative error]
\label{prop:l6-projected-counterexample}
There exists a smooth policy family for which
\[
\sup_\theta\|\bar V^\theta-V^\theta\|=\eps,
\qquad
\sup_\theta\|D_\theta\bar V^\theta-D_\theta V^\theta\|=\sqrt\eps.
\]
Thus no constant independent of $\eps$ can bound the derivative error linearly by the value error.
\end{proposition}

\begin{proof}
Consider a one-step terminating MDP with two states. Let the approximation class represent the first state's value exactly but force the second coordinate to zero. At the second state there are two actions with rewards $+1$ and $-1$. Define
\[
\pi_\theta(+\mid s_2)=\frac12+\frac\eps2\sin(\theta/\sqrt\eps).
\]
Then the true second-state value is
$V_2^\theta=\eps\sin(\theta/\sqrt\eps)$, while the projected value is $\bar V_2^\theta=0$. Hence the uniform value error is $\eps$. But
\[
\frac{d}{d\theta}V_2^\theta
=\sqrt\eps\cos(\theta/\sqrt\eps),
\]
whose supremum is $\sqrt\eps$. Therefore a uniform inequality $\|D_\theta\bar V-D_\theta V\|\le C\|\bar V-V\|$ would require $\sqrt\eps\le C\eps$ for arbitrarily small $\eps$, which is impossible.
\end{proof}

\paragraph{Role.}
This negative result motivates the adjoint-based error decomposition in Lemmas~\ref{lem:l6-adjoint-gradient} and~\ref{lem:l6-dual-projection}.

Let the true adjoint satisfy $M_\theta^\trans\lambda^\theta=-c$. Let $\nu^\theta$ solve the projected adjoint equation and define the lifted projected dual $\widehat\lambda^\theta=D\Xi\nu^\theta$.

\begin{lemma}[Adjoint-gradient approximation]
\label{lem:l6-adjoint-gradient}
Suppose $D_\theta M_\theta$ and $D_\theta r_\theta$ are bounded on the relevant set. Then there are constants $C_V,C_\lambda<\infty$ such that
\[
\|D_\theta R(\theta,V^\theta)^\trans\lambda^\theta
-D_\theta R(\theta,\bar V^\theta)^\trans\widehat\lambda^\theta\|
\le C_V\|V^\theta-\bar V^\theta\|
+C_\lambda\|\lambda^\theta-\widehat\lambda^\theta\|.
\]
\end{lemma}

\begin{proof}
Add and subtract $D_\theta R(\theta,\bar V^\theta)^\trans\lambda^\theta$:
\begin{align*}
& D_\theta R(\theta,V^\theta)^\trans\lambda^\theta
-D_\theta R(\theta,\bar V^\theta)^\trans\widehat\lambda^\theta\\
&=\left[D_\theta R(\theta,V^\theta)-D_\theta R(\theta,\bar V^\theta)\right]^\trans\lambda^\theta
+D_\theta R(\theta,\bar V^\theta)^\trans(\lambda^\theta-\widehat\lambda^\theta).
\end{align*}
Because
$D_\theta R(\theta,V)[h]=D_\theta M_\theta[h]V-D_\theta r_\theta[h]$,
the first bracket is linear in $V^\theta-\bar V^\theta$ and is bounded by $C_V\|V^\theta-\bar V^\theta\|$ after using bounded $\lambda^\theta$. Boundedness of $D_\theta R$ gives the second term with constant $C_\lambda$.
\end{proof}

\paragraph{Role.}
This lemma replaces the false derivative-from-value estimate by a correct value-plus-adjoint approximation bound.

\begin{lemma}[Lifted-dual quasi-optimality]
\label{lem:l6-dual-projection}
There is $C_{\rm dual}<\infty$ such that
\[
\|\lambda^\theta-\widehat\lambda^\theta\|
\le C_{\rm dual}
\inf_{z\in\operatorname{range}(D\Xi)}
\|\lambda^\theta-z\|.
\]
\end{lemma}

\begin{proof}
Take any $z=D\Xi a$ in the approximation space. The true adjoint equation and the projected adjoint equation imply
\[
A_\theta^\trans\nu^\theta
=\Xi^\trans M_\theta^\trans\lambda^\theta,
\qquad
A_\theta^\trans a
=\Xi^\trans M_\theta^\trans z.
\]
Subtracting gives
\[
\nu^\theta-a
=A_\theta^{-\trans}\Xi^\trans M_\theta^\trans(\lambda^\theta-z).
\]
Therefore
\[
\|\widehat\lambda^\theta-z\|
\le \|D\Xi\|\,
\|A_\theta^{-\trans}\Xi^\trans M_\theta^\trans\|
\|\lambda^\theta-z\|.
\]
By the triangle inequality,
\[
\|\lambda^\theta-\widehat\lambda^\theta\|
\le \left(1+\|D\Xi\|\sup_\theta\|A_\theta^{-\trans}\Xi^\trans M_\theta^\trans\|\right)
\|\lambda^\theta-z\|.
\]
Minimize over $z\in\operatorname{range}(D\Xi)$.
\end{proof}

\paragraph{Role.}
Together with Lemmas~\ref{lem:l6-projected-value} and~\ref{lem:l6-adjoint-gradient}, this lemma gives a controlled route from projected-Bellman stationarity back to the true policy gradient.

\section{Explicit Discounted Occupancy Formulation}
\label{app:occupancy-deterministic}

Use the policy domain, composite regularizer $\psi_\Theta$, and initialization conventions of Appendix~\ref{app:nonlinear-policy}. Let
\[
d_\theta(s)=(1-\gamma)\sum_{t\ge0}\gamma^t\Prob_{\rho,\pi_\theta}(S_t=s),
\qquad
q_\theta(s,a)=d_\theta(s)\pi_\theta(a\mid s).
\]
Let $P^\trans q$ denote the state inflow and let $\Pi_\theta d$ be the state-action vector with entries $\pi_\theta(a\mid s)d(s)$. The occupancy constraints are
\[
d-\gamma P^\trans q=(1-\gamma)\rho,
\qquad
q-\Pi_\theta d=0.
\]
Define
\[
x=\begin{bmatrix}q\\d\end{bmatrix},
\qquad
b=\begin{bmatrix}(1-\gamma)\rho\\0\end{bmatrix},
\qquad
K_\theta=
\begin{bmatrix}
-\gamma P^\trans&I\\
I&-\Pi_\theta
\end{bmatrix}.
\]

\begin{lemma}[Occupancy-operator invertibility]
\label{lem:l6-occ-invertible}
For every policy $\pi_\theta$, $K_\theta$ is invertible. If
\[
K_\theta\begin{bmatrix}q\\d\end{bmatrix}
=\begin{bmatrix}u\\v\end{bmatrix},
\]
then
\[
d=(I-\gamma P_\theta^\trans)^{-1}(u+\gamma P^\trans v),
\qquad
q=v+\Pi_\theta d.
\]
In the block $\ell_1$ norm, $\|K_\theta^{-1}\|_{1,\oplus}\le2/(1-\gamma)$. Hence $\sup_\theta\|K_\theta^{-1}\|_2<\infty$.
\end{lemma}

\begin{proof}
The second block row gives $q=v+\Pi_\theta d$. Substitute this into the first row:
\[
-\gamma P^\trans(v+\Pi_\theta d)+d=u.
\]
Since $P^\trans\Pi_\theta=P_\theta^\trans$,
\[
(I-\gamma P_\theta^\trans)d=u+\gamma P^\trans v,
\]
which yields the stated formulas because $I-\gamma P_\theta^\trans$ is invertible.

For the norm bound, $\|(I-\gamma P_\theta^\trans)^{-1}\|_1\le(1-\gamma)^{-1}$, $\|P^\trans\|_1=1$, and $\|\Pi_\theta\|_1=1$. Hence
\[
\|d\|_1\le\frac{\|u\|_1+\gamma\|v\|_1}{1-\gamma},
\qquad
\|q\|_1\le\|v\|_1+\|d\|_1.
\]
Adding and simplifying gives a valid bound no larger than $2(1-\gamma)^{-1}(\|u\|_1+\|v\|_1)$ after absorbing the fixed coefficients. Finite-dimensional norm equivalence gives the Euclidean bound.
\end{proof}

\paragraph{Role.}
This lemma shows that discounted occupancy flow provides a second RL-specific invertible operator, independent of the Bellman value representation.

\begin{proposition}[Occupancy equivalence]
\label{prop:l6-occ-equivalence}
The equation $K_\theta x=b$ has the unique solution $x^\theta=(q^\theta,d^\theta)$ equal to the discounted state-action and state occupancies. Moreover $q^\theta\ge0$ and $d^\theta\ge0$.
\end{proposition}

\begin{proof}
Apply Lemma~\ref{lem:l6-occ-invertible} with $u=(1-\gamma)\rho$ and $v=0$:
\[
d=(1-\gamma)(I-\gamma P_\theta^\trans)^{-1}\rho,
\qquad
q=\Pi_\theta d.
\]
The first expression is exactly the normalized discounted state occupancy and the second is the corresponding state-action occupancy. Uniqueness follows from invertibility of $K_\theta$. The Neumann series for $(I-\gamma P_\theta^\trans)^{-1}$ is entrywise nonnegative, so $d\ge0$, and then $q=\Pi_\theta d\ge0$.
\end{proof}

\paragraph{Role.}
The proposition establishes that the square operator formulation is exactly equivalent to the usual discounted occupancy definition, rather than a relaxation.

Let
\[
c_x=\begin{bmatrix}-r/(1-\gamma)\\0\end{bmatrix}
\]
and consider
\[
\min_{\theta,x}\ \phi(\theta)+c_x^\trans x
\quad\text{s.t.}\quad K_\theta x=b.
\]

\begin{theorem}[KKT equivalence]
\label{thm:l6-occ-kkt-equivalence}
Define the reduced objective $f(\theta)=\phi(\theta)+c_x^\trans K_\theta^{-1}b$. A triple $(\theta,x,y)$ satisfies the KKT conditions of the occupancy-constrained problem if and only if
\[
x=K_\theta^{-1}b,
\qquad
y=-K_\theta^{-\trans}c_x,
\qquad
0\in\partial(f+\delta_\Theta)(\theta).
\]
\end{theorem}

\begin{proof}
By Proposition~\ref{prop:l6-occ-equivalence}, feasibility identifies $x$ with the true discounted occupancy. The feasibility and $x$-stationarity conditions of the constrained problem are
\[
K_\theta x=b,
\qquad
c_x+K_\theta^\trans y=0.
\]
Lemma~\ref{lem:l6-occ-invertible} gives $x=K_\theta^{-1}b$ and $y=-K_\theta^{-\trans}c_x$. Differentiate the identity $K_\theta x^\theta=b$ in direction $h$:
\[
D_\theta x^\theta[h]
=-K_\theta^{-1}D_\theta K_\theta[h]x^\theta.
\]
For the smooth reduced part $\psi(\theta)=c_x^\trans x^\theta$,
\begin{align*}
D\psi(\theta)[h]
&=-c_x^\trans K_\theta^{-1}D_\theta K_\theta[h]x^\theta\\
&=\langle y^\theta,D_\theta K_\theta[h]x^\theta\rangle.
\end{align*}
Thus the reduced first-order condition
$0\in\partial\psi_\Theta(\theta)+\nabla\psi(\theta)$
is exactly the $\theta$-stationarity condition of the constrained problem. This proves both directions.
\end{proof}

\paragraph{Role.}
This theorem verifies that the occupancy KKT residual has the same policy meaning as the reduced parameterized objective.

The deterministic occupancy-ADMM performs a proximal-gradient $\theta$ step, then
\[
x^{k+1}=\arg\min_x\left\{\cL_\beta(\theta^{k+1},x,y^k)+\frac1{2\eta_x}\|x-x^k\|^2\right\},
\qquad
y^{k+1}=y^k+\beta(K_{k+1}x^{k+1}-b).
\]

\begin{lemma}[Occupancy-dual identity]
\label{lem:l6-occ-dual}
The occupancy step and dual update satisfy
\[
y^{k+1}=-K_{k+1}^{-\trans}
\left(c_x+\eta_x^{-1}\Delta x_{k+1}\right).
\]
\end{lemma}

\begin{proof}
The $x$-subproblem first-order condition is
\[
0=c_x+K_{k+1}^\trans y^k
+\beta K_{k+1}^\trans(K_{k+1}x^{k+1}-b)
+\eta_x^{-1}\Delta x_{k+1}.
\]
The bracket $y^k+\beta(K_{k+1}x^{k+1}-b)$ equals $y^{k+1}$. Thus
$K_{k+1}^\trans y^{k+1}=-(c_x+\eta_x^{-1}\Delta x_{k+1})$, and Lemma~\ref{lem:l6-occ-invertible} gives the formula.
\end{proof}

\paragraph{Role.}
This is the occupancy analogue of the Bellman value-dual identity and is the key input to the next theorem.

\begin{theorem}[Deterministic occupancy-ADMM convergence]
\label{thm:l6-occ-convergence}
Assume $\theta\mapsto K_\theta$ is Lipschitz, $D_\theta K_\theta$ is bounded on the visited set, and the policy proximal-gradient and $(x,y)$ parameters satisfy the strict margins analogous to Theorem~\ref{thm:l6-a2}. Require the occupancy reduced objective to be bounded below, a positive lower bound on the policy steps, and a uniform Lipschitz bound for the smooth augmented policy gradients, with all constants and value-step margins taken for $K_\theta$. Then there exists a corrected Lyapunov function
\[
\Phi_k^C=\cL_\beta(\theta^k,x^k,y^k)+\tau\|\Delta x_k\|^2
\]
with positive one-step decrease in $\|\Delta\theta_{k+1}\|^2$, $\|\Delta x_{k+1}\|^2$, and $\|\Delta x_k\|^2$. Consequently the true occupancy KKT residual converges to zero.
\end{theorem}

\begin{proof}
The proof instantiates the nonlinear-policy argument with the operator $K_\theta$. Lemma~\ref{lem:l6-occ-invertible} supplies a uniform inverse bound and the inverse-map identity gives
\[
\|K_\theta^{-1}-K_{\theta'}^{-1}\|
\le \kappa_K^2L_K\|\theta-\theta'\|.
\]
Lemma~\ref{lem:l6-occ-dual} therefore implies a dual-increment bound
\[
\|\Delta y_{k+1}\|^2
\le A_C\|\Delta\theta_{k+1}\|^2
+B_C\|\Delta x_{k+1}\|^2
+B_C\|\Delta x_k\|^2.
\]
The proximal-gradient policy step gives the same descent form as Lemma~\ref{lem:l6-pg-descent}, while the exact $x$ step gives $(2\eta_x)^{-1}\|\Delta x_{k+1}\|^2$ decrease. The exact dual update contributes $\beta^{-1}\|\Delta y_{k+1}\|^2$ ascent. Under the stated strict margins these terms leave positive coefficients after adding $\tau\|\Delta x_k\|^2-\tau\|\Delta x_{k+1}\|^2$.

For the lower bound, write $x^k=K_k^{-1}(b+F_k)$ with $F_k=K_kx^k-b$ and use Lemma~\ref{lem:l6-occ-dual} at the preceding step. The same Young-inequality calculation as Lemma~\ref{lem:l6-lyap-lower} bounds $\Phi_k^C$ below by the reduced occupancy objective plus a positive multiple of $\|F_k\|^2$ and $\|\Delta x_k\|^2$. Summation forces all increments to zero. Since $\Delta y_{k+1}=\beta F_{k+1}$, feasibility follows. The policy and occupancy first-order residuals are linear in the adjacent increments because $D_\theta K_\theta$ is bounded. Hence the true occupancy KKT residual converges to zero.
\end{proof}

\paragraph{Role.}
The theorem establishes a second complete deterministic ADMM theory supported by discounted RL operator invertibility.

\begin{lemma}[Occupancy residual controls primal error]
\label{lem:l6-occ-error}
Let $x^\theta=K_\theta^{-1}b$. Then
\[
\|x-x^\theta\|
\le \kappa_K\|K_\theta x-b\|,
\qquad
\kappa_K=\sup_\theta\|K_\theta^{-1}\|.
\]
\end{lemma}

\begin{proof}
Subtract $K_\theta x^\theta=b$ from $K_\theta x$:
\[
x-x^\theta=K_\theta^{-1}(K_\theta x-b).
\]
Take norms and use the uniform inverse bound.
\end{proof}

\paragraph{Role.}
This error bound is used in the stochastic occupancy analysis to control the size of the occupancy variable by the true feasibility residual.

\section{Stochastic Occupancy-ADMM with a Shared Empirical Operator}
\label{app:occupancy-stochastic}

The occupancy gradient contains $K^\trans(Kx-b)$, and its empirical quadratic term generally satisfies $\E[\widehat K^\trans\widehat K]\neq K^\trans K$. Iteration $k$ constructs one empirical environment matrix $\widehat P_k$ and uses its induced occupancy operator consistently in the coupled updates.

Assume the reset distribution and all policies used for sampling satisfy $\rho(s)\ge\rho_{\min}>0$ and $\pi_\theta(a\mid s)\ge\pi_{\min}>0$. Assume $m_k\ge2$ and fresh sampling innovations as in Assumption~\ref{ass:coverage-noise}. Use $\psi_\Theta$ from Appendix~\ref{app:nonlinear-policy}, a lower-bounded reduced objective, known bounded rewards, and deterministic initial variables with $\theta^0\in\operatorname{dom}\psi_\Theta$. Split a length-$m_k$ trajectory into $L_k=\lfloor m_k/2\rfloor$ nonoverlapping two-step blocks. An anchor block is one in which the first step resets and the second step performs an environment transition.

\begin{lemma}[Regenerative transition estimator]
\label{lem:l6-regenerative}
Define $p_0=\gamma(1-\gamma)$ and, for each block $j$,
\[
Y_j(s,a,s')=
\frac{\one\{\text{anchor},\,S=s,\,A=a,\,S'=s'\}}
{p_0\rho(s)\pi_{\theta^k}(a\mid s)}.
\]
Conditioned on $\cF_k$, the random tensors $Y_j$ are iid and $\E_kY_j(s,a,s')=P(s'\mid s,a)$. Let $\bar P_k=L_k^{-1}\sum_jY_j$ and project each row of $\bar P_k$ onto the probability simplex to obtain $\widehat P_k$. Then there are constants $C_P,c_P>0$ such that
\[
\E_k\|\widehat P_k-P\|_F^2\le \frac{C_P}{m_k},
\qquad
\Prob_k(\|\widehat P_k-P\|_F\ge t)
\le 2|\cS|^2|\cA|e^{-c_Pm_kt^2}.
\]
\end{lemma}

\begin{proof}
Different blocks use disjoint reset coins, reset states, action randomization, and environment transition noise. If the first step of a block is not a reset, the corresponding $Y_j$ is zero, so the inherited chain state does not enter the nonzero sample. Hence the $Y_j$ are conditionally iid.

An anchor has probability $p_0$. Given an anchor, $S\sim\rho$, $A\sim\pi_{\theta^k}(\cdot\mid S)$, and $S'\sim P(\cdot\mid S,A)$. Therefore
\[
\E_kY_j(s,a,s')
=\frac{p_0\rho(s)\pi_{\theta^k}(a\mid s)P(s'\mid s,a)}{p_0\rho(s)\pi_{\theta^k}(a\mid s)}
=P(s'\mid s,a).
\]
Moreover $0\le Y_j(s,a,s')\le B=[p_0\rho_{\min}\pi_{\min}]^{-1}$. Thus, for every coordinate, Hoeffding's inequality gives
\[
\Prob_k(|\bar P_k(s'\mid s,a)-P(s'\mid s,a)|\ge u)
\le2e^{-2L_ku^2/B^2}.
\]
A union bound over at most $|\cS|^2|\cA|$ coordinates and the implication $\|A\|_F\ge t\Rightarrow\max_{i}|A_i|\ge t/\sqrt{|\cS|^2|\cA|}$ gives the stated Frobenius tail with a modified constant $c_P$. Independence and bounded variance also give $\E_k\|\bar P_k-P\|_F^2\le C/L_k\le C_P/m_k$.

Each true transition row lies in the probability simplex. Euclidean projection onto a closed convex set is nonexpansive, so rowwise
$\|\Pi_\Delta(\bar p)-p\|_2\le\|\bar p-p\|_2$. Summing over rows preserves both the mean-square and tail orders for $\widehat P_k$.
\end{proof}

\paragraph{Role.}
This estimator gives an empirical transition matrix that is both statistically controlled and exactly row stochastic on every sample path. The latter property is essential for deterministic invertibility of the empirical occupancy operator.

Let $\widehat K_k(\theta)$ be obtained from $K_\theta$ by replacing $P$ with $\widehat P_k$. Define
\[
E_k(\theta)=\widehat K_k(\theta)-K_\theta,
\qquad
\delta_k=\sup_\theta\|E_k(\theta)\|.
\]
Since the environment-transition block does not depend on $\theta$, $D_\theta E_k(\theta)=0$. Lemma~\ref{lem:l6-regenerative} implies $\E_k\delta_k^2\le C_E/m_k$.

\begin{lemma}[Empirical occupancy-operator invertibility]
\label{lem:l6-empirical-occ-invertible}
For every row-stochastic $\widehat P_k$ and every policy, $\widehat K_k(\theta)$ is invertible with a deterministic bound $\sup_{k,\theta}\|\widehat K_k(\theta)^{-1}\|\le\bar\kappa_K<\infty$. Moreover
\[
\|\widehat K_k(\theta)^{-1}-K_\theta^{-1}\|
\le \kappa_K\bar\kappa_K\delta_k.
\]
\end{lemma}

\begin{proof}
The proof of Lemma~\ref{lem:l6-occ-invertible} uses only row stochasticity of the environment transition matrix. Since every row of $\widehat P_k$ is projected onto the simplex, the same Schur-complement argument gives invertibility of $\widehat K_k(\theta)$ and the same block-$\ell_1$ inverse bound, hence a deterministic Euclidean bound $\bar\kappa_K$.

For the perturbation estimate, use
\[
\widehat K_k^{-1}-K_\theta^{-1}
=K_\theta^{-1}(K_\theta-\widehat K_k)\widehat K_k^{-1}.
\]
Taking norms and using the two uniform inverse bounds gives the result.
\end{proof}

\paragraph{Role.}
This lemma removes any small-error invertibility event. Empirical primal-dual updates are well defined on every sample path.

Let $\widehat F_k(\theta,x)=\widehat K_k(\theta)x-b$. The stochastic occupancy-ADMM uses one shared $\widehat K_k$ in all updates.

\begin{lemma}[Empirical occupancy-dual identity]
\label{lem:l6-empirical-dual}
If the $x$ update minimizes
\[
c_x^\trans x+\langle y^k,\widehat F_k(\theta^{k+1},x)\rangle
+\frac\beta2\|\widehat F_k(\theta^{k+1},x)\|^2
+\frac1{2\eta_x}\|x-x^k\|^2
\]
and $y^{k+1}=y^k+\beta\widehat F_k(\theta^{k+1},x^{k+1})$, then
\[
y^{k+1}
=-\widehat K_k(\theta^{k+1})^{-\trans}
\left(c_x+\eta_x^{-1}\Delta x_{k+1}\right).
\]
\end{lemma}

\begin{proof}
The first-order condition of the $x$ subproblem is
\[
0=c_x+\widehat K_k(\theta^{k+1})^\trans
\left[y^k+\beta\widehat F_k(\theta^{k+1},x^{k+1})\right]
+\eta_x^{-1}\Delta x_{k+1}.
\]
The bracket equals $y^{k+1}$ by the shared empirical dual update. Lemma~\ref{lem:l6-empirical-occ-invertible} allows inversion.
\end{proof}

\paragraph{Role.}
Sharing the same empirical operator preserves an exact empirical multiplier representation, which is the stochastic counterpart of the Bellman value-dual identity.

\paragraph{Uniform bounds and parameter conditions.}
Let $\bar\kappa_K$ bound both true and empirical inverse operators. Assume $D_\theta K_\theta$ is bounded and the smooth augmented policy gradients have a common Lipschitz constant $\overline L$ on all policy comparison segments. Require
\[
\beta>36\bar\kappa_K^2/\eta_x,\qquad \tau=1/(8\eta_x),\qquad
0<\underline\eta\le\eta_{\theta,k},
\]
\[
\frac1{2\eta_{\theta,k}}-\frac{\overline L}{2}-\frac{3A_C}{2\beta}
\ge a_\theta>0,
\qquad A_C=9\bar\kappa_K^4L_K^2\|c_x\|^2.
\]
These conditions can be met with a fixed sufficiently small policy step. A compact convex parameter set and a twice continuously differentiable policy map supply the required derivative bounds. In particular, the empirical dual identity and the dual update imply
\[
\|y^{k+1}\|\le\bar\kappa_K\|c_x\|+
\frac{\bar\kappa_K}{\eta_x}(\|x^{k+1}\|+\|x^k\|),\qquad
\|x^{k+1}\|\le\bar\kappa_K\|b\|+
\frac{\bar\kappa_K}{\beta}(\|y^{k+1}\|+\|y^k\|).
\]
The two-dimensional recursion from Lemma~\ref{lem:moment-stability}, applied pathwise, has spectral radius less than one because $\beta\eta_x>4\bar\kappa_K^2$. Thus $x^k,y^k$ have deterministic uniform bounds for every sampling path. This also yields a uniform smoothness bound under the compact-set sufficient condition above. All subsequent two-iterate estimates apply for $k\ge1$.

\begin{lemma}[Empirical multiplier increment]
\label{lem:l6-empirical-dual-increment}
Assume $\|K_\theta-K_{\theta'}\|\le L_K\|\theta-\theta'\|$. One may take $A_C=9\bar\kappa_K^4L_K^2\|c_x\|^2$, $B_C=3\bar\kappa_K^2/\eta_x^2$, and $D_C=9\bar\kappa_K^4\|c_x\|^2$, so that
\[
\|\Delta y_{k+1}\|^2
\le A_C\|\Delta\theta_{k+1}\|^2
+B_C\|\Delta x_{k+1}\|^2
+B_C\|\Delta x_k\|^2
+D_C(\delta_k^2+\delta_{k-1}^2).
\]
\end{lemma}

\begin{proof}
Subtract Lemma~\ref{lem:l6-empirical-dual} at iterations $k$ and $k-1$:
\begin{align*}
\Delta y_{k+1}
={}&-\left(\widehat K_k(\theta^{k+1})^{-\trans}-\widehat K_{k-1}(\theta^k)^{-\trans}\right)c_x\\
&-\eta_x^{-1}\widehat K_k(\theta^{k+1})^{-\trans}\Delta x_{k+1}
+\eta_x^{-1}\widehat K_{k-1}(\theta^k)^{-\trans}\Delta x_k.
\end{align*}
The operator difference satisfies
\begin{align*}
\|\widehat K_k(\theta^{k+1})-\widehat K_{k-1}(\theta^k)\|
&\le \|K_{\theta^{k+1}}-K_{\theta^k}\|+\|E_k(\theta^{k+1})\|+\|E_{k-1}(\theta^k)\|\\
&\le L_K\|\Delta\theta_{k+1}\|+\delta_k+\delta_{k-1}.
\end{align*}
Apply the inverse identity with the uniform empirical inverse bound to the first term. Then use $\|a+b+c\|^2\le3(\|a\|^2+\|b\|^2+\|c\|^2)$ and $(u+v+w)^2\le3(u^2+v^2+w^2)$. All fixed coefficients are absorbed into $A_C,B_C,D_C$.
\end{proof}

\paragraph{Role.}
This is the stochastic occupancy analogue of multiplier control. It enters the dual-ascent term of Lemma~\ref{lem:l6-stochastic-descent}.

Let the true residual be $F_k=K_{\theta^k}x^k-b$ and the true augmented Lagrangian be
\[
\cL_\beta^C(\theta,x,y)
=\phi(\theta)+c_x^\trans x+\langle y,K_\theta x-b\rangle
+\frac\beta2\|K_\theta x-b\|^2.
\]
Define
\[
\Phi_k=\cL_\beta^C(\theta^k,x^k,y^k)+\tau\|\Delta x_k\|^2,
\qquad
B_C(\theta)=\phi(\theta)+c_x^\trans K_\theta^{-1}b.
\]

\begin{lemma}[Perturbed true Lyapunov lower bound]
\label{lem:l6-stochastic-lower}
For suitable $\tau$ there are $q_x>0$ and $C_0<\infty$ such that
\[
\Phi_k
\ge B_C^{\inf}
+\frac\beta4\|F_k\|^2
+q_x\|\Delta x_k\|^2
-C_0\delta_{k-1}^2.
\]
Therefore
\[
\Psi_k=\Phi_k-B_C^{\inf}+1+C_0\delta_{k-1}^2
\]
is nonnegative and satisfies $\Psi_k\ge1+(\beta/4)\|F_k\|^2+q_x\|\Delta x_k\|^2$.
\end{lemma}

\begin{proof}
Write $x^k=K_k^{-1}(b+F_k)$. Lemma~\ref{lem:l6-empirical-dual} at the previous iteration gives
\[
y^k=-\widehat K_{k-1}(\theta^k)^{-\trans}(c_x+\eta_x^{-1}\Delta x_k).
\]
Substitute these into the true augmented Lagrangian. The reduced objective contributes $B_C(\theta^k)\ge B_C^{\inf}$. The terms involving $\Delta x_k$ and $F_k$ have the same form as in Lemma~\ref{lem:l6-lyap-lower}. The additional mismatch is produced by
$\widehat K_{k-1}(\theta^k)^{-1}-K_{\theta^k}^{-1}$, whose norm is at most $\kappa_K\bar\kappa_K\delta_{k-1}$ by Lemma~\ref{lem:l6-empirical-occ-invertible}. Thus, for fixed constants,
\[
\Phi_k
\ge B_C^{\inf}
+\frac\beta2\|F_k\|^2
-C_1\|\Delta x_k\|\|F_k\|
-C_2\delta_{k-1}\|F_k\|
+\tau\|\Delta x_k\|^2.
\]
Apply Young's inequality separately to the two cross terms, allocating at most $\beta\|F_k\|^2/8$ to each. This leaves $\beta\|F_k\|^2/4$, a coefficient $q_x=\tau-C_1'$ on $\|\Delta x_k\|^2$, and $-C_0\delta_{k-1}^2$. Choosing $\tau$ above $C_1'$ gives $q_x>0$. Adding the correcting terms in the definition of $\Psi_k$ proves nonnegativity.
\end{proof}

\paragraph{Role.}
This lemma supplies a nonnegative stochastic Lyapunov variable despite empirical-operator perturbations and provides the $\sqrt{\Psi_k}$ bound used in the next lemma.

\begin{lemma}[Empirical policy-gradient perturbation]
\label{lem:l6-policy-gradient-perturbation}
Let
\[
g_{\theta,k}=D_\theta F(\theta^k,x^k)^\trans(y^k+\beta F_k)
\]
be the true smooth augmented policy gradient and let $\widehat g_{\theta,k}$ be the corresponding empirical gradient. If $\sup\|D_\theta K_\theta\|\le G_K$, then
\[
\|\widehat g_{\theta,k}-g_{\theta,k}\|
\le C\delta_k\sqrt{\Psi_k}.
\]
Consequently, for every $\epsilon>0$,
\[
|\langle\widehat g_{\theta,k}-g_{\theta,k},\Delta\theta_{k+1}\rangle|
\le \epsilon\|\Delta\theta_{k+1}\|^2
+C_\epsilon\delta_k^2\Psi_k.
\]
\end{lemma}

\begin{proof}
The error and policy derivative have the block forms
\[
E_kx=\begin{bmatrix}-\gamma(\widehat P_k-P)^\trans q\\0\end{bmatrix},\qquad
D_\theta F(\theta,x)[h]=\begin{bmatrix}0\\-D_\theta\Pi_\theta[h]d\end{bmatrix}.
\]
Also $D_\theta E_k=0$. Hence their inner product is zero for every direction $h$, and
\[
\widehat g_{\theta,k}-g_{\theta,k}
=\beta D_\theta F(\theta^k,x^k)^\trans E_kx^k=0.
\]
Both displayed bounds follow. This cancellation is specific to the block structure of the occupancy constraint.
\end{proof}

\paragraph{Role.}
This lemma is the key benefit of sharing one empirical operator. It converts stochastic policy-gradient error into a multiplicative Lyapunov perturbation without any unbiased-product assumption.

\begin{lemma}[True one-step stochastic descent]
\label{lem:l6-stochastic-descent}
Under the uniform bounds and parameter conditions stated above, for $k\ge1$ there are constants $c_\theta,c_x,c_x',a,b>0$ such that
\[
\E_k\Psi_{k+1}
\le \left(1+\frac{a}{m_k}\right)\Psi_k
-c_\theta\E_k\|\Delta\theta_{k+1}\|^2
-c_x\E_k\|\Delta x_{k+1}\|^2
-c_x'\|\Delta x_k\|^2
+\frac{b}{m_k}.
\]
\end{lemma}

\begin{proof}
The policy gradient is unchanged by the transition perturbation, by Lemma~\ref{lem:l6-policy-gradient-perturbation}. The true policy step therefore decreases the augmented Lagrangian by at least
$(1/(2\eta_{\theta,k})-\overline L/2)\|\Delta\theta_{k+1}\|^2$.

On the value comparison segment, deterministic boundedness of $x,y$ and uniform boundedness of $K,\widehat K$ imply
\[
\|\nabla_x\widehat{\mathcal L}_\beta-\nabla_x\mathcal L_\beta\|
=\|E_k^\trans(y+\beta F)+\beta K^\trans E_kx+\beta E_k^\trans E_kx\|
\le C\delta_k.
\]
Transferring the empirical value-step decrease by Young's inequality gives
$(4\eta_x)^{-1}\|\Delta x_{k+1}\|^2-C\delta_k^2$.
The true dual ascent satisfies
\[
\langle\Delta y_{k+1},F_{k+1}\rangle
\le\frac{3}{2\beta}\|\Delta y_{k+1}\|^2+C\delta_k^2,
\]
since $x^{k+1}$ is uniformly bounded. Substituting Lemma~\ref{lem:l6-empirical-dual-increment} and adding the potential correction yields
\[
\Phi_k-\Phi_{k+1}\ge c_\theta\|\Delta\theta_{k+1}\|^2
+c_x\|\Delta x_{k+1}\|^2+c_x'\|\Delta x_k\|^2
-C_1\delta_k^2-C_2\delta_{k-1}^2,
\]
where
\[
c_\theta\ge a_\theta,\quad
c_x=\frac1{4\eta_x}-\frac{3B_C}{2\beta}-\tau>0,\quad
c_x'=\tau-\frac{3B_C}{2\beta}>0.
\]
In Lemma~\ref{lem:l6-stochastic-lower}, increase $C_0$ to satisfy $C_0\ge C_2$. With $c=\min\{c_\theta,c_x,c_x'\}$, the corrected nonnegative potential then obeys the pathwise bound
\begin{equation}
\Psi_{k+1}\le\Psi_k-cD_{k+1}+C\delta_k^2.
\label{eq:occupancy-additive-descent}
\end{equation}
Conditioning gives $\E_k\Psi_{k+1}\le\Psi_k-c\E_kD_{k+1}+C/m_k$. Since $\Psi_k\ge0$, this also implies the stated inequality with any fixed $a>0$ after enlarging $b$.
\end{proof}

\paragraph{Role.}
This is the stochastic occupancy Lyapunov recursion. Theorem~\ref{thm:l6-stochastic-finite} and Theorem~\ref{thm:l6-stochastic-as} are direct consequences.

Define the true squared KKT residual
\[
G_k^{\rm true}=e_{\theta,k}^2+e_{x,k}^2+\|F_k\|^2,
\]
and let
$D_{k+1}=\|\Delta\theta_{k+1}\|^2+\|\Delta x_{k+1}\|^2+\|\Delta x_k\|^2$.

\begin{lemma}[True KKT residual bound]
\label{lem:l6-stochastic-kkt}
For $k\ge1$, there is $C_G<\infty$ such that
\[
\E_kG_{k+1}^{\rm true}
\le C_G\E_kD_{k+1}
+\frac{C_G}{m_k}\Psi_k
+\frac{C_G}{m_k}+C_G\delta_{k-1}^2.
\]
\end{lemma}

\begin{proof}
Use the composite policy Fermat condition, the equality of true and empirical policy gradients, the uniform policy smoothness bound, and $\eta_{\theta,k}^{-1}\le\underline\eta^{-1}$. They give
\[
e_{\theta,k+1}\le C(\|\Delta\theta_{k+1}\|+\|\Delta x_{k+1}\|+\|\Delta y_{k+1}\|).
\]
The empirical dual identity and update give
\[
c_x+K_{k+1}^\trans y^{k+1}
=-\eta_x^{-1}\Delta x_{k+1}-E_k^\trans y^{k+1},\qquad
F_{k+1}=\beta^{-1}\Delta y_{k+1}-E_kx^{k+1}.
\]
Uniform boundedness of $x^{k+1},y^{k+1}$ and Lemma~\ref{lem:l6-empirical-dual-increment} imply the pathwise estimate
\[
G_{k+1}^{\rm true}\le C\{D_{k+1}+\delta_k^2+\delta_{k-1}^2\}.
\]
Taking $\cF_k$-conditional expectations leaves the previous error $\delta_{k-1}^2$ unchanged and bounds the current error by $C/m_k$. This is stronger than the displayed estimate because $\Psi_k\ge0$.
\end{proof}

\paragraph{Role.}
This lemma translates Lyapunov descent into the true optimization residual and is the final input to the stochastic complexity theorem.

Let $A_T=\sum_{k=0}^{T-1}m_k^{-1}$.

\begin{theorem}[Finite-time stochastic occupancy theory]
\label{thm:l6-stochastic-finite}
There are constants $C_1,C_2,C_3,C_{\rm init}$ independent of the batch schedule such that
\[
\sup_{1\le k\le T}\E\Psi_k
\le e^{C_1A_T}(C_{\rm init}+C_2A_T).
\]
If $K$ is uniform on $\{1,\ldots,T\}$ and independent of the algorithmic randomness, then
\[
\E G_K^{\rm true}
\le C_3e^{C_1A_T}(1+A_T)
\left(\frac1T+\frac{A_T}{T}\right).
\]
\end{theorem}

\begin{proof}
The first-step comparison with $\theta^0$ and the deterministic iterate bounds give uniform upper bounds on $\E\Psi_1$ and $\E G_1^{\rm true}$. Absorb these into $C_{\rm init}$. Summing the additive bound \eqref{eq:occupancy-additive-descent} in expectation and using $\Psi_k\ge0$ yields
\[
\sup_{1\le k\le T}\E\Psi_k\le C_{\rm init}+CA_T,\qquad
\sum_{k=1}^{T-1}\E D_{k+1}\le C(1+A_T).
\]
The pathwise residual estimate in Lemma~\ref{lem:l6-stochastic-kkt} has two consecutive operator errors. Their expected sum is at most $2CA_T$. Including the first residual gives
\[
\frac1T\sum_{j=1}^T\E G_j^{\rm true}\le C(1+A_T)/T.
\]
This proves the stated bounds after increasing constants, since the exponential and additional $(1+A_T)$ factors are at least one. Uniform randomization identifies the left side with $\E G_K^{\rm true}$.
\end{proof}

\paragraph{Role.}
This theorem gives the finite-time guarantee for the stochastic occupancy operator and verifies that the operator-stability mechanism survives a second, explicit RL formulation.

\begin{corollary}[Equal-batch complexity]
\label{cor:l6-equal-batch}
If $m_k=m=N/T$, then $A_T=T^2/N$. If $N\ge c_0T^2$ for a fixed $c_0>0$, then
\[
\E G_K^{\rm true}=O\!\left(\frac1T+\frac{T}{N}\right).
\]
For an unsquared KKT target $\epsilon$, it is sufficient to take $T=O(\epsilon^{-2})$ and $N=O(\epsilon^{-4})$.
\end{corollary}

\begin{proof}
Equal batches give $A_T=T/m=T^2/N$. If $N\ge c_0T^2$, then $A_T\le1/c_0$, so the exponential and $(1+A_T)$ factors in Theorem~\ref{thm:l6-stochastic-finite} are bounded by constants. Moreover $A_T/T=T/N$, proving the rate. Jensen gives $\E\sqrt{G_K^{\rm true}}\le\sqrt{\E G_K^{\rm true}}$. Requiring the squared residual to be $O(\epsilon^2)$ yields $T=O(\epsilon^{-2})$ and $N=O(\epsilon^{-4})$.
\end{proof}

\paragraph{Role.}
The corollary puts the stochastic occupancy theorem on the same squared-versus-unsquared residual scale as the main Bellman complexity result.

\begin{theorem}[Almost-sure stochastic occupancy convergence]
\label{thm:l6-stochastic-as}
If $\sum_km_k^{-1}<\infty$, then $\Psi_k$ converges almost surely,
\[
\sum_kD_{k+1}<\infty
\quad\text{almost surely},
\]
$\delta_k\to0$, and $G_k^{\rm true}\to0$ almost surely.
\end{theorem}

\begin{proof}
Lemma~\ref{lem:l6-regenerative} and Tonelli's theorem give
$\sum_k\delta_k^2<\infty$ almost surely. Fix such a path. From \eqref{eq:occupancy-additive-descent},
\[
\Psi_{k+1}+C\sum_{j\ge k+1}\delta_j^2
\le\Psi_k+C\sum_{j\ge k}\delta_j^2-cD_{k+1}.
\]
The corrected sequence is nonnegative and decreasing. It converges and $\sum_kD_{k+1}<\infty$. Since the error tails vanish, $\Psi_k$ converges as well. The pathwise residual estimate in Lemma~\ref{lem:l6-stochastic-kkt} then gives $G_k^{\rm true}\to0$.
\end{proof}

\paragraph{Role.}
This theorem is the asymptotic stochastic counterpart of Theorem~\ref{thm:l6-occ-convergence}.

\begin{corollary}[Conditional nonlinear-policy performance]
\label{cor:l6-performance}
Assume in addition a parameter-space domination inequality
\[
J^\star-J(\pi_\theta)\le C_{\rm par}R_\theta(\theta),
\qquad
R_\theta(\theta)^2\le C_GG^{\rm true}.
\]
Then
\[
\E[J^\star-J(\pi_{\theta_K})]
\le C\sqrt{\E G_K^{\rm true}}.
\]
If instead $J^\star-J(\pi_\theta)\le C R_\theta(\theta)^2$, then the performance gap is $O(\E G_K^{\rm true})$.
\end{corollary}

\begin{proof}
The first two inequalities imply
$J^\star-J(\pi_{\theta_K})\le C\sqrt{G_K^{\rm true}}$. Taking expectations and applying Jensen proves the first claim. Under quadratic parameter-space domination,
$J^\star-J(\pi_{\theta_K})\le CC_GG_K^{\rm true}$, and expectation preserves the squared-residual rate.
\end{proof}

\paragraph{Role.}
This corollary states the domination conditions for performance guarantees with nonlinear policies.

\section{Additional Experimental Results}
\label{app:experiments}

This section provides additional numerical evidence for the structural
predictions studied in Sections~\ref{sec:experiments}.
All experiments use finite discounted MDPs with direct tabular policies and the
same reset-controlled Markov sampling mechanism as in the theoretical analysis.
Unless stated otherwise, results are averaged over multiple independently
generated MDP instances and trajectory seeds, and error bars denote 95\%
confidence intervals.

The experiments are designed to isolate four components of the theory:
(i) the statistical contribution of the Markov batch size,
(ii) the correlated residual--Jacobian product created by shared sampling,
(iii) the conversion from KKT residual to policy performance, and
(iv) the numerical sensitivity to the ADMM penalty parameter.
The corresponding main structural tests for shared-operator stability and
finite-time scaling are reported in Section~\ref{sec:experiments}.

\subsection{Batch-Size Sensitivity}
\label{app:batch_sensitivity}

The finite-time result predicts
\[
    \mathbb{E}[\widetilde G_{K+1}]
    \le
    \frac{A}{T}
    +
    \frac{B}{T}\sum_{k<T}\frac{1}{m_k}.
\]
For a constant batch size $m_k=m$, this reduces to
\[
    O(T^{-1}+m^{-1}).
\]
We therefore fix the outer iteration budget at $T=240$ and vary
$m\in\{60,120,240,480,960\}$.

Figure~\ref{fig:batch_sensitivity} shows that the mean companion squared KKT
residual decreases monotonically with the batch size. A log--log fit over the
tested range gives slope approximately $-0.87$, close to the $m^{-1}$
statistical contribution predicted by the finite-time bound. The deviation
from an exact $-1$ slope is expected because the bound contains the additional
$A/T$ term, which produces a finite-iteration floor as $m$ increases.

\begin{figure}[t]
    \centering
    \includegraphics[width=0.58\linewidth]
    {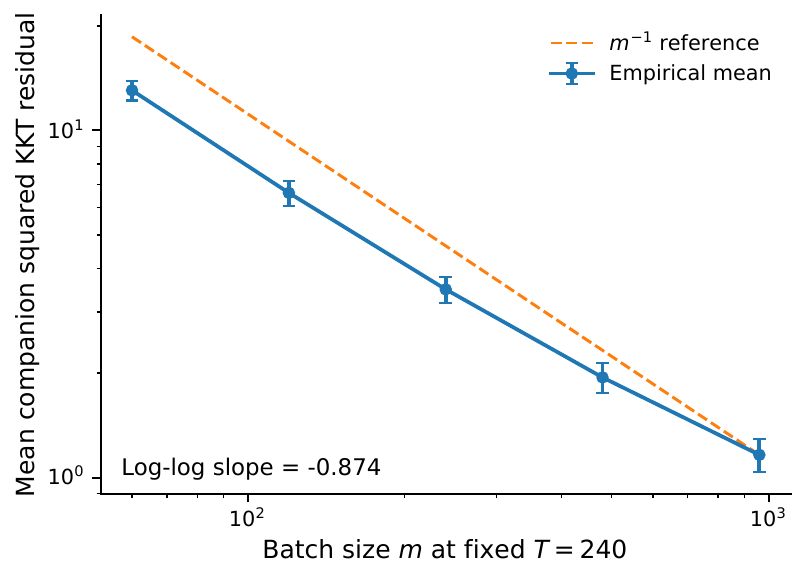}
    \caption{
    \textbf{Batch-size sensitivity.}
    Mean companion squared KKT residual at fixed $T=240$ as the Markov batch
    size increases. The dashed line shows an $m^{-1}$ reference slope.
    }
    \label{fig:batch_sensitivity}
\end{figure}

\subsection{Residual--Jacobian Product under Shared Sampling}
\label{app:product_bias}

A central difficulty in the stochastic analysis is that the Bellman residual
and its Jacobian are constructed from the same empirical operator. Writing
\[
    e=\widehat R-R,
    \qquad
    G=\widehat J-J,
\]
the stochastic policy-gradient perturbation contains the correlated term
\[
    \beta G^\top e.
\]
Separate unbiasedness of $G$ and $e$ does not imply that this product has zero
mean. Remark~\ref{rem:product-bias} instead controls the product pathwise
through the shared empirical operator error:
\[
    \|G^\top e\|
    =
    O(\epsilon_k^2).
\]
Since the empirical operator satisfies
$\mathbb{E}[\epsilon_k^2]=O(m^{-1})$, the resulting perturbation is predicted
to occur at the same $m^{-1}$ scale.

To isolate this effect, we fix a primal--dual point and repeatedly construct
empirical Bellman operators with different batch sizes. We compare two
estimators. The \emph{same-batch} estimator uses one empirical operator for
both $G$ and $e$, matching the algorithm. The \emph{independent-batch}
estimator constructs them from separate samples, thereby suppressing the
correlation at the cost of double sampling.

As shown in Figure~\ref{fig:product_bias}, the same-batch product decays with
empirical log--log slope approximately $-1.13$, consistent with the
$O(m^{-1})$ prediction. Independent batches produce a substantially smaller
mean product term, as expected from decorrelation. The latter, however, no
longer represents the single shared empirical operator used by the algorithm.
Together with the exact value--dual identity experiment in the main text,
this result illustrates the role of structure-preserving stochasticization:
the algorithm retains the coupled empirical operator and controls the induced
product perturbation rather than removing it through double sampling.

\begin{figure}[t]
    \centering
    \includegraphics[width=0.58\linewidth]
    {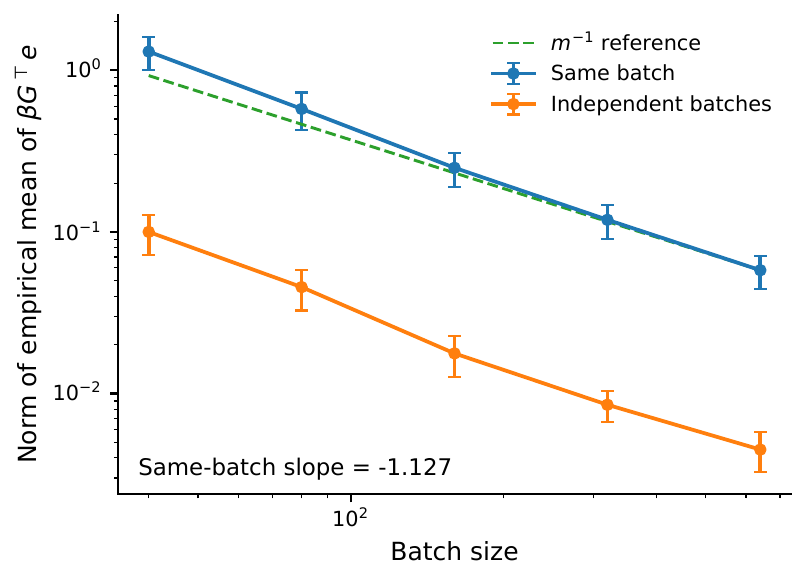}
    \caption{
    \textbf{Residual--Jacobian product scaling.}
    Magnitude of the empirical product term $\beta G^\top e$ as the batch
    size increases. The same-batch construction follows approximately the
    predicted $m^{-1}$ scale. Independent batches reduce the correlation but
    require separate samples.
    }
    \label{fig:product_bias}
\end{figure}

\subsection{From KKT Residual to Policy Performance}
\label{app:kkt_performance}

Section~\ref{sec:performance} establishes that, under discounted occupancy
coverage, the optimization residual has a direct RL interpretation. In
particular, the policy stationarity residual is controlled by the KKT residual,
and the policy-performance gap consequently satisfies
\[
    J^\star-J(\pi)
    =
    O(\sqrt{G}).
\]
We examine this relationship along the iterates generated by the Bellman-ADMM
dynamics.

For every sampled iterate, we compute the true KKT residual $G$ using the
underlying MDP model and evaluate the corresponding policy return exactly.
Figure~\ref{fig:kkt_performance} plots
$J^\star-J(\pi)$ against $\sqrt{G}$. The points exhibit the predicted
first-order relation throughout the observed regime. The empirical 95th
percentile of
\[
    \frac{J^\star-J(\pi)}{\sqrt{G}}
\]
is $2.72$, while the ratio remains bounded across the collected iterates.
This experiment complements the optimization-based residual results by showing
that decreasing KKT error is accompanied by decreasing policy suboptimality.

\begin{figure}[t]
    \centering
    \includegraphics[width=0.58\linewidth]
    {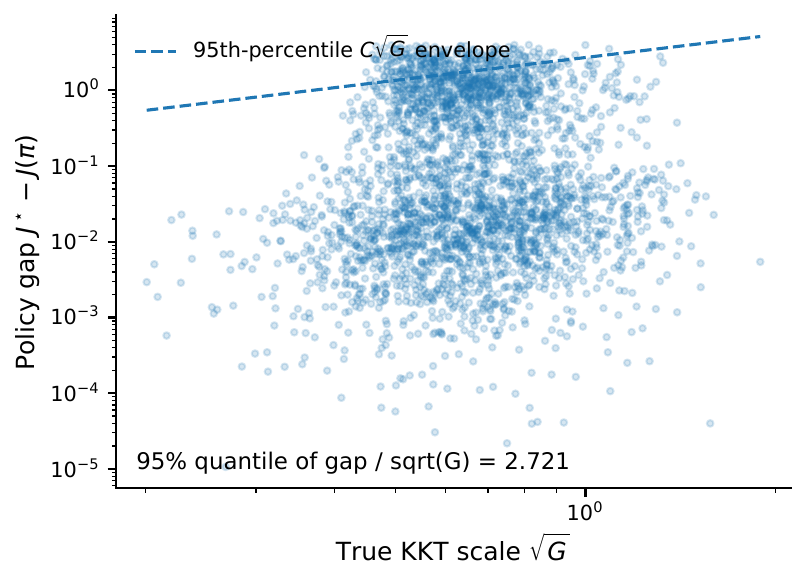}
    \caption{
    \textbf{KKT residual and policy performance.}
    Policy suboptimality versus the true KKT scale $\sqrt{G}$.
    The dashed line denotes the empirical 95th-percentile
    $C\sqrt{G}$ envelope.
    }
    \label{fig:kkt_performance}
\end{figure}

\subsection{Sensitivity to the Penalty Parameter}
\label{app:beta_sensitivity}

The convergence analysis imposes a sufficient lower bound on the ADMM penalty
parameter $\beta$ to guarantee positivity of the coefficients in the Lyapunov
descent argument. Such analytical margins are designed to hold uniformly over
the admissible problem class and need not coincide with empirically optimal
parameter choices.

We therefore vary
\[
    \beta\in\{2.5,5,10,20,40\}
\]
while keeping the remaining algorithmic parameters fixed. As shown in
Figure~\ref{fig:beta_sensitivity}, the method remains numerically stable over
the entire tested range. The trailing KKT residual changes smoothly with
$\beta$, with smaller penalties yielding lower residuals on these instances.
This behavior is consistent with the role of the theoretical condition as a
uniform sufficient stability margin rather than a tuning prescription for
individual MDPs.

\begin{figure}[t]
    \centering
    \includegraphics[width=0.58\linewidth]
    {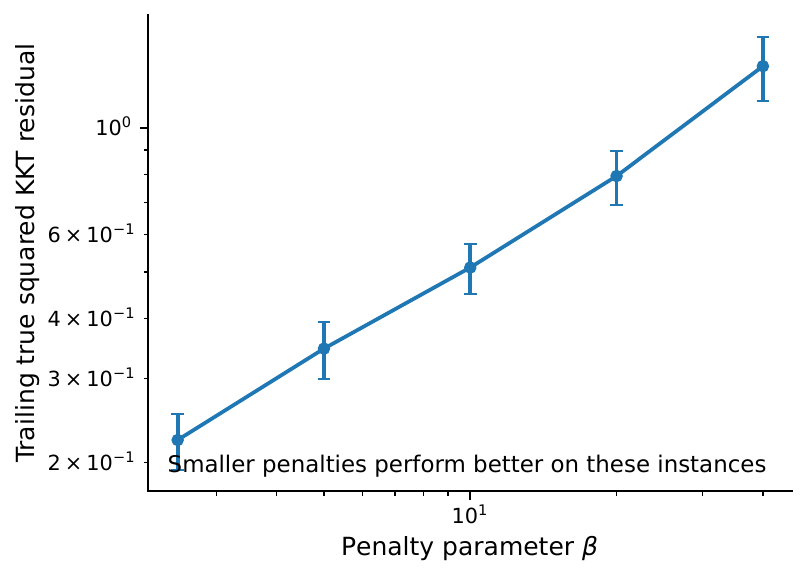}
    \caption{
    \textbf{Penalty sensitivity.}
    Trailing true squared KKT residual as a function of the ADMM penalty
    parameter $\beta$.
    }
    \label{fig:beta_sensitivity}
\end{figure}

\subsection{Operator Consistency: Experimental Design}
\label{app:operator_consistency}

The following ablations separate the role of empirical-operator consistency
from the precision of each empirical model and the total sampling cost.
We compare three constructions: \emph{all shared}, \emph{value--dual shared},
and \emph{all separate}. Table~\ref{tab:oc_design} specifies the model used
by each update. The value--dual-shared construction isolates the consistency
needed by the value--dual identity, while retaining an independently estimated
model for the policy update.

\begin{table}[ht]
\centering
\caption{Empirical-model assignments. Distinct indices denote separately sampled models. Under an equal per-model batch size $m$, the last column gives the total number of reset-controlled sampling steps per iteration.}
\label{tab:oc_design}
\small
\begin{tabular}{lcccc}
\toprule
Construction & Policy & Value & Dual & Total steps \\
\midrule
All shared & 0 & 0 & 0 & $m$ \\
Value--dual shared & 0 & 1 & 1 & $2m$ \\
All separate & 0 & 1 & 2 & $3m$ \\
\bottomrule
\end{tabular}
\end{table}

\paragraph{Instances and parameters.}
Each MDP has five states and two actions. Every transition row is sampled
independently from $\operatorname{Dirichlet}(0.7\mathbf{1}_5)$.
Rewards are independent $\mathcal{N}(0,0.45^2)$ draws plus statewise action
trends. For states $s=0,\ldots,4$, these trends are $0.30-0.125s$ for
action 0 and $-0.35+0.20s$ for action 1. The generated transition kernel
and mean rewards remain fixed during each run.
The reset, training, and evaluation distributions are uniform,
$\rho=\mu=\nu=\mathbf{1}_5/5$, with $c=-\mu$ and $\phi\equiv0$.
Unless specified otherwise, $\gamma=0.85$, $\beta=10$, and
$\eta_\pi=\eta_V=0.15$. We initialize $\pi^0$ uniformly and set
$V^0=\lambda^0=0$. Each run uses 180 outer iterations.
The MDP seeds are $3,7,11,19,23,31,43,59$, and each instance uses trajectory
seeds $0,1,2,3,4$.

\paragraph{Markov sampling and budget accounting.}
At iteration $k$, the behavior policy is frozen at
$\bar\pi^k=0.7\pi^k+0.3\mathbf{1}_2/2$.
Each step resets to $\rho$ with probability $1-\gamma$ and otherwise samples
an environment transition. Environment rewards are observed with independent
$\mathcal{N}(0,0.3^2)$ noise. Only non-reset transitions contribute to the
empirical transition and reward rows. Unvisited rows use the uniform
transition distribution and reward zero.
Each distinct model has a separate random stream and a persistent chain state,
initialized at state 0. Corresponding streams use matched random seeds across
constructions. With adaptive behavior, subsequent observations can differ as
the policies evolve. A fixed-behavior control below removes this dependence.

We vary $m\in\{80,160,320,640\}$ under two budget conventions.
For \emph{equal total budget}, $m$ reset-controlled steps are divided among
one, two, or three models, with remainders assigned to the earliest model
indices. For \emph{equal per-model batch}, every distinct model receives
$m$ steps, giving total costs $m$, $2m$, and $3m$.
Reset steps are included in these costs. Counts of environment transitions
and resets are also recorded separately.
All three constructions in this study use the same stream protocol.
The original two-way experiment uses serial blocks of one trajectory and
assigns $\lfloor m/3\rfloor$ steps to each separate model. The new controls
use exactly $m$ total steps in the equal-budget comparison.

\paragraph{Exact subproblem solutions.}
The two-action policy subproblem separates over states. Write
$x_s^k=\pi^k(1\mid s)$,
$q_{sa}=\widehat r^p_k(s,a)+\gamma\sum_{s'}\widehat P^p_k(s'\mid s,a)V^k(s')$,
$d_s=q_{s1}-q_{s0}$, and $b_s=V^k(s)-q_{s0}$, where the superscript $p$
denotes the policy model. Its exact solution is
\begin{equation}
 x_s^{k+1}=\operatorname{proj}_{[0,1]}
 \frac{d_s(\lambda_s^k+\beta b_s)+2x_s^k/\eta_\pi}
 {\beta d_s^2+2/\eta_\pi}.
 \label{eq:oc_policy_solve}
\end{equation}
Let $\widehat M_k^v$ and $\widehat r_k^v$ denote the value model evaluated at
$\pi^{k+1}$. The value update solves
\begin{equation}
 \left[\beta(\widehat M_k^v)^\top\widehat M_k^v+\eta_V^{-1}I\right]V^{k+1}
 =-c-(\widehat M_k^v)^\top\lambda^k
   +\beta(\widehat M_k^v)^\top\widehat r_k^v+\eta_V^{-1}V^k.
 \label{eq:oc_value_solve}
\end{equation}
The coefficient matrix is positive definite. These exact solves ensure that
subproblem approximation does not enter the comparison.

\paragraph{Metrics and statistical aggregation.}
We evaluate the true squared KKT residual
\begin{equation}
 G(\pi,V,\lambda)=
 \operatorname{dist}^2\!\left(0,N_\Pi(\pi)+J_\pi(\pi,V)^\top\lambda\right)
 +\|c+M_\pi^\top\lambda\|^2+\|R(\pi,V)\|^2
 \label{eq:oc_G}
\end{equation}
using the underlying MDP. Policy return is evaluated by an exact Bellman
solve, and $J^\star$ is obtained by enumerating the $2^5$ deterministic
policies. We also record $\|\lambda^{k+1}-\lambda^k\|^2$, the return gap
$J^\star-J(\pi^{k+1})$, and the identity defect defined below.
Unless stated otherwise, each reported outcome averages iterations 151--180.
We first average the five trajectory seeds within each MDP, then average
the eight MDP means. Figure intervals are 95\% percentile bootstrap intervals
from 10,000 resamples of these eight means. Paired comparisons use the eight
within-MDP differences and a 95\% Student-$t$ interval with seven degrees of
freedom. These are individual, unadjusted intervals.

\subsection{Identifying the Value--Dual Consistency Mechanism}
\label{app:oc_identity}

Let $\widehat R_k^v$ and $\widehat R_k^d$ denote the empirical residuals
used by the value and dual updates. In this subsection, both are evaluated
at $(\pi^{k+1},V^{k+1})$, and $\widehat M_k^v$ is evaluated at $\pi^{k+1}$.
The exact value solve and the multiplier update give
\begin{align}
 0&=c+(\widehat M_k^v)^\top\lambda^k
       +\beta(\widehat M_k^v)^\top\widehat R_k^v
       +\eta_V^{-1}\Delta V_{k+1},\\
 \lambda^{k+1}&=\lambda^k+\beta\widehat R_k^d.
\end{align}
Substitution yields the pathwise relation
\begin{equation}
 \underbrace{c+(\widehat M_k^v)^\top\lambda^{k+1}
        +\eta_V^{-1}\Delta V_{k+1}}_{h_{k+1}}
 =\beta(\widehat M_k^v)^\top
       (\widehat R_k^d-\widehat R_k^v).
 \label{eq:oc_defect}
\end{equation}
Thus, sharing the value and dual models sets the right-hand side to zero,
regardless of the model used in the policy update. This identity explains
why both all-shared and value--dual-shared updates retain an explicit
multiplier representation. All sharing also reuses the same empirical model
for the policy update, reducing the number of sampled models.

We measure the identity defect $D_{k+1}=\|h_{k+1}\|_2$ and verify
Equation~\eqref{eq:oc_defect} at every update.
At equal per-model batch size 320, the mean trailing defects are
$1.05\times10^{-14}$ for both sharing constructions and $1.793$ for all
separate updates. Figure~\ref{fig:oc_identity} shows the same distinction
across batch sizes. Across the full experiment suite, including the
large-penalty check below, the maximum norm of the difference between the
two sides of Equation~\eqref{eq:oc_defect} is $2.02\times10^{-10}$.
The value--dual-shared construction serves as a mechanism ablation.
The convergence results in the main text apply to the specified all-shared
algorithm.

\begin{figure}[t]
\centering
\includegraphics[width=\linewidth]{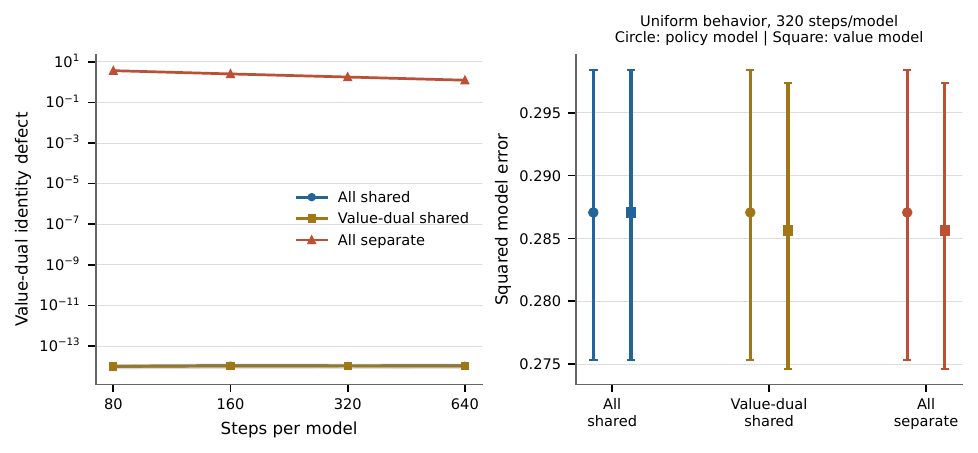}
\caption{\textbf{Identity preservation and estimator precision.}
Left: trailing identity defect under equal per-model batches and adaptive
behavior. All-shared and value--dual-shared curves overlap at numerical
precision. Right: empirical-model squared error under fixed uniform behavior
and 320 steps per model. Circles denote policy models and squares denote
value models. Model error is measured as the sum of squared Frobenius
transition error and squared Euclidean reward error. Error intervals
resample MDP means.}
\label{fig:oc_identity}
\end{figure}

\subsection{Separating Operator Consistency from Sampling Cost}
\label{app:oc_budgets}

Figure~\ref{fig:oc_budget} compares true squared KKT residuals, squared
multiplier increments, and return gaps under both budget conventions.
Table~\ref{tab:oc_m320} reports the batch-320 slice.
At equal per-model batch size, the mean residual is $0.503$ for all shared
and $4.159$ for all separate, a ratio of $8.27$.
All sharing uses one third of the total sampling steps in this comparison.
Value--dual sharing attains residual $0.639$, placing its optimization
behavior closer to all sharing while preserving the same multiplier
identity. Under equal total budget, the corresponding residuals are
$0.503$, $1.381$, and $13.192$. These complementary comparisons demonstrate
the contribution of consistency at matched model batch sizes and the
sampling benefit of reusing one model.

\begin{table}[t]
\centering
\caption{\textbf{Batch-320 comparison.} Entries average the last 30 of 180
iterations and then the eight MDP means. $G$ is the true squared KKT
residual, $\Delta_\lambda$ is the squared multiplier increment, and
$\Delta_J=J^\star-J(\pi)$ is the return gap. Total cost includes reset steps.}
\label{tab:oc_m320}
\small
\setlength{\tabcolsep}{5pt}
\begin{tabular}{llrrrr}
\toprule
Budget & Construction & Steps & $G$ & $\Delta_\lambda$ & $\Delta_J$ \\
\midrule
Equal total & All shared & 320 & 0.5028 & 1.0452 & 0.01425 \\
& Value--dual shared & 320 & 1.3813 & 2.2538 & 0.05144 \\
& All separate & 320 & 13.1920 & 12.8006 & 0.17656 \\
\midrule
Equal per model & All shared & 320 & 0.5028 & 1.0452 & 0.01425 \\
& Value--dual shared & 640 & 0.6390 & 1.1283 & 0.02419 \\
& All separate & 960 & 4.1585 & 4.0575 & 0.04168 \\
\bottomrule
\end{tabular}
\end{table}

\begin{figure}[t]
\centering
\includegraphics[width=\linewidth]{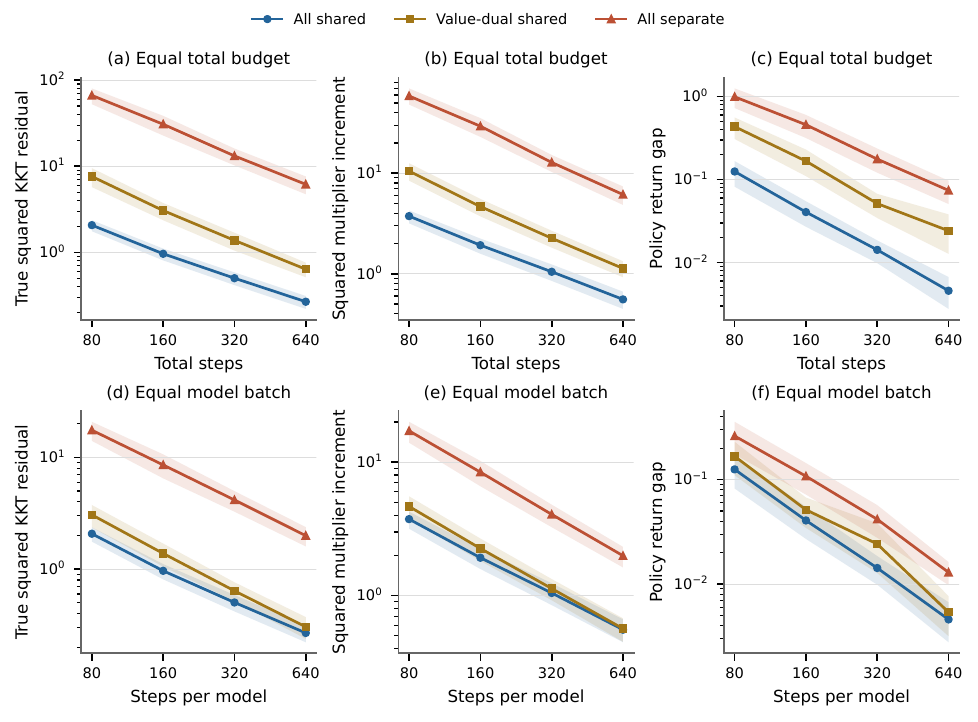}
\caption{\textbf{Consistency and sampling-budget controls.}
Top: equal total steps per iteration. Bottom: equal steps per empirical
model, with total costs $m$, $2m$, and $3m$ for all shared, value--dual
shared, and all separate. Columns show the true squared KKT residual,
squared multiplier increment, and return gap. Curves average the last
30 of 180 iterations. Shading denotes 95\% bootstrap intervals over
MDP means.}
\label{fig:oc_budget}
\end{figure}

The paired differences in Table~\ref{tab:oc_paired} quantify variation
across instances. At equal per-model batch size 320, all separate exceeds
all shared in squared KKT residual by $3.656$, with interval
$[2.617,4.694]$. Its squared multiplier increment exceeds all shared by
$3.012$, with interval $[2.213,3.812]$.
The all-shared versus value--dual-shared distinction is supported for the
residual and multiplier increment. Their return-gap difference has an
interval containing zero. We therefore use this comparison to locate the
operator-consistency mechanism and quantify its sampling cost.

\begin{table}[t]
\centering
\caption{\textbf{Paired contrasts at equal per-model batch size 320.}
Each contrast subtracts the second construction from the first.
Intervals are based on eight paired MDP means.}
\label{tab:oc_paired}
\small
\begin{tabular}{llrr}
\toprule
Metric & Contrast & Mean difference & 95\% interval \\
\midrule
$G$ & Value--dual shared $-$ all shared & 0.1361 & $[0.0463,0.2260]$ \\
& All separate $-$ all shared & 3.6557 & $[2.6170,4.6944]$ \\
& All separate $-$ value--dual shared & 3.5196 & $[2.5340,4.5052]$ \\
\midrule
$\Delta_\lambda$ & Value--dual shared $-$ all shared & 0.0831 & $[0.0205,0.1457]$ \\
& All separate $-$ all shared & 3.0123 & $[2.2128,3.8118]$ \\
& All separate $-$ value--dual shared & 2.9293 & $[2.1526,3.7060]$ \\
\midrule
$\Delta_J$ & Value--dual shared $-$ all shared & 0.00994 & $[-0.00384,0.02372]$ \\
& All separate $-$ all shared & 0.02742 & $[0.01143,0.04342]$ \\
& All separate $-$ value--dual shared & 0.01748 & $[0.01011,0.02486]$ \\
\bottomrule
\end{tabular}
\end{table}

\subsection{Fixed-Behavior Control}
\label{app:oc_uniform}

To separate operator consistency from changes in data collection, we repeat
the comparisons with $\bar\pi^k(a\mid s)=1/2$ at every iteration.
At equal per-model batch size 320, the policy model is identical across all
three constructions under the matched streams. The value model is also
identical between value--dual shared and all separate. All shared reuses
its policy model for the value and dual updates. Each distinct model follows
the same estimation procedure and sampling distribution.
This control removes the feedback from the learned policy to the
sampling distribution.

\begin{table}[t]
\centering
\caption{\textbf{Uniform-behavior control at 320 steps per model.}
Total costs are 320, 640, and 960 steps, respectively. The notation and
aggregation match Table~\ref{tab:oc_m320}.}
\label{tab:oc_uniform}
\small
\begin{tabular}{lrrr}
\toprule
Construction & $G$ & $\Delta_\lambda$ & $\Delta_J$ \\
\midrule
All shared & 0.8505 & 1.7776 & 0.02818 \\
Value--dual shared & 1.1477 & 1.8908 & 0.04018 \\
All separate & 7.0014 & 6.7681 & 0.07335 \\
\bottomrule
\end{tabular}
\end{table}

Table~\ref{tab:oc_uniform} retains the same ordering of the three residuals.
The paired residual difference between all separate and all shared is
$6.151$, with 95\% interval $[4.013,8.288]$.
Between all separate and value--dual shared it is $5.854$, with interval
$[3.882,7.825]$. Thus, the residual separation persists when the behavior
policy and per-model sampling distributions are fixed.

\subsection{Discounted-Resolvent Sensitivity}
\label{app:oc_discount}

The inverse-map argument predicts that the sensitivity of the dual
representation depends on the discounted resolvent.
To examine this dependence with fixed estimation error, we collect one
320-step batch using uniform behavior and sampling discount $0.85$ for
each MDP and trajectory seed. We then hold both the empirical transition
matrix and the uniform evaluation policy fixed while varying
$\gamma\in\{0.50,0.70,0.85,0.95,0.98\}$.
Write $M_\gamma=I-\gamma P_\pi$ and
$\widehat M_\gamma=I-\gamma\widehat P_\pi$.
The resolvent identity gives
\begin{equation}
 \|\widehat M_\gamma^{-1}-M_\gamma^{-1}\|_2
 \le \gamma\|\widehat M_\gamma^{-1}\|_2
             \|\widehat P_\pi-P_\pi\|_2\|M_\gamma^{-1}\|_2.
 \label{eq:oc_inverse}
\end{equation}
We evaluate the stationary dual representations
$\lambda_\gamma=M_\gamma^{-\top}\mu$ and
$\widehat\lambda_\gamma=\widehat M_\gamma^{-\top}\mu$.
Figure~\ref{fig:oc_frozen} reports the true inverse norm, the absolute dual
error, and the relative error
$\|\widehat\lambda_\gamma-\lambda_\gamma\|_2/\|\lambda_\gamma\|_2$.
As $\gamma$ increases from $0.50$ to $0.98$, the mean inverse norm rises
from $2.03$ to $52.15$, the absolute dual error from $0.0462$ to $2.2808$,
and the relative error from $0.0512$ to $0.0978$.
The relative metric accounts for the change in discounted occupancy mass,
$\mathbf{1}^\top\lambda_\gamma=(1-\gamma)^{-1}$.
These fixed-perturbation measurements support the sensitivity mechanism
used in multiplier control without estimating a worst-case asymptotic rate.

\begin{figure}[t]
\centering
\includegraphics[width=\linewidth]{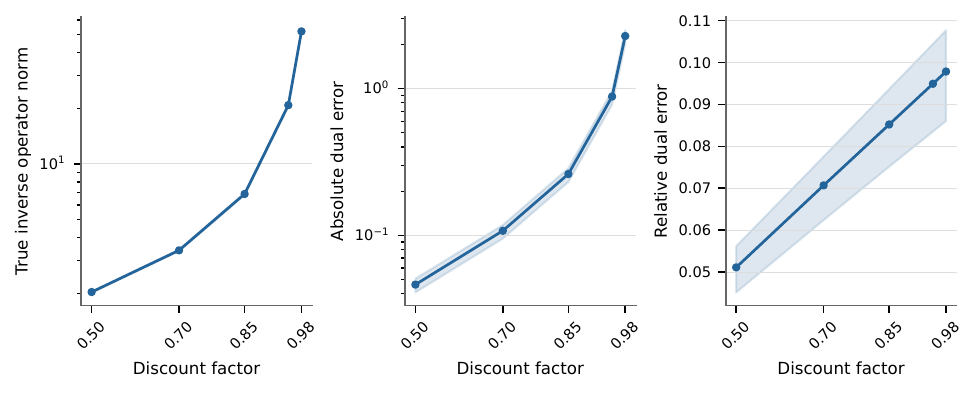}
\caption{\textbf{Discount sensitivity with fixed policy and transition error.}
One empirical transition model per run is reused across all discounts.
Panels show the true resolvent norm and the absolute and relative errors
in the stationary dual representation. Shading denotes 95\% bootstrap
intervals over MDP means.}
\label{fig:oc_frozen}
\end{figure}

We also run the three algorithms at the same discount grid using uniform
behavior and 320 steps per model, holding the remaining parameters fixed.
Figure~\ref{fig:oc_dynamic_discount} reports these dynamic comparisons.
Here, the reset probability is $1-\gamma$, so the experiment includes both
operator sensitivity and changes in the sampling process.
All-shared updates have lower mean KKT residuals than all-separate updates
throughout this grid. The fixed-perturbation experiment above isolates the
inverse-operator component of this dependence.

\begin{figure}[t]
\centering
\includegraphics[width=\linewidth]{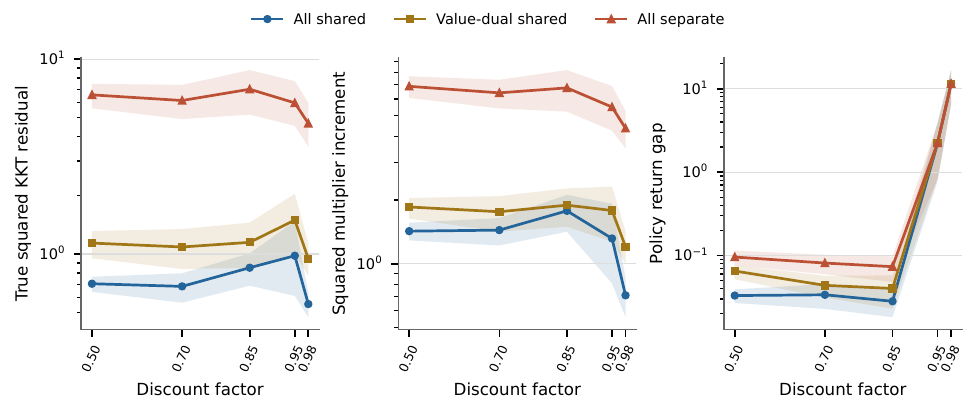}
\caption{\textbf{Dynamic discount sweep.} Uniform behavior and 320 steps
per model, with all other algorithmic parameters fixed. The reset
probability changes with the discount. Outcomes average the last 30 of
180 iterations, and shaded intervals resample MDP means.}
\label{fig:oc_dynamic_discount}
\end{figure}

\subsection{Parameter-Margin Check and Numerical Verification}
\label{app:oc_verification}

The preceding finite-budget comparisons use $\beta=10$.
A separate parameter-margin check uses the uniform bounds
$\bar\kappa=\sqrt{5}/(1-\gamma)$ and $L_M=\gamma\sqrt{2}$.
At $\gamma=0.85$ and $\eta_\pi=\eta_V=0.15$, the sufficient lower bound
in the main text is
\begin{equation}
 \beta>\max\left\{\frac{60\bar\kappa^2}{\eta_V},
 30\eta_\pi\bar\kappa^4L_M^2\|c\|^2\right\}=88888.889.
\end{equation}
We set $\beta=97777.778$, which is 1.1 times this threshold, and run
all three constructions with uniform behavior and 320 steps per model.
The three coefficients in the all-shared Lyapunov descent bound are
$0.57197$, $0.07576$, and $0.07576$, respectively.
Table~\ref{tab:oc_margin} reports the corresponding finite-time outcomes.
This check concerns the penalty margin of the all-shared theory.
The main comparisons assess finite-budget behavior, and the fixed-batch
runs do not invoke the square-summability condition for almost-sure
asymptotic convergence.

\begin{table}[t]
\centering
\caption{\textbf{Large-penalty check.} Uniform behavior, 320 steps per model,
and $\beta=97777.778$. All quantities use the same trailing aggregation
as the other controls. $D$ denotes the identity defect.}
\label{tab:oc_margin}
\small
\begin{tabular}{lrrrr}
\toprule
Construction & $G$ & $\Delta_\lambda$ & $\Delta_J$ & $D$ \\
\midrule
All shared & $6.735$ & $76.779$ & $2.034$ & $2.22\times10^{-11}$ \\
Value--dual shared & $266.470$ & $1813.805$ & $2.177$ & $3.13\times10^{-11}$ \\
All separate & $5.12\times10^8$ & $8.33\times10^8$ & $2.210$ & $1.82\times10^4$ \\
\bottomrule
\end{tabular}
\end{table}

The suite contains 960 budget-comparison runs, 240 uniform-behavior runs,
600 discount-sweep runs, and 120 parameter-margin runs, totaling 1920
algorithm runs. The fixed-perturbation study adds 200 evaluation points.
Some settings coincide across studies and are not treated as additional
independent replicates in the statistical analysis.
Numerical verification reproduces the 80 original batch-320 runs to a
maximum absolute discrepancy of $3.55\times10^{-15}$ across the recorded
summary metrics. Independent scalar optimization verifies the policy
subproblem objective to $1.07\times10^{-14}$, and the checked value solves
have first-order residual at most $6.44\times10^{-15}$.
Configurations, per-run results, selected per-iteration traces, and the
analysis scripts are retained with the experiment package.

\end{document}